\documentclass{article}
\usepackage{iclr2027_conference,times}
\usepackage[utf8]{inputenc}
\DeclareUnicodeCharacter{2212}{\ensuremath{-}}
\usepackage[T1]{fontenc}
\usepackage{microtype}
\usepackage{url}
\usepackage{amsmath,amssymb}
\usepackage{graphicx}
\usepackage{booktabs}
\usepackage{float}
\usepackage{afterpage}
\usepackage{tikz}
\usetikzlibrary{arrows.meta,positioning,fit,backgrounds,calc}
\usepackage{pgfplots}
\usepackage{pgfplotstable}
\pgfplotsset{compat=1.18}
\usepgfplotslibrary{groupplots,fillbetween,patchplots}
\pgfplotsset{
  somapanel/.style={
    width=3.72cm, height=4.10cm, scale only axis,
    xmode=log, log basis x=10, ymin=1.67, ymax=2.63,
    ymajorgrids, xmajorgrids, xminorgrids,
    grid style={gray!25, very thin}, tick align=outside, tick pos=left,
    xlabel={wall-clock hours}, every axis label/.style={font=\scriptsize},
    tick label style={font=\scriptsize}, title style={font=\small},
    legend style={font=\scriptsize, draw=none, fill=none, at={(0.98,0.98)}, anchor=north east,
                  cells={anchor=west}, row sep=-1.5pt},
    legend cell align=left, mark size=1.25pt, line width=0.9pt,
    every axis plot/.style={mark=*},
  }
}
\definecolor{somaA}{HTML}{481D6F}\definecolor{somaB}{HTML}{423F85}
\definecolor{somaC}{HTML}{2A778E}\definecolor{somaD}{HTML}{1E9D89}\definecolor{somaE}{HTML}{7DAF17}
\definecolor{somaF}{HTML}{24908C}
\definecolor{somaG}{HTML}{B5DE2B}
\definecolor{somaOrange}{HTML}{E76F51}
\definecolor{secBlue}{HTML}{2563A6}
\definecolor{secTeal}{HTML}{16858C}
\definecolor{secGreen}{HTML}{3A9B62}
\definecolor{secPaleBlue}{HTML}{EAF2F8}
\definecolor{secPaleGreen}{HTML}{EAF6EF}
\definecolor{cmpSoma}{HTML}{16865C}
\definecolor{cmpTerminal}{HTML}{73CFA2}
\definecolor{cmpBptt}{HTML}{2F6FB0}
\definecolor{cmpSpsa}{HTML}{D1495B}
\definecolor{cmpEggroll}{HTML}{76528F}

\definecolor{somaT01}{HTML}{440154}\definecolor{somaT02}{HTML}{481F70}
\definecolor{somaT03}{HTML}{443983}\definecolor{somaT04}{HTML}{3B528B}
\definecolor{somaT05}{HTML}{31688E}\definecolor{somaT06}{HTML}{287C8E}
\definecolor{somaT07}{HTML}{21918C}\definecolor{somaT08}{HTML}{20A486}
\definecolor{somaT09}{HTML}{35B779}\definecolor{somaT10}{HTML}{5EC962}
\definecolor{somaT11}{HTML}{90D743}\definecolor{somaT12}{HTML}{C8E020}
\definecolor{somaT13}{HTML}{FDE725}
\usepackage[ruled,vlined,linesnumbered]{algorithm2e}
\usepackage{xcolor}
\usepackage[colorlinks=true,linkcolor=blue,citecolor=blue,urlcolor=blue]{hyperref}

\newcommand{\soma}{SOMA}

\newcommand{\R}{\mathbb{R}}
\newcommand{\E}{\mathbb{E}}
\newcommand{\softmax}{\operatorname{softmax}}

\usepackage{amsthm}
\newtheorem{proposition}{Proposition}
\newtheorem{corollary}[proposition]{Corollary}

\iclrfinalcopy

\title{Scaling Zero-Order Pretraining\\through Model Sharding}

\author{
\begin{tabular}{c}
Francois Chaubard \\
\normalfont Stanford University
\end{tabular}
\And
\begin{tabular}{c}
Mykel J. Kochenderfer \\
\normalfont Stanford University
\end{tabular}
\And
\begin{tabular}{c}
Chris R{\'e} \\
\normalfont Stanford University
\end{tabular}
}
\begin{document}
\raggedbottom
\maketitle
\lhead{Preprint}
\begin{abstract}
Zero-order optimization (ZO) enables training without backpropagation, making it relevant to forward-only hardware and non-differentiable loss, but its gradient variance grows with perturbed dimension. This inhibits large model training. Sharded Optimization Mixture of Assemblies (SOMA) is an architecture designed with ZO in mind. SOMA is an ensemble of LSTM experts that train independently on $N$ clusters of data using simultaneous perturbation stochastic approximation (SPSA). Its separable loss function removes cross-expert perturbation noise at the cost of jointly learned representations across domains. Experts train independently, without exchanging gradients, activations or optimizer state. We use 80,000 estimated RTX~5090 GPU-hours to study SOMA compared to baseline methods. Modest sharding improves training compute efficiency over all tested monolithic ZO controls. We study a 8.44M model at 150 aggregate GPU-hour budget and show SOMA $N=2$ with 64 perturbations reaches 1.76 test nats/byte, versus 2.00--2.11 for monolithic SPSA at 64, 256 or 1,024 perturbations and 2.21 for EGGROLL. On WikiText-103, these frozen checkpoints reach 2.07, 2.25--2.36 and 2.49, respectively. On a fixed separable objective with equal-size blocks, we prove that independent losses reduce relative gradient variance to approximately $1/N$ of a shared-loss estimator's. Holding starting weights, data, perturbations and compute fixed, independent rather than summed losses lower SOMA $N=4$ test loss by 0.035 nats/byte after 1,000 updates across three seeds. Finally, we show larger ensembles offer a separate inference benefit. At similar model size with top-$k$ routing ($k=4$), SOMA $N=256$ achieves 2.36M tokens/s versus 257k for SOMA $N=8$ ($9.19\times$, including routing), at lower test loss (1.68 versus 1.71), albeit with SOMA $N=256$ using $59.9\times$ as much aggregate training compute. We release all training and evaluation code and checkpoints for reproduction.
\end{abstract}
\definecolor{figSevenGreenLight}{HTML}{63B881}
\definecolor{figSevenGreenMedium}{HTML}{238B45}
\definecolor{figSevenGreenDark}{HTML}{005A32}
\definecolor{figSevenRedLight}{HTML}{E05C5C}
\definecolor{figSevenRedMedium}{HTML}{D22828}
\definecolor{figSevenRedDark}{HTML}{A31F1F}

\begingroup
\setlength{\intextsep}{4pt}
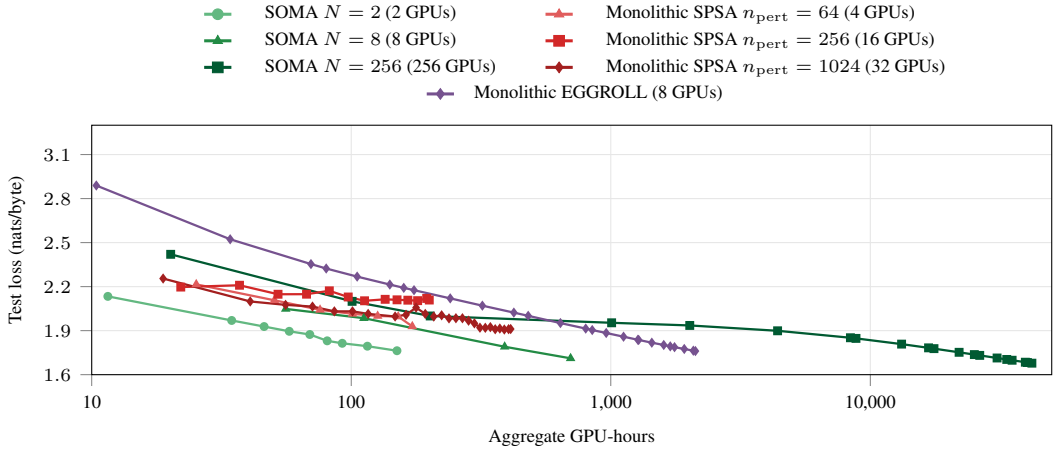
\begin{figure}[H]
\centering
\definecolor{somaC}{HTML}{2A778E}
\definecolor{somaE}{HTML}{7DAF17}
\definecolor{cmpSoma}{HTML}{16865C}
\definecolor{cmpSpsa}{HTML}{D1495B}
\definecolor{cmpEggroll}{HTML}{76528F}
{\small Model size $\simeq8.44$M, SOMA $n_{\rm pert}=B=64$, SPSA $B=64$}\\[3pt]
\begin{tikzpicture}
\begin{groupplot}[
 group style={group size=1 by 1,horizontal sep=.70cm},
 width=12.7cm,height=3.3cm,scale only axis,
 ymin=1.6,ymax=3.3,ytick={1.6,1.9,2.2,2.5,2.8,3.1},
 ymajorgrids,xmajorgrids,grid style={gray!20},minor x tick num=0,
 tick align=outside,tick pos=left,
 tick label style={font=\scriptsize},label style={font=\scriptsize},
 title style={font=\small},
 every axis plot/.append style={line width=.85pt,mark size=1.25pt}]
\nextgroupplot[
 xmode=log,log basis x=10,xmin=10,xmax=50000,xtick={10,100,1000,10000},xticklabels={10,100,{1,000},{10,000}},
 xlabel={Aggregate GPU-hours},ylabel={Test loss (nats/byte)},
 legend to name=computeComparisonLegend,legend columns=2,
 legend style={draw=none,font=\scriptsize,column sep=9pt,row sep=2pt},
 legend cell align=left]
\addplot[figSevenGreenLight,solid,mark=*,mark size=1.25pt,forget plot] table[x=gpu_hours,y=loss,col sep=comma]{figdata/rebuttal/figure1_soma_N2_measured.csv};
\addplot[figSevenGreenMedium,solid,mark=triangle*,mark size=1.25pt,forget plot] table[x=gpu_hours,y=loss,col sep=comma]{figdata/rebuttal/figure1_soma_N8_measured.csv};
\addplot[figSevenGreenDark,solid,mark=square*,mark size=1.25pt,forget plot] table[x=gpu_hours,y=loss,col sep=comma]{figdata/rebuttal/figure1_soma_N256_measured.csv};
\addplot[figSevenRedLight,solid,mark=triangle*,mark size=1.25pt,forget plot] table[x=gpu_hours,y=loss,col sep=comma]{figdata/rebuttal/figure1_soma_N1_measured.csv};
\addplot[figSevenRedMedium,solid,mark=square*,mark size=1.25pt,forget plot] table[x=gpu_hours,y=loss,col sep=comma]{figdata/rebuttal/figure1_monolithic_np256_measured.csv};
\addplot[figSevenRedDark,solid,mark=diamond*,mark size=1.25pt,forget plot] table[x=gpu_hours,y=loss,col sep=comma]{figdata/rebuttal/figure1_monolithic_np1024_measured.csv};
\addplot[cmpEggroll,solid,mark=diamond*,mark size=1.25pt,forget plot] table[x=gpu_hours,y=loss,col sep=comma]{figdata/rebuttal/figure1_eggroll_measured.csv};
\end{groupplot}
\node[anchor=south] at ($(group c1r1.north west)!.5!(group c1r1.north east)+(0,.15cm)$)
 {\scriptsize
\newcommand{\legendSwatch}[2]{\tikz[baseline=-.5ex]{\draw[#1,line width=1.15pt] (0,0)--(.36,0);\draw[#1] plot[mark=#2,mark size=2pt] coordinates {(.18,0)};}}
\begin{tabular}{@{}ll@{\hspace{18pt}}ll@{}}
\legendSwatch{figSevenGreenLight}{*} & SOMA $N=2$ (2 GPUs) & \legendSwatch{figSevenRedLight}{triangle*} & Monolithic SPSA $n_{\rm pert}=64$ (4 GPUs)\\[2pt]
\legendSwatch{figSevenGreenMedium}{triangle*} & SOMA $N=8$ (8 GPUs) & \legendSwatch{figSevenRedMedium}{square*} & Monolithic SPSA $n_{\rm pert}=256$ (16 GPUs)\\[2pt]
\legendSwatch{figSevenGreenDark}{square*} & SOMA $N=256$ (256 GPUs) & \legendSwatch{figSevenRedDark}{diamond*} & Monolithic SPSA $n_{\rm pert}=1024$ (32 GPUs)\\[2pt]
\multicolumn{4}{c}{\legendSwatch{cmpEggroll}{diamond*}\quad Monolithic EGGROLL (8 GPUs)}
\end{tabular}};
\end{tikzpicture}
\caption{Modest sharding improves training compute efficiency. Near 150 aggregate GPU-hours, SOMA $N=2$ reaches 1.76 test nats/byte, versus 2.00--2.11 for monolithic SPSA and 2.21 for EGGROLL. Appendix~\ref{app:figure2-controls} gives checkpoint and timing details.}
\label{fig:independent-evaluation}
\end{figure}
\endgroup
\clearpage

\section{Introduction}
Zero-order optimization (ZO) is useful when exact gradients are unavailable or costly, including forward-only hardware and non-differentiable loss. However, ZO methods struggle to improve loss as model size grows because, at a fixed number of perturbations per global step, relative gradient variance increases linearly with the number of perturbed parameters.

Simultaneous perturbation stochastic approximation (SPSA)~\citep{spall1992} estimates updates by averaging scalar loss differences along $n_{\rm pert}$ independent perturbation directions. With $M$ jointly perturbed parameters, leading relative gradient variance scales as $M/n_{\rm pert}$~\citep{nesterov2017}. Controlling that error requires a compute budget of $n_{\rm pert}\propto M$. If forward compute grows linearly with parameters at fixed batch size and context length, compute per global step grows as $O(M^2)$ (Section~\ref{sec:train}). Increasing the perturbation count improves each estimate but leaves fewer updates within a fixed compute budget. Reducing model size makes estimation easier but limits the parameters available for prediction. 

Modern techniques, such as EGGROLL~\citep{sarkar2025eggroll}, improve the efficiency of perturbation evaluations through low-rank structure. We investigate an orthogonal technique, changing the architecture so that total model size can grow without enlarging each independently estimated parameter block and gradient variance.

Following Cluster--Branch--Train--Merge (CBTM)~\citep{gururangan2023}, Sharded Optimization Mixture of Assemblies (SOMA) trains domain-specialized recurrent experts on separate corpus clusters (Figures~\ref{fig:arch} and~\ref{fig:recipe}). Each expert receives its own loss, making the training objective separable and removing noise from perturbations to other experts. Gradient estimation depends on expert size rather than full ensemble size, but experts give up joint learning of recurrent representations across domains. This changes both the architecture and the optimization problem. Domain specialization gives each expert a distinct learning task, independent losses isolate its update signal, and routing combines selected expert predictions at inference. The challenge is whether these independently learned predictors compensate for the shared representations and joint adaptation that decomposition removes.

We compare learned models at approximately fixed total size and aggregate compute, then isolate why, with independent versus summed-loss continuations at fixed architecture and gradient-variance measurements. Six ablations examine data partitioning, initialization, head updates, independent losses, perturbation and batch allocation, and active expert count. Finally, we study SOMA's inference benefits in exchange for larger training budgets. Our contributions are:
\begin{itemize}\itemsep1pt\interlinepenalty=10000
\item We introduce SOMA for independent ZO expert training and demonstrate improved compute efficiency over the tested monolithic controls. Near 8.44M parameters and 150 aggregate GPU-hours, SOMA $N=2$ reaches 1.76 test nats/byte, versus 2.00--2.11 for monolithic SPSA at 64, 256 or 1,024 perturbations and 2.21 for EGGROLL (Figure~\ref{fig:independent-evaluation}). The same frozen checkpoints score 2.07, 2.25--2.36 and 2.49 on WikiText-103 (Table~\ref{tab:wikitext-transfer}). An early comparison with equal compute and tuning budgets favors SOMA $N=2$ over monolithic SPSA in all three optimization seeds (Table~\ref{tab:early-tuning-test}).
\item We explain and test the independent loss mechanism. Controlling for the same expert sizes, and number of experts, we prove that independent losses reduce relative gradient variance to approximately $1/N$ of a shared-loss estimator's. Independent losses give $4.02\times$ lower error compared to summed loss on SOMA $N=4$ (Figure~\ref{fig:theory-empirical}). This translates to lower test loss by 0.035 nats/byte after 1,000 updates across three seeds (Figure~\ref{fig:loss-separation}). We compare the marginal benefits of sharding versus investing in more $n_{\rm pert}$ or batch size ($B$) and show sharding reduces relative gradient error more than increasing $n_{\rm pert}$ or $B$ at comparable aggregate GPU-seconds per global step (Figure~\ref{fig:variance-efficiency}), with $118\times$ lower measured relative centered variance for SOMA $N=256$ than the $N=1$ control (Figure~\ref{fig:allocation-variance}).
\item Finally, we characterize the separate inference benefit of heavier sharding (Figure~\ref{fig:isoparam-wallclock}), measuring how loss changes with the number of active experts at fixed expert width (Figure~\ref{fig:topk-flops-frontier}). At approximately equal total model size with top-$k$ routing ($k=4$), SOMA $N=256$ achieves 2.36M tokens/s versus 257k for SOMA $N=8$ ($9.19\times$, including routing), with test losses of 1.68 and 1.71, albeit with SOMA $N=256$ using $59.9\times$ as much aggregate training compute (Tables~\ref{tab:inference-full} and~\ref{tab:inference-n8-n256-loss}).
\end{itemize}

\section{Related work}\label{sec:related-work}
SOMA combines independent expert training with zeroth-order estimation. Its closest precedents therefore concern both how language models are divided into experts and how their parameters can be learned with ZO.
CBTM trains independent transformer language-model experts on corpus clusters and combines them at inference~\citep{li2022btm,gururangan2023}. SmallTalk LM also trains independent models and routes from a short prefix~\citep{filippova2025smalltalk}. SOMA adopts this organization to study a ZO-specific question, whether local learning signals compensate for the loss of joint adaptation. It differs from jointly trained mixtures of experts, whose router and experts exchange training signals~\citep{shazeer2017}.

Earlier recurrent ZO work scales a monolithic architecture~\citep{chaubard2025cdrge}. EGGROLL uses low-rank evolution strategies for recurrent-model pretraining~\citep{sarkar2025eggroll}. MeZO and related methods reduce estimator or memory costs for language-model fine-tuning~\citep{malladi2023mezo,wang2024lezo,chen2024lozo,yu2024subzero}. MeZO-SVRG uses control variates and Sparse MeZO selects parameters to update, both evaluated in fine-tuning~\citep{gautam2024mezosvrg,liu2025sparsemezo}. Appendix~\ref{app:related-work} gives the broader comparison.

ZO pretraining is distinct from adapting a model already trained with backpropagation. \citet{allaire2025zopretraining} study this difficulty in a 20M-parameter model. KronZO pretrains GPT-2 Small on OpenWebText with Kronecker-structured perturbations~\citep{allaire2026kronzo}. 
EGGROLL reports 3.40 bits/byte (or 2.36 nats/byte) pretraining an INT8 recurrent LLM on MiniPile ~\citep{sarkar2025eggroll}. These results establish other routes to ZO pretraining. Their losses use different datasets and evaluation protocols, so they do not provide a common ranking with FineWeb-Edu. Figure~\ref{fig:independent-evaluation} compares our FP32 EGGROLL reproduction with SOMA and monolithic SPSA on the same test set across aggregate training budgets.

\section{The SOMA architecture}\label{sec:arch}\label{sec:model-size}
SOMA (Figure~\ref{fig:arch}) is an ensemble of experts that each predict independently using their own recurrent state. Expert predictions are combined only at inference, not during training. We do not use transformers as each expert would require a key-value cache that would grow memory with context length. A shared, weight-tied matrix $E$ embeds and decodes bytes. Each expert has two residual blocks, each applying an LSTM followed by a multilayer perceptron (MLP). The fixed router selects the top-$k$ experts, with $k=\min(4,N)$, using term frequency--inverse document frequency (tf--idf) weights~\citep{salton1988term}, reduced by truncated singular value decomposition (SVD)~\citep{deerwester1990indexing}. Their next-byte probabilities are averaged. In Figure~\ref{fig:arch}, expert $f_i$ produces hidden state $h_i$ and prediction $p_i$.
\begin{figure}[!htbp]\centering
\resizebox{\linewidth}{!}{%
\begin{tikzpicture}[node distance=6mm and 11mm,
  box/.style={draw,rounded corners,minimum height=7mm,minimum width=13mm,align=center,font=\small},
  shared/.style={draw,rounded corners,minimum height=7mm,minimum width=14mm,align=center,font=\small,fill=orange!12},
  ex/.style={draw,rounded corners,minimum height=7mm,minimum width=22mm,align=center,font=\scriptsize,fill=blue!8},
  pred/.style={draw,rounded corners,minimum height=8mm,minimum width=21mm,align=center,font=\scriptsize,fill=green!9},
  merge/.style={draw,rounded corners,minimum height=8mm,minimum width=22mm,align=center,font=\scriptsize,fill=purple!8},
  sample/.style={draw,rounded corners,minimum height=7mm,minimum width=16mm,align=center,font=\scriptsize,fill=red!7},
  >={Stealth[length=2mm]}]
  \node[box] (ctx) {context};
  \node[box,below=9mm of ctx] (tok) {token};
  \node[box,right=of ctx] (router) {tf--idf $+$ SVD\\\scriptsize centroid router};
  \node[shared,right=of tok] (E) {shared $E$\\\scriptsize token embedding};
  \node[ex,right=13mm of E,yshift=11mm] (e1) {$f_1$: LSTM$_1$/LSTM$_2$ (cluster 1)};
  \node[ex,below=3mm of e1] (e2) {$f_2$: LSTM$_1$/LSTM$_2$ (cluster 2)};
  \node[ex,below=3mm of e2] (e3) {$\cdots\ f_N$ (cluster $N$)};
  \node[shared,right=10mm of e2] (Et) {shared $E^\top$\\\scriptsize tied decode};
  \node[pred,right=6mm of Et] (pred) {selected predictions\\$p_i=\softmax(h_iE^\top)$};
  \node[merge,right=6mm of pred] (ens) {top-4 ensemble\\$p=|S|^{-1}\!\sum_{i\in S}p_i$};
  \node[sample,right=6mm of ens] (sample) {sample next token\\$y_{t+1}\sim p$};
  \node[fit=(e1)(e2)(e3),inner sep=1.5mm] (bank) {};
  \draw[->] (ctx) -- (router);
  \draw[->] (tok) -- (E);
  \draw[->,dashed] (router.east) to[out=15,in=110]
    node[above,font=\scriptsize]{select top-4 set $S$} (bank.north);
  \draw[->] (E) -- (e1.west); \draw[->] (E) -- (e2.west); \draw[->] (E) -- (e3.west);
  \draw[->] (e1.east) -- (Et); \draw[->] (e2.east) -- (Et); \draw[->] (e3.east) -- (Et);
  \draw[->] (Et) -- (pred); \draw[->] (pred) -- (ens); \draw[->] (ens) -- (sample);
  \draw[<->,dotted,thick,orange!70!black] (E.south) to[out=-90,in=-90,looseness=0.6] node[below,font=\scriptsize,black]{tied (same matrix $E$)} (Et.south);
  \node[font=\scriptsize,below=4mm of E,align=center,text=black!70] {frozen shared head\\(local decode update in seed)};
  \node[font=\scriptsize,above=2mm of bank.north,align=center,text=black!70] {expert bodies: local SPSA};
\end{tikzpicture}
}
\caption{The SOMA architecture. A fixed router selects up to four independently trained LSTM experts. They share a frozen embedding and decoder, and their predictions are averaged. The decoder is trained once in the seed run.}
\label{fig:arch}
\end{figure}
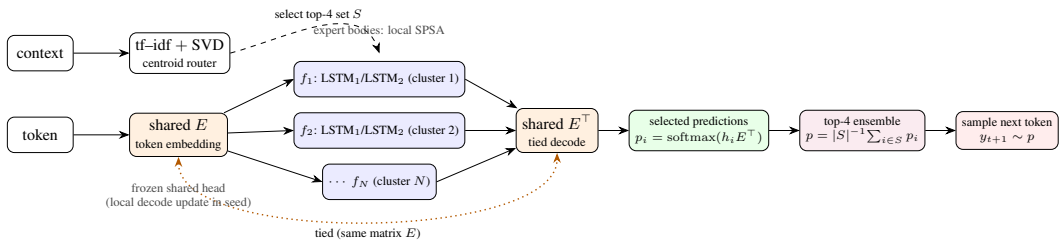

Let $x_{<t}$ be the context before byte position $t$, $y_t$ the target byte, and $p_e$ expert $e$'s predicted distribution. For the router-selected set $S$, the ensemble prediction $p$ averages the $|S|$ expert probabilities,
\begin{equation}
p(y_t\mid x_{<t})=\frac{1}{|S|}\sum_{e\in S}p_e(y_t\mid x_{<t}).
\label{eq:main-mixture}
\end{equation}
Each expert predicts the same byte vocabulary. The router chooses $S$ from the observed context and keeps it fixed over the scored window. Training does not differentiate through this mixture. Each expert learns on its assigned cluster without evaluating the others.

\section{The SOMA recipe}\label{sec:recipe}
To specialize experts without jointly training a router, we use the same corpus partition for training assignment and inference routing. Following the CBTM recipe, we first lightly train an $N=1$ seed for 1,000 updates on a 10B-byte subset of FineWeb-Edu~\citep{penedo2024fineweb}. The seed learns the body and shared embedding/decoder. We cluster the 100B-byte corpus using word-level tf--idf, a 128-dimensional SVD and balanced spherical $k$-means. Each sequence is assigned to one cluster (Figure~\ref{fig:recipe}). We use a byte-level tokenizer with a 256-symbol vocabulary to reduce compute across the repeated forward evaluations required by SPSA.
\begin{figure}[!htbp]\centering
\resizebox{0.95\linewidth}{!}{%
\begin{tikzpicture}[node distance=5mm and 7mm,font=\small,>={Stealth[length=2mm]},
  st/.style={draw,rounded corners,minimum height=9mm,minimum width=23mm,align=center},
  cl/.style={draw,rounded corners,minimum height=6mm,minimum width=10mm,align=center,font=\scriptsize,fill=blue!8}]
  \node[st,fill=yellow!15] (s) {\textbf{Seed}\\\scriptsize pretrain $N{=}1$\\\scriptsize learn $E$ $+$ body};
  \node[st,right=of s,fill=orange!12] (c) {\textbf{Cluster}\\\scriptsize tf--idf $+$ SVD$_{128}$\\\scriptsize balanced $k$-means, $N$ clusters};
  \node[st,right=of c,fill=green!8] (b) {\textbf{Branch}\\\scriptsize copy seed body\\\scriptsize to $N$ experts\\\scriptsize copy $E$};
  \node[st,right=of b,fill=blue!8] (t) {\textbf{Train}\\\scriptsize independent SPSA\\\scriptsize 1 expert / GPU};
  \node[st,right=of t,fill=red!7] (m) {\textbf{Merge}\\\scriptsize tf--idf top-$k$ route\\\scriptsize $+$ ensemble};
  \draw[->] (s)--(c); \draw[->] (c)--(b); \draw[->] (b)--(t); \draw[->] (t)--(m);
  \node[font=\scriptsize,below=1mm of s,text=black!70,align=center]{prepare shared head};
  \node[font=\scriptsize,below=1mm of c,text=black!70,align=center]{offline, once};
  \node[font=\scriptsize,below=1mm of t,text=black!70,align=center]{no routing during training\\no gradient exchange};
  \node[font=\scriptsize,below=1mm of m,text=black!70,align=center]{compute for\\selected experts};
\end{tikzpicture}}
\caption{The SOMA recipe. A seed learns the body and shared embedding and decoder. Clustering assigns each expert its data. Each expert copies the seed body and trains independently with the decoder frozen. The same clustering pipeline uses top-$k$ routing with $k=\min(4,N)$ at inference.}
\label{fig:recipe}
\end{figure}
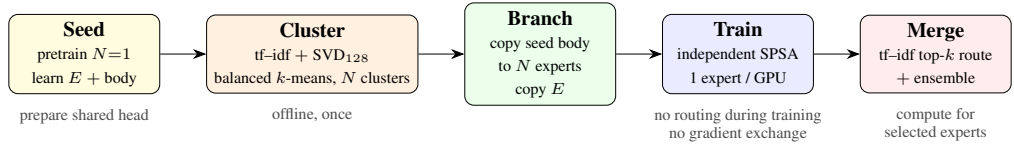

The seed head uses the cross-entropy derivative with respect to the decoder weights. We use delta-rule to compute the exact decoder-path gradient without backpropagating through the recurrent body or perturbing the tied embedding/decoder weights (Appendix~\ref{app:training-details}), treating hidden states as fixed. Expert bodies use SPSA without backpropagating through the recurrent state. For expert $e$, let $L_e$ be its independent loss and $\theta_e\in\mathbb{R}^{d_e}$ its $d_e$ perturbed parameters. For the dense-probe calculation, independent Rademacher probes $z_j\in\{-1,+1\}^{d_e}$ give the estimate
\begin{equation}
\widehat g_e=\frac{1}{n_{\rm pert}}\sum_{j=1}^{n_{\rm pert}}
\frac{L_e(\theta_e+\varepsilon z_j)-L_e(\theta_e-\varepsilon z_j)}{2\varepsilon}z_j.
\label{eq:main-spsa}
\end{equation}

Each expert uses its own loss, optimizer state and GPU. Training probes are sparse, with density 0.5 on ordinary weights. The empirical predictions use their recorded distribution (Appendix~\ref{app:estimator-iso}). All directions and both signs share a batch before accumulation. We use Adam-style moments~\citep{kingma2015adam} and accumulate independently sampled batches and directions. Keeping the head out of SPSA reduces the search dimension.

At inference, the clustering pipeline selects the nearest $\min(4,N)$ centroids based on the context before the experts run, so the same partition determines both what an expert learns and what sequence it is selected to produce at test time. Appendix~\ref{app:recipe-details} gives the implementation settings. Appendix~\ref{app:training-details} specifies the updates. 

The main sweeps use 1,024-byte windows, $n_{\rm pert}=64$ perturbation directions and batch size $B=64$ sequences per expert unless varied. The learning rate and perturbation radius ($\varepsilon$) start at $10^{-3}$ and decay on expert validation plateaus.

\section{Gradient variance}\label{sec:train}
The purpose of independent losses is to make each perturbed forward pass as informative as possible for gradient estimation. We quantify this benefit by comparing independent and summed losses at fixed total model size. We also compare sharding with increasing perturbations per step. 

To isolate the effect of sharding, we first hold the batch and objective fixed. Let $g_e=\nabla L_e$ be the exact batch gradient, $\|\cdot\|$ the Euclidean norm and $\E$ expectation over the probes. For a loss with three continuous derivatives, nonzero $g_e$, dense Rademacher directions and small $\varepsilon$,
\begin{equation}
\frac{\E\|\widehat g_e-g_e\|^2}{\|g_e\|^2}
=\frac{d_e-1}{n_{\rm pert}}+O(\varepsilon^2).
\label{eq:main-variance}
\end{equation}
For $N$ equal blocks of a separable objective with $M$ total perturbed parameters across all experts, the ratio of independent loss to shared-loss variance is, to leading order,
\begin{equation}
\frac{M/N-1}{M-1}\xrightarrow[M\to\infty]{N\ \mathrm{fixed}}\frac{1}{N}.
\label{eq:main-sharding}
\end{equation}
Independent losses remove cross-expert perturbation noise on a separable objective. Under parameter-linear forward cost, independent and joint perturbation rounds require equal aggregate work, whereas increasing $n_{\rm pert}$ reduces variance by proportionally increasing work. Adding fixed-width experts leaves $d_e$ unchanged. Appendix~\ref{app:block-sharding-proof} gives the proofs and a stationarity bound permitting larger steps for smaller blocks.

Using recorded probes and saved $\varepsilon$ at 10,000 updates, SOMA $N=256$ has $118\times$ lower relative centered variance than $N=1$, while doubling $n_{\rm pert}$ halves variance (Figure~\ref{fig:allocation-variance}). On a fixed objective, independent losses yield a $4.02\times$ reduction versus $4.00\times$ predicted. Figure~\ref{fig:theory-empirical} also tests the batch--perturbation trade-off (Appendices~\ref{app:estimator-iso} and~\ref{app:theory-empirical}). Centered variance excludes bias, which is included in the gradient-error comparison in Section~\ref{sec:compute-cost}. Figures~\ref{fig:loss-separation} and~\ref{fig:independent-evaluation} test the consequences for learning and predictions at a fixed training budget.

\section{Experiments}\label{sec:exp}
While we have shown sharding reduces relative gradient variance, the practical question is whether sharding improves predictions within a training budget. We first compare similarly sized models at comparable aggregate compute, then test the result across optimization seeds and on another held-out corpus. Controlled continuations (independent loss vs. summed) isolate the contribution of independent losses. 
\paragraph{Evaluation.}
Table~\ref{tab:experimental-setup} defines the evaluation sets. Expert validation loss controls learning-rate decay, ensemble validation loss tracks the scaling curves, and test loss scores the held-out comparisons. We route on 256 observed bytes and score the next 768 targets, resetting hidden state between windows. Appendix~\ref{app:independent-evaluation-data} details the test-set overlap checks.
Table~\ref{tab:fixed-parameter-counts} gives model settings.

\paragraph{Distributed training regimes.}
Independent expert losses let us distribute experts across GPUs without exchanging updates. We call this Fully Sharded Optimization (FSO). To increase $n_{\rm pert}$ or $B$ beyond what one GPU can handle and remain as parallelized as possible, Distributed Data and Perturbation Parallelism (DDPP) divides an expert's evaluations across GPUs and combines their gradient estimates before each update. %
Appendix~\ref{app:compute-details} gives implementation and compute details for both regimes. DDPP was critical to implement for fair baseline comparison of large $n_{\rm pert}$ and $B$ runs. 

\paragraph{Training compute at fixed model size.}
Near 8.44M parameters and 150 aggregate GPU-hours, SOMA $N=2$ reaches 1.76 test nats/byte. Monolithic SPSA reaches 2.00, 2.11 and 2.00 with $n_{\rm pert}=64,256,1024$, respectively, and our EGGROLL reproduction reaches 2.21 at similar cost (Figure~\ref{fig:independent-evaluation}). Increasing the monolith's perturbation budget does not close the gap in these runs. Table~\ref{tab:near150-controls} gives the checkpoint update counts and device costs. SOMA $N=2$ completes $2.2\times$ as many updates per GPU-hour as the four-GPU DDPP monolith, because each expert is smaller, executes more efficiently on its GPU, and exchanges nothing. At equal update counts of 30k per expert, SOMA $N=2$ reaches 1.87 test nats/byte against 2.00 for the monolith.

\paragraph{Three-seed tuning comparison.}\label{sec:three-seed-tuning}
The practical comparison also depends on how well each architecture's optimizer is configured. We therefore repeat the learning-rate and $\varepsilon$ search for monolithic SPSA and SOMA $N=2$ across three optimization seeds. Each architecture receives five settings per seed, giving 30 runs with 1.65 aggregate GPU-hours per setting on identical hardware. Seeds change sampled batches and perturbations while holding each architecture's starting checkpoint fixed. Ensemble validation loss selects the settings before test evaluation. SOMA $N=2$ achieves lower test loss in all three seeds, with a mean reduction of 0.0396 nats/byte, a minimum of 0.0358 and a maximum of 0.0462. Both architectures select the lowest learning rate tested. Appendix~\ref{app:early-tuning} reports the search settings and per-seed losses.

\paragraph{External-corpus evaluation.}
To test whether the loss advantage extends beyond FineWeb-Edu, we evaluate the same checkpoints near 150 aggregate GPU-hours on WikiText-103. SOMA $N=2$ reaches 2.07 test nats/byte, compared with 2.25--2.36 for monolithic SPSA and 2.49 for EGGROLL. All models remain frozen and score identical byte targets (Appendix~\ref{app:wikitext-transfer}).

\paragraph{Independent expert losses.}
To isolate whether independent or local learning signal is beneficial versus a summed-loss alternative, we use SOMA $N=4$ architecture, with fixed starting weights, Adam states, data, perturbations and compute budget (number of steps). We compare using individual losses trained disjointly, or the sum of all four losses trained jointly. Independent losses improves test loss by 0.035 nats/byte after 1,000 updates, with improvement in all three random seeds (Figure~\ref{fig:loss-separation}). This isolates the benefit of independent losses on learning at equal compute.

\begin{figure}[!htbp]\centering
\begin{tikzpicture}
\begin{groupplot}[
 group style={group size=2 by 1,horizontal sep=1.55cm},
 width=.375\linewidth,height=3.0cm,scale only axis,
 xlabel={Aggregate GPU-hours},scaled x ticks=false,scaled y ticks=false,
 grid=major,grid style={gray!22},tick pos=left,tick align=outside,
 tick label style={font=\scriptsize},label style={font=\scriptsize},
 title style={font=\small},
 every axis plot/.append style={line width=1.05pt,mark size=2.5pt},
 legend style={draw=none,fill=none,font=\scriptsize,legend columns=2,column sep=8pt},
]
\nextgroupplot[
 legend to name=losscomparisonlegend,
 title={Start at 2,000 updates},ylabel={Test loss (nats/byte)},
 xmin=0,xmax=.16,ymin=2.345,ymax=2.46,
 xtick={0,.05,.10,.15},xticklabel style={/pgf/number format/fixed,/pgf/number format/precision=2},ytick={2.36,2.38,2.40,2.42,2.44,2.46}]
\addplot[cmpSoma,mark=*]
 table[x=local_gpu_hours,y=local_mean,col sep=comma]{figdata/aggregate_gpu_hours/loss_separation.csv};
\addlegendentry{SOMA: 4 Train loops, 4 Shards}
\addplot[cmpSpsa,mark=square*]
 table[x=summed_gpu_hours,y=summed_mean,col sep=comma]{figdata/aggregate_gpu_hours/loss_separation.csv};
\addlegendentry{SPSA: One train loop, 4 Shards}
\nextgroupplot[
 title={Start at 4M updates},
 xmin=0,xmax=5.55,ymin=1.80430,ymax=1.80470,
 xtick={0,2,4},ytick={1.8043,1.8044,1.8045,1.8046,1.8047},
 yticklabel style={/pgf/number format/fixed,/pgf/number format/precision=4,/pgf/number format/fixed zerofill}]
\addplot[cmpSoma,mark=*]
 table[x=local_gpu_hours,y=local_mean,col sep=comma]{figdata/aggregate_gpu_hours/loss_separation_mature.csv};
\addplot[cmpSpsa,mark=square*]
 table[x=summed_gpu_hours,y=summed_mean,col sep=comma]{figdata/aggregate_gpu_hours/loss_separation_mature.csv};
\end{groupplot}
\node[font=\scriptsize] at ($(group c1r1.north)!0.5!(group c2r1.north)+(0,1.25cm)$) {$N=4,\ n_{\rm pert}=B=64$, model size 0.140M};
\node at ($(group c1r1.north)!0.5!(group c2r1.north)+(0,.95cm)$) {\pgfplotslegendfromname{losscomparisonlegend}};
\end{tikzpicture}
\caption{Independent expert losses improve learning, with a smaller but still prevalent gain late in training. SOMA $N=4$ uses independent expert losses or their sum. Within each panel, starting weights, Adam states, data, perturbations and update counts match. Points show mean test loss across three random seeds.  Appendix~\ref{app:loss-separation} gives the protocol.}
\label{fig:loss-separation}
\end{figure}
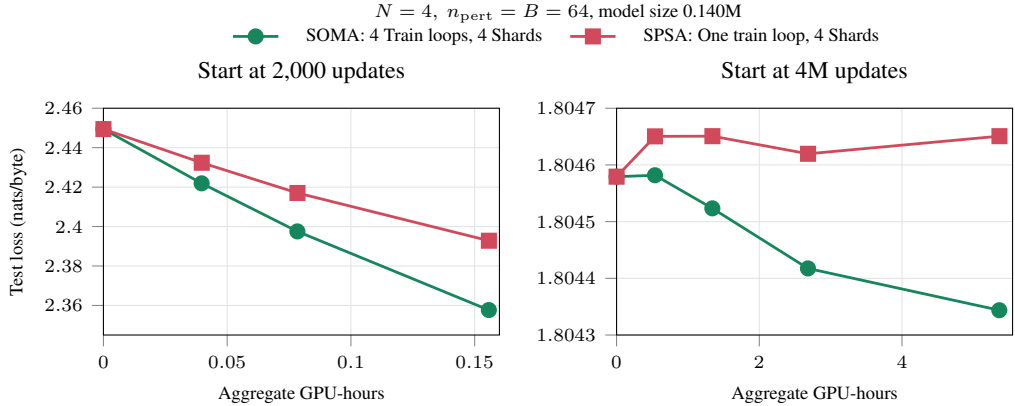

Independent losses reduce measured variance by factors of 3.19--4.64 on GPU and 3.20--4.65 in FP64, against the small-$\varepsilon$ prediction of four. After 4M updates, independent losses continue to improve test loss by 0.00024 nats/byte over 10,000 further updates, while summed losses worsen it by 0.00007, with independent losses favored in all three pairs. Appendix~\ref{app:loss-separation} gives the per-seed results.

\paragraph{Scaling sweeps.}
We first ask whether adding independent experts improves prediction despite the absence of joint training. At fixed expert width and $n_{\rm pert}=B=64$, increasing $N$ from one to 256 lowers ensemble validation loss from 1.84 to 1.72 after approximately 184 estimated training hours (Figure~\ref{fig:loss_curves}, left). Adding experts increases total parameters and training compute which SOMA can use to monotonically improve in loss without increasing $n_{\rm pert}$ or $B$. We also show that SOMA can benefit from additional $n_{\rm pert}$ or $B$ (Figure~\ref{fig:loss_curves}, center and right, respectively). Figure~\ref{fig:independent-evaluation} and Figure~\ref{fig:isoparam-wallclock} address the separate question of loss at a fixed total model size and fixed training budget, developed in Section~\ref{sec:compute-cost}. SOMA $N=256$ continues to improve from 1.72 ensemble validation nats/byte at 4M updates per expert to 1.71 at 4.16M, with test loss 1.68. SOMA $N=8$ reaches test loss 1.71 using 700 GPU-hours. These results show that independently trained experts can combine into a useful language model.

\paragraph{Other ablations.}
Together with independent losses, we test the role of data partitioning, initialization, head updates, perturbation and batch allocation, and active expert count make six ablation studies. In the BPTT recipe controls, semantically clustered shards outperform random disjoint shards and global IID data. Warm initialization lowers loss, and decoder-only head training retains most of the benefit of updating both paths through the tied head (Figure~\ref{fig:original-ablations}). Appendix~\ref{app:original-ablations} gives the settings and three-seed results. To isolate allocation within an expert, we compare matched SOMA $N=8$ continuations with $(n_{\rm pert},B)=(64,1024),(256,256),(1024,64)$ at the same arithmetic compute. Their seed-mean final test losses differ by less than 0.00006 nats/byte after 512 updates, and the order changes between the two paired seeds (Appendix~\ref{app:allocation-continuation}).

These comparisons establish a training-compute advantage for SOMA $N=2$ among the tested controls and isolate independent losses as a mechanism that improves learning. We next examine why the lowest gradient variance need not identify the best allocation of additional compute.

\setlength{\intextsep}{4pt plus 2pt minus 2pt}
\setlength{\textfloatsep}{6pt plus 2pt minus 2pt}
\section{Allocating additional parallel compute}\label{sec:convergence}\label{sec:compute-cost}
Additional parallel compute can train more independent experts or evaluate more batches and perturbations for each expert update. This choice changes both the learning signal and the number of updates affordable within a budget. We therefore compare what each allocation buys, first in gradient accuracy per global step and then in the loss reached after training. This separates the benefit of a more accurate update from the benefit of completing more updates before a deadline.

The sweeps in Figure~\ref{fig:loss_curves} establish that all three allocations can improve loss at comparable estimated training durations. They vary $N\in\{1,2,8,32,256\}$ at fixed expert width and $n_{\rm pert},B\in\{16,64,256,1024\}$ at $N=8$, with larger effective batches implemented through independent-batch accumulation. Because these settings spend different aggregate compute, their equal-time ranking need not identify the best use of a fixed budget. Appendix~\ref{app:compute-details} gives the operation counts and training-time accounting used to make that distinction.

Figure~\ref{fig:variance-efficiency} compares relative gradient variance with measured aggregate GPU-seconds per global step. The $N$ curve uses separate 10,000-update ensembles of roughly 8.44M parameters. The $n_{\rm pert}$ and $B$ curves share one $N=1$ checkpoint and reference gradient. Each architecture is evaluated against its own reference gradient.
For expert $e$, let $g_{e,\rm ref}$ be the BPTT gradient on a large independent reference batch and $\widehat g_e$ the ZO estimate on a sampled batch. The relative gradient variance $R$ averages over sampled batches and perturbations, holding the reference fixed,
\begin{equation}
R=\frac{\sum_{e=1}^{N}\E\|\widehat g_e-g_{e,\rm ref}\|^2}
{\sum_{e=1}^{N}\|g_{e,\rm ref}\|^2}.
\label{eq:main-reference-error}
\end{equation}
This measures total gradient error, including batch variation, perturbation variation and bias relative to the reference. We sum errors across experts before dividing by their summed squared gradient norms.

We pair this error with the aggregate GPU time needed to update every expert once, including communication and waiting. Appendix~\ref{app:variance-efficiency} gives the reference batches, estimator and timing procedure.
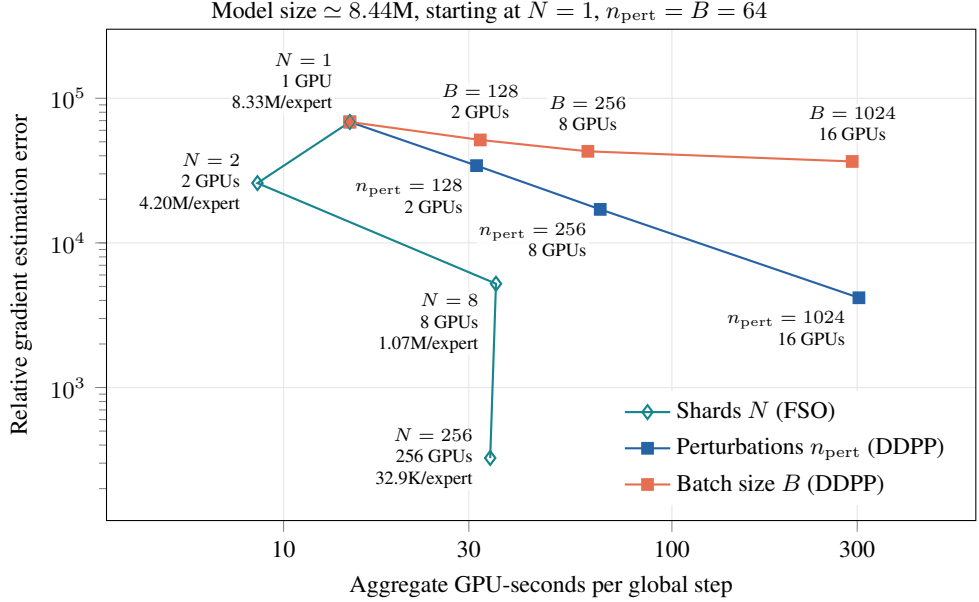
\begin{figure}[!tbp]\centering
{\small Model size $\simeq8.44$M, starting at $N=1$, $n_{\rm pert}=B=64$}\\[1pt]
\begin{tikzpicture}
\begin{axis}[name=trainedStage,width=11.5cm,height=6.5cm,scale only axis,xmode=log,ymode=log,
xmin=3.5,xmax=606.16250583,ymin=120.309186218,ymax=300000,
xtick={10,30,100,300},xticklabels={10,30,100,300},scaled ticks=false,
grid=major,grid style={black!10},tick align=outside,tick pos=left,clip=false,
tick label style={font=\small},label style={font=\small},
ylabel={Relative gradient estimation error},xlabel={Aggregate GPU-seconds per global step},
legend columns=1,legend style={at={(0.98,0.03)},anchor=south east,draw=none,fill=white,font=\small,cells={anchor=west},row sep=2pt}]
\addplot[secTeal,thick,mark=diamond,mark size=2.6pt,mark options={solid,fill=white}] table[x=aggregate_GPU_seconds,y=relative_mse,col sep=comma]{figdata/population_reference_20260909/tenk_N.csv};
\addlegendentry{Shards $N$ (FSO)}
\addplot[secBlue,thick,mark=square*,mark size=2pt] table[x=aggregate_GPU_seconds,y=relative_mse,col sep=comma]{figdata/population_reference_20260909/tenk_npert.csv};
\addlegendentry{Perturbations $n_{\rm pert}$ (DDPP)}
\addplot[somaOrange,thick,mark=square*,mark size=2pt] table[x=aggregate_GPU_seconds,y=relative_mse,col sep=comma]{figdata/population_reference_20260909/tenk_B.csv};
\addlegendentry{Batch size $B$ (DDPP)}
\pgfplotstableread[col sep=comma]{figdata/population_reference_20260909/tenk_N.csv}\tenkNTable
\pgfplotstablegetelem{0}{aggregate_GPU_seconds}\of\tenkNTable\expandafter\let\csname tenkNAX\endcsname\pgfplotsretval
\pgfplotstablegetelem{0}{relative_mse}\of\tenkNTable\expandafter\let\csname tenkNAY\endcsname\pgfplotsretval
\pgfplotstablegetelem{0}{label}\of\tenkNTable\expandafter\let\csname tenkNALabel\endcsname\pgfplotsretval
\pgfplotstablegetelem{0}{gpu_label}\of\tenkNTable\expandafter\let\csname tenkNAGPU\endcsname\pgfplotsretval
\pgfplotstablegetelem{0}{expert_size_label}\of\tenkNTable\expandafter\let\csname tenkNASize\endcsname\pgfplotsretval
\node[font=\scriptsize,align=right,anchor=south east,inner sep=.7pt,fill=white,fill opacity=.9,text opacity=1,xshift=-2mm,yshift=1mm] (tenkNANode) at (axis cs:\tenkNAX,\tenkNAY) {$\tenkNALabel$\\\tenkNAGPU\\\tenkNASize};
\pgfplotstablegetelem{1}{aggregate_GPU_seconds}\of\tenkNTable\expandafter\let\csname tenkNBX\endcsname\pgfplotsretval
\pgfplotstablegetelem{1}{relative_mse}\of\tenkNTable\expandafter\let\csname tenkNBY\endcsname\pgfplotsretval
\pgfplotstablegetelem{1}{label}\of\tenkNTable\expandafter\let\csname tenkNBLabel\endcsname\pgfplotsretval
\pgfplotstablegetelem{1}{gpu_label}\of\tenkNTable\expandafter\let\csname tenkNBGPU\endcsname\pgfplotsretval
\pgfplotstablegetelem{1}{expert_size_label}\of\tenkNTable\expandafter\let\csname tenkNBSize\endcsname\pgfplotsretval
\node[font=\scriptsize,align=right,anchor=east,inner sep=.7pt,fill=white,fill opacity=.9,text opacity=1,xshift=-2mm,yshift=0mm] (tenkNBNode) at (axis cs:\tenkNBX,\tenkNBY) {$\tenkNBLabel$\\\tenkNBGPU\\\tenkNBSize};
\pgfplotstablegetelem{2}{aggregate_GPU_seconds}\of\tenkNTable\expandafter\let\csname tenkNCX\endcsname\pgfplotsretval
\pgfplotstablegetelem{2}{relative_mse}\of\tenkNTable\expandafter\let\csname tenkNCY\endcsname\pgfplotsretval
\pgfplotstablegetelem{2}{label}\of\tenkNTable\expandafter\let\csname tenkNCLabel\endcsname\pgfplotsretval
\pgfplotstablegetelem{2}{gpu_label}\of\tenkNTable\expandafter\let\csname tenkNCGPU\endcsname\pgfplotsretval
\pgfplotstablegetelem{2}{expert_size_label}\of\tenkNTable\expandafter\let\csname tenkNCSize\endcsname\pgfplotsretval
\node[font=\scriptsize,align=right,anchor=north east,inner sep=.7pt,fill=white,fill opacity=.9,text opacity=1,xshift=-2mm,yshift=-1mm] (tenkNCNode) at (axis cs:\tenkNCX,\tenkNCY) {$\tenkNCLabel$\\\tenkNCGPU\\\tenkNCSize};
\pgfplotstablegetelem{3}{aggregate_GPU_seconds}\of\tenkNTable\expandafter\let\csname tenkNDX\endcsname\pgfplotsretval
\pgfplotstablegetelem{3}{relative_mse}\of\tenkNTable\expandafter\let\csname tenkNDY\endcsname\pgfplotsretval
\pgfplotstablegetelem{3}{label}\of\tenkNTable\expandafter\let\csname tenkNDLabel\endcsname\pgfplotsretval
\pgfplotstablegetelem{3}{gpu_label}\of\tenkNTable\expandafter\let\csname tenkNDGPU\endcsname\pgfplotsretval
\pgfplotstablegetelem{3}{expert_size_label}\of\tenkNTable\expandafter\let\csname tenkNDSize\endcsname\pgfplotsretval
\node[font=\scriptsize,align=right,anchor=east,inner sep=.7pt,fill=white,fill opacity=.9,text opacity=1,xshift=-2mm,yshift=0mm] (tenkNDNode) at (axis cs:\tenkNDX,\tenkNDY) {$\tenkNDLabel$\\\tenkNDGPU\\\tenkNDSize};
\pgfplotstableread[col sep=comma]{figdata/population_reference_20260909/tenk_npert.csv}\tenknpertTable
\pgfplotstablegetelem{1}{aggregate_GPU_seconds}\of\tenknpertTable\expandafter\let\csname tenknpertBX\endcsname\pgfplotsretval
\pgfplotstablegetelem{1}{relative_mse}\of\tenknpertTable\expandafter\let\csname tenknpertBY\endcsname\pgfplotsretval
\pgfplotstablegetelem{1}{label}\of\tenknpertTable\expandafter\let\csname tenknpertBLabel\endcsname\pgfplotsretval
\pgfplotstablegetelem{1}{gpu_label}\of\tenknpertTable\expandafter\let\csname tenknpertBGPU\endcsname\pgfplotsretval
\node[font=\scriptsize,align=right,anchor=north east,inner sep=.7pt,fill=white,fill opacity=.9,text opacity=1,xshift=-1.5mm,yshift=-1.5mm] (tenknpertBNode) at (axis cs:\tenknpertBX,\tenknpertBY) {$\tenknpertBLabel$\\\tenknpertBGPU};
\pgfplotstablegetelem{2}{aggregate_GPU_seconds}\of\tenknpertTable\expandafter\let\csname tenknpertCX\endcsname\pgfplotsretval
\pgfplotstablegetelem{2}{relative_mse}\of\tenknpertTable\expandafter\let\csname tenknpertCY\endcsname\pgfplotsretval
\pgfplotstablegetelem{2}{label}\of\tenknpertTable\expandafter\let\csname tenknpertCLabel\endcsname\pgfplotsretval
\pgfplotstablegetelem{2}{gpu_label}\of\tenknpertTable\expandafter\let\csname tenknpertCGPU\endcsname\pgfplotsretval
\node[font=\scriptsize,align=right,anchor=north east,inner sep=.7pt,fill=white,fill opacity=.9,text opacity=1,xshift=-1.5mm,yshift=-1.5mm] (tenknpertCNode) at (axis cs:\tenknpertCX,\tenknpertCY) {$\tenknpertCLabel$\\\tenknpertCGPU};
\pgfplotstablegetelem{3}{aggregate_GPU_seconds}\of\tenknpertTable\expandafter\let\csname tenknpertDX\endcsname\pgfplotsretval
\pgfplotstablegetelem{3}{relative_mse}\of\tenknpertTable\expandafter\let\csname tenknpertDY\endcsname\pgfplotsretval
\pgfplotstablegetelem{3}{label}\of\tenknpertTable\expandafter\let\csname tenknpertDLabel\endcsname\pgfplotsretval
\pgfplotstablegetelem{3}{gpu_label}\of\tenknpertTable\expandafter\let\csname tenknpertDGPU\endcsname\pgfplotsretval
\node[font=\scriptsize,align=right,anchor=north east,inner sep=.7pt,fill=white,fill opacity=.9,text opacity=1,xshift=-1.5mm,yshift=-1.5mm] (tenknpertDNode) at (axis cs:\tenknpertDX,\tenknpertDY) {$\tenknpertDLabel$\\\tenknpertDGPU};
\pgfplotstableread[col sep=comma]{figdata/population_reference_20260909/tenk_B.csv}\tenkBTable
\pgfplotstablegetelem{1}{aggregate_GPU_seconds}\of\tenkBTable\expandafter\let\csname tenkBBX\endcsname\pgfplotsretval
\pgfplotstablegetelem{1}{relative_mse}\of\tenkBTable\expandafter\let\csname tenkBBY\endcsname\pgfplotsretval
\pgfplotstablegetelem{1}{label}\of\tenkBTable\expandafter\let\csname tenkBBLabel\endcsname\pgfplotsretval
\pgfplotstablegetelem{1}{gpu_label}\of\tenkBTable\expandafter\let\csname tenkBBGPU\endcsname\pgfplotsretval
\node[font=\scriptsize,align=center,anchor=south,inner sep=.7pt,fill=white,fill opacity=.9,text opacity=1,xshift=0mm,yshift=2.5mm] (tenkBBNode) at (axis cs:\tenkBBX,\tenkBBY) {$\tenkBBLabel$\\\tenkBBGPU};
\pgfplotstablegetelem{2}{aggregate_GPU_seconds}\of\tenkBTable\expandafter\let\csname tenkBCX\endcsname\pgfplotsretval
\pgfplotstablegetelem{2}{relative_mse}\of\tenkBTable\expandafter\let\csname tenkBCY\endcsname\pgfplotsretval
\pgfplotstablegetelem{2}{label}\of\tenkBTable\expandafter\let\csname tenkBCLabel\endcsname\pgfplotsretval
\pgfplotstablegetelem{2}{gpu_label}\of\tenkBTable\expandafter\let\csname tenkBCGPU\endcsname\pgfplotsretval
\node[font=\scriptsize,align=center,anchor=south,inner sep=.7pt,fill=white,fill opacity=.9,text opacity=1,xshift=0mm,yshift=2.5mm] (tenkBCNode) at (axis cs:\tenkBCX,\tenkBCY) {$\tenkBCLabel$\\\tenkBCGPU};
\pgfplotstablegetelem{3}{aggregate_GPU_seconds}\of\tenkBTable\expandafter\let\csname tenkBDX\endcsname\pgfplotsretval
\pgfplotstablegetelem{3}{relative_mse}\of\tenkBTable\expandafter\let\csname tenkBDY\endcsname\pgfplotsretval
\pgfplotstablegetelem{3}{label}\of\tenkBTable\expandafter\let\csname tenkBDLabel\endcsname\pgfplotsretval
\pgfplotstablegetelem{3}{gpu_label}\of\tenkBTable\expandafter\let\csname tenkBDGPU\endcsname\pgfplotsretval
\node[font=\scriptsize,align=center,anchor=south,inner sep=.7pt,fill=white,fill opacity=.9,text opacity=1,xshift=0mm,yshift=2.5mm] (tenkBDNode) at (axis cs:\tenkBDX,\tenkBDY) {$\tenkBDLabel$\\\tenkBDGPU};
\addplot[secTeal,only marks,forget plot,mark=diamond,mark size=2.6pt,mark options={solid,fill=white}] table[x=aggregate_GPU_seconds,y=relative_mse,col sep=comma,row predicate/.code={\ifnum\pgfplotstablerow>0\relax\pgfplotstableuserowfalse\fi}]{figdata/population_reference_20260909/tenk_N.csv};
\end{axis}
\end{tikzpicture}
\caption{Sharding and perturbation averaging lower gradient error. Larger batches also help. The $N$ curve uses separate ensembles after 10,000 updates. The other curves reuse the same $N=1$ checkpoint. Error is estimated from measured perturbations against each expert's BPTT reference on separate data, summed across experts and divided by the summed squared reference norms. FSO cost sums separately timed expert updates. DDPP cost includes all allocated GPUs, communication and waiting. Appendix~\ref{app:variance-efficiency} gives the full calculation. Model sizes follow Section~\ref{sec:model-size}.}
\label{fig:variance-efficiency}
\end{figure}

Increasing $n_{\rm pert}$ from 64 to 1024 lowers relative error by 94\%, while increasing $B$ over the same range lowers it by 47\%. Larger batches reduce data-sampling error, and more perturbations reduce error from the finite set of directions. Sharding changes the estimation problem itself.  At fewer aggregate GPU-seconds per global step compared to $n_{\rm pert}=1024$ or $B=1024$, $N=256$ at $n_{\rm pert}=B=64$ wins at 99.5\% lower relative gradient variance than the monolithic model. Note, despite $N=256$ FSO using 256 GPUs, there is no communication between GPUs during training, so steps can be much faster. Also note that $N=2$ is faster per step than $N=1$ because each model is smaller and executes more efficiently on the GPU.

Adding fixed-width experts increases total model size without enlarging each local estimation problem and monotonically lowers loss across the tested expert counts at comparable estimated wall-clock times (Figure~\ref{fig:loss_curves}). At similar model size, modest sharding favors training compute efficiency, while heavier sharding offers a separate inference benefit. Inference throughput is non-monotonic because smaller experts do not necessarily execute more efficiently on the GPU. For example, SOMA $N=256$ achieves substantially higher throughput than SOMA $N=8$ (Figure~\ref{fig:isoparam-wallclock}). SOMA $N=2$ uses all experts at inference, while SOMA $N=256$ uses top-$k$ routing with $k=4$. We compare test loss within a training budget to determine whether improved estimation offsets the reduction in active model size.

With aggregate budget $C$ and deadline $D$, let $u$ and $\tau$ be the measured aggregate cost and elapsed time per global step. The maximum number of global steps $s_{\max}$ is
\begin{equation}
s_{\max}(N,n_{\rm pert},B)=\left\lfloor\min\left\{\frac{C}{u(N,n_{\rm pert},B)},\frac{D}{\tau(N,n_{\rm pert},B)}\right\}\right\rfloor.
\label{eq:main-budget-deadline}
\end{equation}
We choose the lowest measured loss among configurations that fit the available hardware and have coverage through $s_{\max}$.

A deadline can favor more shards. Near 8.44M parameters, $N=256$ reaches ensemble validation loss 1.94 in an estimated 16.8 hours, versus 1.99 in 17.3 hours for $N=2$, using 4.30k rather than 34.6 GPU-hours (Figure~\ref{fig:deadline-allocation}). More parallel hardware gives lower loss at similar elapsed time, though greater aggregate cost.

Perturbation averaging remains useful within an expert. For SOMA $N=8$ at fixed expert width, increasing $n_{\rm pert}$ from 64 to 1024 reaches ensemble validation loss 1.69 rather than 1.77, using 1.07k rather than 1.42k GPU-hours (Figure~\ref{fig:loss_curves}, Appendix~\ref{app:experimental-details}). The training curves favor enough decomposition to improve estimation while keeping each expert large enough to predict well. Inference introduces a further constraint because additional experts need not all be evaluated for each prediction.

\section{Active-parameter inference analysis}\label{sec:active-parameter-inference}
Training updates all experts, but inference executes only the selected experts. Figure~\ref{fig:isoparam-wallclock} compares the resulting training and inference trade-offs at approximately fixed total model size. At 8.44M total parameters with top-$k$ routing ($k=4$), SOMA $N=256$ achieves 2.36M tokens/s versus 257k for SOMA $N=8$ ($9.19\times$, including routing), at 1.68 test loss versus 1.71, albeit with SOMA $N=256$ using $59.9\times$ as much aggregate training compute. This shows us our tradeoff. We can invest more into training cost to reduce inference costs by sharding more. 

Separately, we can select the number of active experts $k$ in our inference ensemble on the fly without any retraining, improving model performance up to a point. As shown in Figure~\ref{fig:topk-flops-frontier}, we sweep fixed-width experts from $k=1$ to $N$ at the same wall clock. At $k=4$, SOMA $N=256$ reaches ensemble validation loss 1.72 versus 1.78 for SOMA $N=8$.For SOMA $N=256$, increasing $k$ from 4 to 12 lowers ensemble validation loss by 0.013 nats/byte at three times the expert compute. Evaluating all 256 experts costs more and gives higher loss. Experts with less similar errors benefit more from averaging (Figure~\ref{fig:n256-correlation}).

\begin{figure}[H]
\centering
{\small Approximately 8.44M total parameters ($n_{\rm pert}=64$, $B=64$)}\par\smallskip
\begin{tikzpicture}
\begin{groupplot}[group style={group size=3 by 1,horizontal sep=.45cm},
  width=3.65cm,height=4.10cm,scale only axis,
  xmode=log,log basis x=10,
  ymajorgrids,xmajorgrids,xminorgrids,grid style={gray!25,very thin},
  tick align=outside,tick pos=left,
  every axis label/.style={font=\scriptsize},tick label style={font=\scriptsize},
  title style={font=\small},
  legend columns=4,
  legend style={font=\scriptsize,draw=none,fill=none,cells={anchor=west},/tikz/column sep=18pt},
  legend cell align=left,/tikz/line width=.9pt,ymin=1.67,ymax=2.35]
\nextgroupplot[legend to name=figfourteenlegend,xmin=.5,xmax=110,xtick={1,10,100},
  xlabel={\shortstack{Training wall-clock \\(hours)}},
  ylabel={\shortstack{Ensemble validation loss\\(nats/byte)}}]
\addplot[somaA,mark=*,mark size=1pt] table[x=wall_clock_hr,y=loss,col sep=comma]{figdata/rebuttal/figure13_wallclock_N1.csv};\addlegendentry{$N=1$}
\addplot[somaB,mark=*,mark size=1pt] table[x=wall_clock_hr,y=loss,col sep=comma]{figdata/rebuttal/figure13_wallclock_N2.csv};\addlegendentry{$N=2$}
\addplot[somaC,mark=*,mark size=1pt] table[x=wall_clock_hr,y=loss,col sep=comma]{figdata/rebuttal/figure13_wallclock_N8.csv};\addlegendentry{$N=8$}
\addplot[somaE,mark=*,mark size=1pt] table[x=wall_clock_hr,y=loss,col sep=comma]{figdata/rebuttal/figure13_wallclock_N256.csv};\addlegendentry{$N=256$}
\nextgroupplot[xmin=1,xmax=18000,xtick={1,100,10000},yticklabels={},
  xlabel={Aggregate GPU-hours}]
\addplot[somaA,mark=*,mark size=1pt] table[x=aggregate_gpu_hours,y=loss,col sep=comma]{figdata/rebuttal/figure13_wallclock_N1.csv};
\addplot[somaB,mark=*,mark size=1pt] table[x=aggregate_gpu_hours,y=loss,col sep=comma]{figdata/rebuttal/figure13_wallclock_N2.csv};
\addplot[somaC,mark=*,mark size=1pt] table[x=aggregate_gpu_hours,y=loss,col sep=comma]{figdata/rebuttal/figure13_wallclock_N8.csv};
\addplot[somaE,mark=*,mark size=1pt] table[x=aggregate_gpu_hours,y=loss,col sep=comma]{figdata/rebuttal/figure13_wallclock_N256.csv};
\nextgroupplot[xmin=40000,xmax=4000000,xtick={100000,1000000},yticklabels={},xlabel={\shortstack{Inference throughput\\(tokens/s)}}]
\addplot[somaA,only marks,mark=*,mark size=2pt] coordinates {(555149.690690,2.01445746)};
\node[font=\scriptsize,anchor=south,yshift=4pt] at (axis cs:555149.690690,2.01445746) {$N=1$};
\addplot[somaB,only marks,mark=*,mark size=2pt] coordinates {(467776.362624,1.78445923)};
\node[font=\scriptsize,anchor=north,yshift=-4pt] at (axis cs:467776.362624,1.78445923) {$N=2$};
\addplot[somaC,only marks,mark=*,mark size=2pt] coordinates {(257078.173391,1.8125267)};
\node[font=\scriptsize,anchor=south,yshift=4pt] at (axis cs:257078.173391,1.8125267) {$N=8$};
\addplot[somaE,only marks,mark=*,mark size=2pt] coordinates {(2362416.389133,1.79222798)};
\node[font=\scriptsize,anchor=south,xshift=-5pt,yshift=4pt] at (axis cs:2362416.389133,1.79222798) {$N=256$};
\end{groupplot}
\node[anchor=south,yshift=6pt] at ($(group c1r1.north west)!0.5!(group c3r1.north east)$) {\pgfplotslegendfromname{figfourteenlegend}};
\end{tikzpicture}
\caption{SOMA at approximately fixed total model size (8.44M parameters). Left and middle show ensemble validation loss against training wall-clock and aggregate GPU-hours. Right pairs each run's lowest measured loss with inference throughput for its architecture, including routing. Modest sharding favors training compute efficiency, while SOMA $N=256$ offers higher inference throughput. Endpoint training budgets differ. Appendices~\ref{app:figure2-controls} and~\ref{app:inference-benchmarks} give the training accounting and inference protocol.}
\label{fig:isoparam-wallclock}
\end{figure}
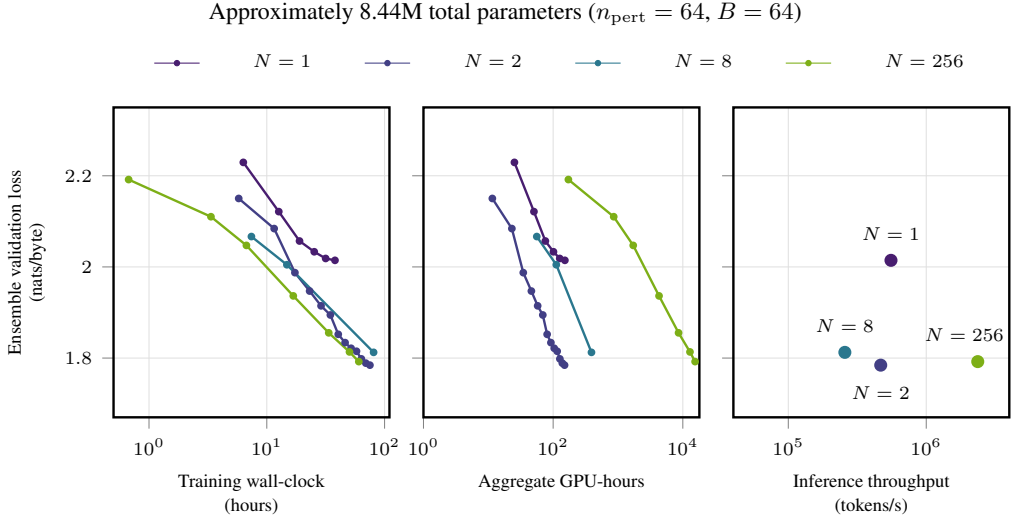

At approximately equal total model size, throughput including routing is 0.555M, 0.468M, 0.257M and 2.36M tokens/s for $N=1,2,8,256$, respectively (Figure~\ref{fig:isoparam-wallclock}, right, and Table~\ref{tab:inference-full}). Thus, inference throughput is non-monotonic in sharding degree. For SOMA $N=8$, recurrent execution, normalization and dispatch padding offset the active-parameter saving. Fewer parameters do not imply proportionally less runtime (Appendix~\ref{app:inference-benchmarks}). With top-$k$ routing at $k=4$, SOMA $N=256$ instead reaches $4.26\times$ the monolith's throughput and $9.19\times$ that of SOMA $N=8$, with both models using optimized grouped Triton kernels. Heavier sharding therefore offers an inference benefit distinct from training-compute efficiency.

\section{Limitations and conclusion}
In this paper, we show the benefits of sharding when using zero-order optimization. At approximately fixed total model size and training compute, SOMA $N=2$ outperforms the tested monolithic ZO controls. Controlled experiments support independent losses as a mechanism to reduce relative gradient error and loss. Heavier sharding offers a separate inference benefit, with SOMA $N=256$ providing $9.19\times$ the inference throughput of approximately equally sized SOMA $N=8$ under top-$k$ routing at $k=4$. Architectural independence therefore offers distinct training and inference benefits at the cost of joint learning across domains. As new inference-friendly hardware begins to popularize, we hope this work inspires others to adopt ZO and train novel, perhaps non-differentiable architectures.

This ZO pretraining study uses one byte-level corpus and small recurrent experts, limited to 8.44M parameters. Each expert sees only its own cluster, so more sharding trades less data per expert for a larger total model. At $N=256$, each expert processes the equivalent of about 700 passes over its training shard (Appendix~\ref{app:experimental-details}). While the most important runs were multi-seed, most scaling runs were single seed runs. %
Appendix~\ref{app:limitations} discusses further limitations.

\label{sec:main-end}
\setlength{\intextsep}{12pt plus 2pt minus 2pt}
\setlength{\textfloatsep}{20pt plus 2pt minus 4pt}

\clearpage
\section*{AI use statement}
Generative AI tools assisted with validating theoretical analysis, claims and proofs, implementing training and evaluation code, and in reviewing the manuscript. The authors take full responsibility for the final text, claims, code and results.

\section*{Reproducibility statement}
We release all code and checkpoints for full reproducibility. The appendices and zip file contain all guidance on how to reproduce all trainings and evaluations. The accompanying source includes the plotted CSV data. 

\clearpage
\flushbottom
\bibliographystyle{iclr2027_conference}
\bibliography{references}
\clearpage
\appendix
\raggedbottom
\section{Scaling sweeps and timing}\label{app:experimental-details}\label{app:total-device-time}
The scaling sweeps ask whether additional experts, perturbations or data per update improve prediction when more parallel compute is available. To interpret their learning curves, we distinguish elapsed training time (Wall-clock hours) from work summed across devices (Aggregate GPU-hours). This appendix records the common settings, evaluation sets and timing procedure for Figure~\ref{fig:loss_curves}.

\begingroup
\setlength{\intextsep}{4pt}
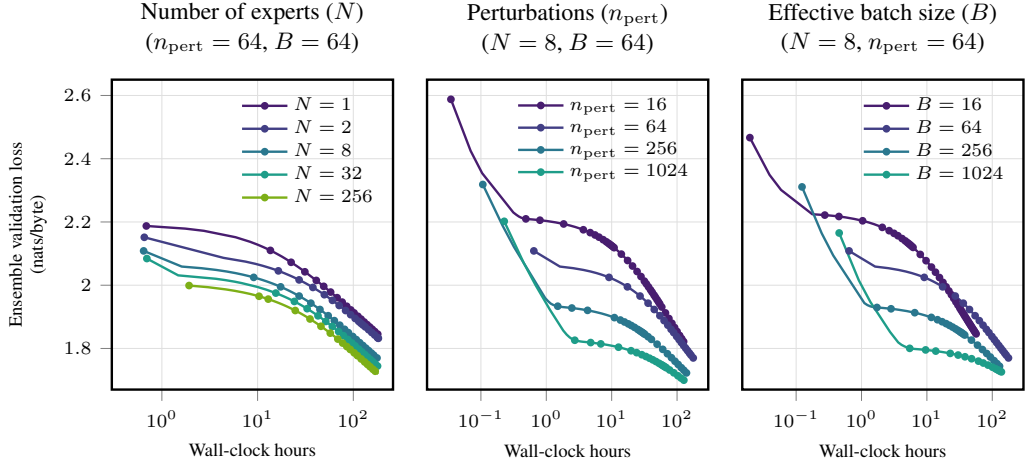
\begin{figure}[!htbp]\small\centering
\begin{tikzpicture}
\begin{groupplot}[group style={group size=3 by 1,horizontal sep=.45cm,ylabels at=edge left},
somapanel,height=4.10cm,xmin=.015,xmax=300,ymin=1.67,ymax=2.65,
xtick={.1,1,10,100},
xlabel={Wall-clock hours},ylabel={\shortstack{Ensemble validation loss\\(nats/byte)}}]
\nextgroupplot[xmin=.3,xmax=260,title={\shortstack{Number of experts ($N$)\\($n_{\rm pert}=64$, $B=64$)}}]
\addplot[somaA,mark repeat=8,mark size=1pt] table[x=wall_clock_hr,y=val,col sep=comma]{figdata/figure1_left_equal_time/N1.csv};\addlegendentry{$N=1$}
\addplot[somaB,mark repeat=8,mark size=1pt] table[x=wall_clock_hr,y=val,col sep=comma]{figdata/figure1_left_equal_time/N2.csv};\addlegendentry{$N=2$}
\addplot[somaC,mark repeat=8,mark size=1pt] table[x=wall_clock_hr,y=val,col sep=comma]{figdata/figure1_left_equal_time/N8.csv};\addlegendentry{$N=8$}
\addplot[somaD,mark repeat=8,mark size=1pt] table[x=wall_clock_hr,y=val,col sep=comma]{figdata/figure1_left_equal_time/N32.csv};\addlegendentry{$N=32$}
\addplot[somaE,mark repeat=8,mark size=1pt] table[x=wall_clock_hr,y=val,col sep=comma]{figdata/figure1_left_equal_time/N256.csv};\addlegendentry{$N=256$}
\nextgroupplot[title={\shortstack{Perturbations ($n_{\rm pert}$)\\($N=8$, $B=64$)}},yticklabels=\empty,ytick style={draw=none}]
\addplot[somaA,mark repeat=8,mark size=1pt] table[x=wall_clock_hr,y=val,col sep=comma]{figdata/measured_step_time/np16.csv};\addlegendentry{$n_{\rm pert}=16$}
\addplot[somaB,mark repeat=8,mark size=1pt] table[x=wall_clock_hr,y=val,col sep=comma]{figdata/measured_step_time/N8.csv};\addlegendentry{$n_{\rm pert}=64$}
\addplot[somaC,mark repeat=8,mark size=1pt] table[x=wall_clock_hr,y=val,col sep=comma]{figdata/measured_step_time/np256.csv};\addlegendentry{$n_{\rm pert}=256$}
\addplot[somaD,mark repeat=8,mark size=1pt] table[x=wall_clock_hr,y=val,col sep=comma]{figdata/measured_step_time/np1024.csv};\addlegendentry{$n_{\rm pert}=1024$}
\nextgroupplot[title={\shortstack{Effective batch size ($B$)\\($N=8$, $n_{\rm pert}=64$)}},yticklabels=\empty,ytick style={draw=none}]
\addplot[somaA,mark repeat=8,mark size=1pt] table[x=wall_clock_hr,y=val,col sep=comma]{figdata/measured_step_time/b16.csv};\addlegendentry{$B=16$}
\addplot[somaB,mark repeat=8,mark size=1pt] table[x=wall_clock_hr,y=val,col sep=comma]{figdata/measured_step_time/N8.csv};\addlegendentry{$B=64$}
\addplot[somaC,mark repeat=8,mark size=1pt] table[x=wall_clock_hr,y=val,col sep=comma]{figdata/measured_step_time/ac4.csv};\addlegendentry{$B=256$}
\addplot[somaD,mark repeat=8,mark size=1pt] table[x=wall_clock_hr,y=val,col sep=comma]{figdata/measured_step_time/ac16.csv};\addlegendentry{$B=1024$}
\end{groupplot}
\end{tikzpicture}
\caption{SOMA sweeps on logarithmic wall-clock axes. Left, Fixed-width experts scaling from $N\in\{1,2,8,32,256\}$, shown through a common cutoff of 183.6 estimated hours. Increasing number of experts (and ensemble size) improves training monotonically. Center, $n_{\rm pert}\in\{16,64,256,1024\}$. Right, effective batch $B\in\{16,64,256,1024\}$. The $B=256,1024$ runs accumulate four and sixteen batches of 64, each with 64 independently sampled directions, giving 256 and 1,024 directions per update. SOMA $N=8$ improves loss with more $n_{\rm pert}$ and $B$. Settings are in Table~\ref{tab:experimental-setup} and Appendix~\ref{app:experimental-details}.
}
\label{fig:loss_curves}
\end{figure}
\endgroup

\begin{table}[H]\centering\small
\caption{Experimental settings and evaluation sets.}
\label{tab:experimental-setup}
\begin{center}    
\begin{tabular}{@{}p{0.28\linewidth}p{0.67\linewidth}@{}}
\toprule
Quantity & Configuration \\
\midrule
Training corpus / tokenizer & 100B FineWeb-Edu bytes / byte vocabulary $V=256$ \\
Context length & 1,024 tokens \\
Expert validation set & The 1\% split within each expert's shard. The other 99\% is used for training. Expert validation loss controls learning-rate decay. \\
Ensemble validation set & 976 windows from 202 FineWeb-Edu documents, used for scaling curves. Scores are ensemble validation loss. \\
Test set & 4,096 FineWeb-Edu documents screened for training overlap. Each contributes one 1,025-byte window. Scores are test loss. Construction and exclusions are in Appendix~\ref{app:independent-evaluation-data}. \\
Evaluation protocol & On the ensemble validation and test sets, routing uses the first 256 input bytes and loss scores the next 768 targets. Both losses are mean cross entropy in nats/byte. \\
Expert body & width $d_t=32$, $D=2$ layers, 32.9k body parameters \\
Shared tied head & $d_E=32$, 8.19k parameters, trained in seed and then frozen \\
Baseline SPSA & $n_{\rm pert}=64$, $B=64$, probe density 0.5 \\
Perturbation / learning rate & initial $\varepsilon=\mathrm{lr}=10^{-3}$, decayed together on expert validation plateaus \\
Adam moments / weight decay & $(\beta_1,\beta_2)=(0.9,0.999)$ / $10^{-4}$ \\
Random seed / training precision & 1 / FP32 weights, TF32 kernels \\
Plateau rule / floor & halve after 1,000 updates without improvement of $10^{-8}$ / $10^{-5}$  \\
Expert-count sweep & $N\in\{1,2,8,32,256\}$ \\
Perturbation sweep & $n_{\rm pert}\in\{16,64,256,1024\}$ at $N=8,B=64$ \\
Effective-batch sweep & $B\in\{16,64,256,1024\}$ at $N=8,n_{\rm pert}=64$ \\
Training placement & one independent expert per RTX~5090 \\
\bottomrule
\end{tabular}
\end{center}
\end{table}

We vary $N$, $n_{\rm pert}$, and $B$ as listed in Table~\ref{tab:experimental-setup}. Figure~\ref{fig:loss_curves} displays the primary sweeps through 1024. The default setting is $N=8$, $n_{\rm pert}=64$, and $B=64$. At four million updates per expert, increasing $N$ from 1 to 256 lowers ensemble validation loss from 1.86 to 1.72. The $N=1$ curve continues beyond four million updates to 1.84 at the approximately 184-hour cutoff used in Section~\ref{sec:exp}.
The best endpoint in that figure, $N=8$ with $n_{\rm pert}=1024$, reaches 1.69 nats/byte. Table~\ref{tab:experimental-setup}
gives the full setup. These runs have different update budgets and no recorded common stopping rule. Their ordering could change with further training.

We use the fixed ensemble validation set and the protocol in Table~\ref{tab:experimental-setup}. Hidden state is reset between chunks. The same evaluation applies to $N=1$ and optimizer baselines. Section~\ref{sec:independent-evaluation} evaluates frozen ensembles on the test set. The training sweeps and LSTM controls are re-evaluated in CPU fp32 from archived weights, using exact GELU. The N=256 endpoint agrees with the historical GPU score within $10^{-4}$ nats/byte. The fixed-parameter SOMA curves use the recovered training evaluator.

The 100B figure is full corpus size that is sharded. Each expert consumes $100B/N$ bytes. At 4.16M updates, $B=64$ and $T=1024$ give $2.73\times10^{11}$ bytes per expert. For a balanced 99\%-training split over 256 shards, that is approximately 705 passes over a 387M-byte training shard.

\subsection{Wall-clock time}
To compare progress over a training duration, we place the recorded loss measurements on an estimated elapsed-time axis. For Figure~\ref{fig:loss_curves}, checkpoint updates are converted to training hours using measured mean seconds per step from the matching W\&B runs. We use the slowest expert's mean for each ensemble. The timing samples cover the plotted training interval after the first 1,000 updates and exclude later replays of earlier checkpoints. This estimates training time at the observed rate without adding deployment delays or offline interruptions. All curves use recorded ensemble validation losses. No loss is interpolated or fitted. The historical ensemble validation values agree with 33 matching checkpoint re-evaluations within $0.0001$ nats/byte.

The measured seconds per step are $0.138, 0.169, 0.145, 0.157, 0.165$ for $N=1,2,8,32,256$, respectively. At $N=8$, they are $0.063, 0.145, 0.192, 0.409$ for $n_{\rm pert}=16,64,256,1024$ and $0.035, 0.145, 0.221, 0.819$ for $B=16,64,256,1024$. The larger batches use accumulation, as described in Appendix~\ref{app:training-details}.

\subsection{Aggregate training work}
Adding experts at fixed width grows both total model size and work per update. We compare their learning curves against aggregate GPU-hours, summing the recorded training time across experts.

\begin{figure}[H]\centering
\begin{tikzpicture}
\begin{axis}[
width=.84\linewidth,height=5.3cm,scale only axis,xmode=log,
 title={Expert width $32$, $n_{\rm pert}=B=64$},title style={font=\small,yshift=22pt},
xlabel={Aggregate GPU-hours},ylabel={\shortstack{Ensemble validation loss\\(nats/byte)}},
ymin=1.67,ymax=2.63,grid=major,grid style={gray!22},
tick label style={font=\small},label style={font=\small},
legend style={font=\scriptsize,draw=none,legend columns=5,at={(.5,1.03)},anchor=south}]
\addplot[somaA,thick,mark=*,mark size=1.5pt] table[x=aggregate_gpu_hours,y=val,col sep=comma]{figdata/aggregate_gpu_hours/N1.csv};\addlegendentry{$N=1$}
\addplot[somaB,thick,mark=*,mark size=1.5pt] table[x=aggregate_gpu_hours,y=val,col sep=comma]{figdata/aggregate_gpu_hours/N2.csv};\addlegendentry{$N=2$}
\addplot[somaC,thick,mark=*,mark size=1.5pt] table[x=aggregate_gpu_hours,y=val,col sep=comma]{figdata/aggregate_gpu_hours/N8.csv};\addlegendentry{$N=8$}
\addplot[somaD,thick,mark=*,mark size=1.5pt] table[x=aggregate_gpu_hours,y=val,col sep=comma]{figdata/aggregate_gpu_hours/N32.csv};\addlegendentry{$N=32$}
\addplot[somaE,thick,mark=*,mark size=1.5pt] table[x=aggregate_gpu_hours,y=val,col sep=comma]{figdata/aggregate_gpu_hours/N256.csv};\addlegendentry{$N=256$}
\end{axis}
\end{tikzpicture}
\caption{More experts do not lower loss at every training budget.
These are the expert-count measurements on the ensemble validation set from Figure~\ref{fig:loss_curves}, plotted against aggregate GPU-hours estimated from the recorded mean step times across all experts. Expert width is fixed and $n_{\rm pert}=B=64$. Greater $N$ adds parameters and costs more per update. Section~\ref{sec:independent-evaluation} compares the perturbation and batch allocations at fixed model size.}
\label{fig:loss-curves-total-device-time}
\end{figure}
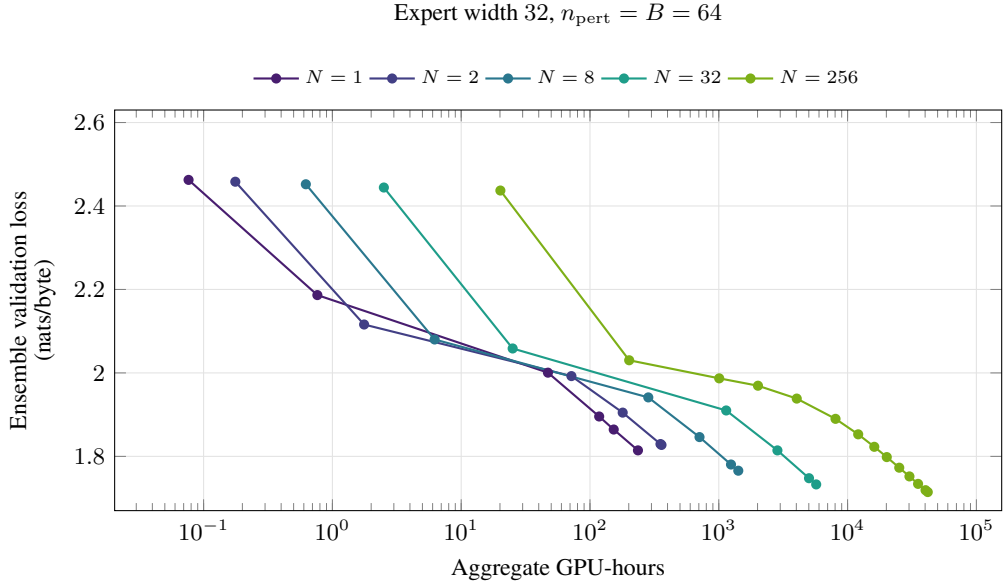

\section{Further related work}\label{app:related-work}
Separating experts affects communication, gradient estimation and model execution. Beyond the closest predecessors discussed in Section~\ref{sec:related-work}, we compare methods that address these costs without using the same decomposition.

BTX converts separately trained branches into a jointly post-trained mixture of experts~\citep{sukhbaatar2024btx}. Federated averaging combines locally trained model updates~\citep{mcmahan2017}, and DiLoCo amortizes communication over local gradient updates~\citep{douillard2023diloco}. SOMA has no model averaging during the expert phase. Each expert and its optimizer state remain local, and predictions are combined at inference.

MeZO-SVRG reduces estimator variance using stochastic variance-reduced gradient updates during language-model fine-tuning~\citep{gautam2024mezosvrg}. Sparse MeZO instead perturbs selected parameters within an existing model~\citep{liu2025sparsemezo}. Restricting the updated coordinates does not by itself reduce the parameters executed at inference. These methods study fine-tuning a pretrained model, whereas SOMA changes the architecture and local objectives used for pretraining. Our experiments do not compare against these two methods.

DeepZero trains deep networks from scratch with sparse coordinate-wise finite differences~\citep{chen2024deepzero}. ZO-Muon combines subspace estimation with spectral orthogonalization~\citep{lang2026zomuon}. FOGZO injects finite-difference information into straight-through gradients for quantization-aware pretraining~\citep{yang2025fogzo}. Evolution strategies can distribute perturbations through shared randomness and scalar returns~\citep{salimans2017}.

Forward gradients use directional automatic differentiation~\citep{baydin2022forward}, while UORO estimates online recurrent derivatives without storing an entire BPTT graph~\citep{tallec2018uoro}. Both require derivatives. SOMA uses scalar body-loss evaluations. ByT5, MEGABYTE and BLT provide related byte-level modeling approaches~\citep{xue2022byt5,yu2023megabyte,pagnoni2024blt}.

\section{Architecture and parameter counts}\label{app:architecture-details}

The training and inference comparisons depend on which parameters are shared, perturbed and executed for each prediction. We specify these components here so that total model size and active model size can be counted consistently. The width-32 ensembles use a shared, weight-tied embedding/decode matrix $E$, $N$ recurrent experts $f_1,\dots,f_N$, and the fixed tf--idf cluster router of Section~\ref{sec:recipe}. Each selected assembly emits the next-token distribution. The router selects up to four assemblies from the observed context before they run. We use recurrent experts instead of transformers. A transformer cache requires $O(Td_t)$ memory per
layer and active expert for context length $T$, while dense-attention prefill requires $O(T^2d_t)$ compute.
An LSTM carries $O(d_t)$ recurrent state per layer, independent of context length. Model size grows through the number of bodies $N$ and
the expert width and number of layers (Fig.~\ref{fig:arch}). In SOMA, clustering supplies meaningful
specialization signal without per-expert projections, decorrelation losses, or a jointly trained gate. The bodies
are architecturally identical small LSTMs that differ in the data they see.

We use one token per UTF-8 input byte with a 256-class vocabulary rather than tiktoken's GPT-2-compatible
\texttt{r50k\_base} vocabulary of 50{,}257 tokens~\citep{openaitiktoken}. The dense vocabulary projection
therefore has 196$\times$ fewer output classes and 196$\times$ less decoder compute per model token at fixed width, which matters in ZO
because we project $2n_{\rm pert}B$ perturbed sequences per step. Byte sequences are longer, so this is not an end-to-end saving per unit of text.
Let $V=256$ be vocabulary size and $d_E$ the embedding width. We encode and decode with one matrix $E\in\R^{V\times d_E}$, with $d_E=d_t$ in the $d_t=32$ configuration. It is used twice and weight-tied as the input embedding
(the row $E_{\text{token}}\in\R^{d_t}$) and as the output decode head
(for body hidden state $h\in\R^{d_t}$, $\mathrm{logits}=h\,E^\top$). Because $d_E{=}d_t$, each expert
decodes its hidden state directly, with no per-expert projections. We evaluate prediction quality within 1,024-byte windows. The
token-to-logit map is learned \emph{once}, in the separate seed run, by
the local decoder-path rule of \S\ref{sec:train}. It is then \emph{frozen} and copied into every
expert, so the experts share one fixed decode. Adding additional experts adds only expert-local
parameters, never re-paying the $V\,d_E$ head cost nor re-learning the token vocabulary.

Let $D$ count residual blocks. The recurrent body uses $D{=}2$ pre-norm residual blocks of width
$d_t{=}32$, each an LSTM sublayer followed by a multi-layer perceptron (MLP) with a Gaussian error linear unit (GELU) nonlinearity (hidden $4d_t$), with gain-only LayerNorm and no
input/output projections because $d_E{=}d_t$. The final hidden state is decoded by $E^\top$
into next-token cross entropy. Each block contributes $16d_t^2$ weights and $2d_t$ normalization gains, followed by one final normalization. For $D$ blocks,
\[
M_{\rm body}(D,d_t)=16D d_t^2+(2D+1)d_t.
\]
With $D=2$, this gives $M_{\rm body}=32d_t^2+5d_t$ at $d_E=d_t$, while the shared tied head contains
$M_{\rm head}(d_t)=Vd_t$. Thus
$M_{\rm ensemble}=M_{\rm head}+NM_{\rm body}\approx8.44\mathrm{M}$ at $N=256$. Each expert body contains 32.9k parameters. Each of the two bias-free LSTM layers contributes $8d_t^2$ weights and each MLP contributes $8d_t^2$. Four block normalization gains and the final normalization add $5d_t$ parameters. The wider fixed model size controls add $2d_Ed_t$ projection weights. Their heads are tied between input and output within an expert, but updated locally and distinct across experts. We therefore count $N$ heads for those controls, giving $M_{\rm ensemble}=N(32d_t^2+5d_t+2d_Ed_t+Vd_E)$ for $d_E=32$ and $d_t\ne d_E$. The shared-head formula applies to the width-32 ensembles. Fitted tf--idf, SVD, and centroid state is separate from these neural counts. The SVD has up to $50{,}000\times128=6.4$M fitted coefficients.

These counting rules determine the expert widths used to keep the models near 8.44M total parameters in Figure~\ref{fig:independent-evaluation}. Table~\ref{tab:fixed-parameter-counts} lists both the ensemble total and each expert body size, making explicit how sharding changes the local estimation problem.
\begin{table}[H]\centering\small
\caption{Parameter counts for the SPSA sharding comparison.
Model size counts expert bodies and heads. All heads have width $d_E=32$. Wider bodies use input and output projections.
The $N=1,2,8$ heads are updated locally and counted separately. The $N=256$ head is shared and frozen.}
\label{tab:fixed-parameter-counts}
\begin{tabular}{@{}rrrrr@{}}
\toprule
$N$ & $d_t$ & head per expert & ensemble parameters & expert body parameters \\
\midrule
1 & 509 & 8.19k & 8.33M & 8.33M \\
2 & 361 & 8.19k & 8.41M & 4.20M \\
8 & 181 & 8.19k & 8.55M & 1.06M \\
256 & 32 & 8.19k (shared) & 8.44M & 32.9k \\
\bottomrule
\end{tabular}
\end{table}

We route the observed context, with the same tf--idf method used
to shard the training data, to a set $S$ of up to four clusters. Our benchmark assumes 256 observed bytes. Short-prompt generation is not evaluated here and requires an explicit cold-start routing policy. We embed each token with $E$ to $x\in\R^{d_t}$ and
pass the embedded tokens through each selected $\mathrm{LSTM}_1$ and $\mathrm{LSTM}_2$ body, decode each hidden
state $h_i\in\R^{d_t}$ through $E^\top$ to obtain $p_i$, average the selected predictions into the final
$p(\cdot\mid\text{context})$, and sample the next token.

\section{Training recipe details}\label{app:recipe-details}
Independent training requires each expert to receive a useful specialization task without relying on updates from the other experts. The seed run supplies the initial byte representation, and clustering determines both the training assignment and the inference route. The settings below specify these two stages of the recipe in Section~\ref{sec:recipe}.

The 1,000-update seed run uses $n_{\rm pert}=B=64$ and initial learning rate 0.0025 on the 10B-byte subset. This samples 65.5M byte positions before repeated perturbation evaluations. Appendix~\ref{app:training-details} specifies the body and decoder updates.

The 100B-byte corpus is split into non-overlapping 1,024-byte sequences, giving approximately 97.7M sequences. The word-level $\mathrm{TfidfVectorizer}$ uses \texttt{max\_features}=50000, sublinear term frequencies, English stopwords and \texttt{min\_df}=5. We reduce to 128 dimensions with truncated SVD, subtract the mean and $L_2$-normalize. Balanced spherical $k$-means fits $N$ centroids, and each sequence is assigned to its nearest centroid. The builder default uses 400,000 sample chunks, consistent with the saved IDF values. Both the clustering code and the evaluation router decode byte windows with UTF-8 and \texttt{errors="ignore"}. Incomplete or invalid UTF-8 sequences are omitted from the router's text features, while the language model retains and scores the original bytes. The artifact supplies the exact fitted routers.

Clustering precedes the 99/1 expert training and expert validation split. Expert validation loss drives learning-rate decay. Training assignment uses each full 1,024-byte window. At evaluation, the same fitted vectorizer, SVD and centroids see only the 256 observed prefix bytes before selecting experts, and loss is scored on the following 768 bytes. The router therefore has less context than was used for training assignment and never reads the scored continuation.

In the width-32 ensembles, each expert copies the seed body and tied embedding/decoder $E$. The body continues training on its cluster while $E$ remains frozen. The wider controls retain expert-local decoder updates, as specified in Appendix~\ref{app:figure2-controls} and Table~\ref{tab:fixed-parameter-counts}.

\label{app:additional-ablations}
\subsection{Architecture ablations}\label{app:original-ablations}

The recipe assigns different data to experts and gives them a learned starting representation. To check which of these choices helps prediction, we vary partitioning, initialization and head training separately. These BPTT controls use $B=512$ and the same ensemble validation protocol.

\begin{figure}[!htbp]
\centering
\begin{tikzpicture}
\begin{groupplot}[
 group style={group size=3 by 1,horizontal sep=.95cm},
 width=.245\linewidth,height=4.25cm,scale only axis,
 ylabel={\shortstack{Ensemble validation loss\\(nats/byte)}},
 xticklabel style={font=\scriptsize,align=center},
 yticklabel style={font=\scriptsize},label style={font=\scriptsize},
 title style={font=\small,align=center},
 ymajorgrids,grid style={gray!20},
 tick align=outside,/tikz/mark size=2.4pt,
 nodes near coords,point meta=y,
 every node near coord/.append style={font=\scriptsize,anchor=south,yshift=3pt,/pgf/number format/fixed,/pgf/number format/precision=3,/pgf/number format/fixed zerofill},
 error bars/y dir=both,error bars/y explicit]
\nextgroupplot[
 title={Partitioning\\$N=32$, 180k updates},
 xmin=.5,xmax=3.5,ymin=1.36,ymax=1.47,
 xtick={1,2,3},xticklabels={Semantic,\shortstack{Random\\disjoint},\shortstack{Global\\IID}}]
\addplot+[only marks,color=somaC,mark=*,mark options={fill=somaC,draw=somaC},error bars/error bar style={draw=somaC},restrict x to domain=1:1] table[x=row,y=mean_loss,y error=standard_error,col sep=comma]{figdata/plot_clarity_v2/partition_ablation.csv};
\addplot+[only marks,color=somaOrange,mark=*,mark options={fill=somaOrange,draw=somaOrange},error bars/error bar style={draw=somaOrange},restrict x to domain=2:2] table[x=row,y=mean_loss,y error=standard_error,col sep=comma]{figdata/plot_clarity_v2/partition_ablation.csv};
\addplot+[only marks,color=somaE,mark=*,mark options={fill=somaE,draw=somaE},error bars/error bar style={draw=somaE},restrict x to domain=3:3] table[x=row,y=mean_loss,y error=standard_error,col sep=comma]{figdata/plot_clarity_v2/partition_ablation.csv};
\nextgroupplot[
 title={Initialization\\$N=8$, 850k updates},ylabel={},
 xmin=.5,xmax=2.5,ymin=1.405,ymax=1.447,
 xtick={1,2},xticklabels={Warm,Cold}]
\addplot+[only marks,color=somaC,mark=*,mark options={fill=somaC,draw=somaC},error bars/error bar style={draw=somaC},restrict x to domain=1:1] table[x=row,y=mean_loss,y error=standard_error,col sep=comma]{figdata/plot_clarity_v2/initialization_ablation.csv};
\addplot+[only marks,color=somaOrange,mark=*,mark options={fill=somaOrange,draw=somaOrange},error bars/error bar style={draw=somaOrange},restrict x to domain=2:2] table[x=row,y=mean_loss,y error=standard_error,col sep=comma]{figdata/plot_clarity_v2/initialization_ablation.csv};
\nextgroupplot[
 title={Head training\\$N=1$, 55k updates},ylabel={},
 xmin=.5,xmax=3.5,ymin=1.4,ymax=2.4,
 xtick={1,2,3},xticklabels={\shortstack{Train\\encoder\\and decoder},\shortstack{Train\\decoder\\only},\shortstack{Freeze\\random\\head}}]
\addplot+[only marks,color=somaD,mark=*,mark options={fill=somaD,draw=somaD},error bars/error bar style={draw=somaD},restrict x to domain=1:1] table[x=row,y=mean_loss,y error=standard_error,col sep=comma]{figdata/plot_clarity_v2/head_ablation.csv};
\addplot+[only marks,color=somaC,mark=*,mark options={fill=somaC,draw=somaC},error bars/error bar style={draw=somaC},restrict x to domain=2:2] table[x=row,y=mean_loss,y error=standard_error,col sep=comma]{figdata/plot_clarity_v2/head_ablation.csv};
\addplot+[only marks,color=somaA,mark=*,mark options={fill=somaA,draw=somaA},error bars/error bar style={draw=somaA},restrict x to domain=3:3] table[x=row,y=mean_loss,y error=standard_error,col sep=comma]{figdata/plot_clarity_v2/head_ablation.csv};
\end{groupplot}
\node[font=\small,anchor=south] at ([yshift=1.1cm]group c2r1.north) {BPTT, $B=512$, expert width 32};
\end{tikzpicture}
\caption{Semantic shards and warm initialization lower loss. Means and standard errors over three seeds for each condition at the shown training steps. The seed model's training cost is excluded. Updating the tied head only through the decoder raises loss by 0.029 nats/byte, while freezing a random head performs much worse.}
\label{fig:original-ablations}
\end{figure}
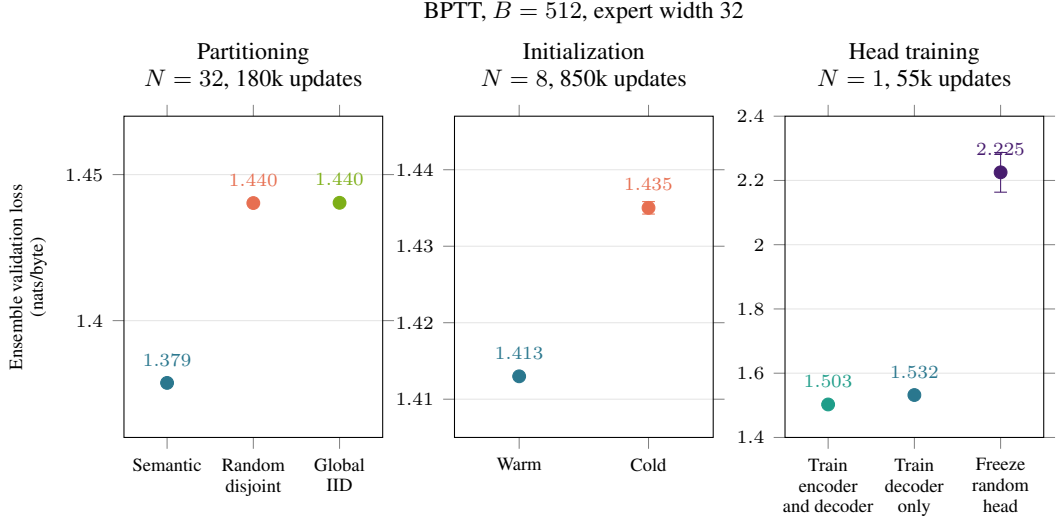

These controls support semantic sharding and warm initialization. Decoder-only training retains most of the benefit of training the tied head.

\section{Training updates}\label{app:training-details}

The update rule separates learning the byte decoder from estimating recurrent-body gradients. This keeps the vocabulary matrix out of the perturbation and limits the dimension of each SPSA estimate. We first give the local decoder update, then specify how body estimates are accumulated and passed to Adam. Let $L_{\mathrm{dec}}$ be mean cross entropy over $B$ sequences and $T$ target positions. Write $E\in\R^{V\times d_E}$ and the decoder input $h_{b,t}\in\R^{d_E}$ as a column vector for sequence $b$ at position $t$. Let $y_{b,t+1}$ be the next byte, $\mathrm{onehot}(y)\in\R^V$ its target indicator, and $p_{b,t}=\softmax(Eh_{b,t})\in\R^V$ the predicted probability vector. Holding $h_{b,t}$ fixed, the direct decoder-path derivative is
\begin{equation}
\nabla_E L_{\mathrm{dec}}
=\frac{1}{BT}\sum_{b=1}^{B}\sum_{t=1}^{T}
\big(p_{b,t}-\mathrm{onehot}(y_{b,t+1})\big)h_{b,t}^{\top}
\in\R^{V\times d_E},
\label{eq:delta}
\end{equation}
Because $E$ is weight-tied, applying the update changes
the same values used on the input side, but Eq.~\eqref{eq:delta} deliberately omits the indirect derivative
through the input embedding and recurrent body. It is therefore an exact decoder-path delta rule, not the full
gradient of a tied embedding/decode matrix. Computing this local expression does not invoke reverse-mode
automatic differentiation or propagate derivatives through the recurrent network. The outer product can
be accumulated from the current forward hidden states and output residuals without retaining a recurrent
activation tape. $E$ is never SPSA-perturbed.

Under \soma{} each expert optimizes a single local objective,
the next-token cross entropy on its own cluster, with no coupling to the others. Let
$\theta\in\R^d$ denote that expert's $d$ perturbed body parameters and let $E$ denote its
embedding/decoder, held fixed while forming the body estimate. For a minibatch $\mathcal B$, write $L_{\mathcal B}(\theta,E)$ for its mean
next-token cross entropy and $L(\theta)=\E_{\mathcal B}[L_{\mathcal B}(\theta,E)]$ for the corresponding
cluster objective. For each batch we independently sample $n_{\rm pert}$ sparse Rademacher probes $z_j$. Normalization-gain coordinates are −1 or +1 with equal probability; every other coordinate is zero with probability 1/2 and −1 or +1 with probability 1/4 each. The reference theory uses dense Rademacher probes. The measurements in Appendix~\ref{app:estimator-iso} use the recorded perturbation distribution.
Within each batch, all $n_{\rm pert}$ central-difference pairs use the same data, giving the estimate $\widehat g_{n_{\rm pert}}(\theta)$. At update $t\ge1$, $\widehat g_t$ averages these estimates over $A$ independently sampled accumulation batches and directions. For parameters $\theta_t$ and learning rate $\eta_t$, Adam-style moments~\citep{kingma2015adam} $m_t,v_t$ track the first and second moments with decay factors $\beta_1,\beta_2$. Their bias-corrected values are $\bar m_t,\bar v_t$, and $\delta_{\rm Adam}$ stabilizes division. Starting with $m_0=v_0=0$, the update is
\begin{equation}
\begin{gathered}
\widehat g_{n_{\rm pert}}(\theta)
=\frac{1}{n_{\rm pert}}\sum_{j=1}^{n_{\rm pert}}
\frac{L_{\mathcal B}(\theta+\varepsilon z_j,E)
      -L_{\mathcal B}(\theta-\varepsilon z_j,E)}{2\varepsilon}\,z_j,\\
m_t=\beta_1m_{t-1}+(1-\beta_1)\widehat g_t,\qquad
v_t=\beta_2v_{t-1}+(1-\beta_2)\widehat g_t^{2},\\
\bar m_t=\frac{m_t}{1-\beta_1^t},\qquad
\bar v_t=\frac{v_t}{1-\beta_2^t},\qquad
\theta_{t+1}=\theta_t-\eta_t\frac{\bar m_t}{\sqrt{\bar v_t}+\delta_{\rm Adam}}.
\end{gathered}
\label{eq:localspsa}
\end{equation}
Squares, square roots and division act coordinatewise. Accumulation gives $An_{\rm pert}$ directions per update. The $B=256,1024$ runs use four and sixteen batches of 64. We use $\beta_1=0.9$ and $\beta_2=0.999$.
Equation~\eqref{eq:localspsa} shows the update before weight decay and follows
the adaptive-momentum ZO family~\citep{chen2019zoadamm}. For coupled weight decay $\lambda$, the implementation adds $\lambda\theta$ before the moments. The wide SPSA controls use $\lambda=10^{-4}$. If $q_j$ is the probability that probe coordinate $j$ is nonzero and $g_j=\partial L/\partial\theta_j$, its expected input to Adam is $q_jg_j+\lambda\theta_j+O(\varepsilon^2)$. Here $q_j=1/2$ for ordinary weights and $q_j=1$ for normalization gains, with no inverse-probability rescaling. Thus coupled decay is twice as large relative to the expected data-gradient contribution on ordinary weights. This is a pre-Adam statement, not an exact multiplier of Adam's effective learning rate. One update uses $2An_{\rm pert}$ perturbed forward evaluations.

For a hidden state $h$ and target byte $y$, let $e_v$ denote vocabulary row $v$ of $E$. The width-32 ensembles train the body on the full-vocabulary next-token cross entropy
$L_{\mathrm{CE}}(h)=\log\sum_v e^{\,e_v\cdot h}-e_y\cdot h$, evaluated for each of the
$2n_{\mathrm{pert}}$ perturbed bodies per accumulation batch against the \emph{frozen} tied head $E$. SPSA reads the
directional derivative from the same $\pm$ pair (Eq.~\eqref{eq:localspsa}). Because $E$ never enters the
perturbation, this is a local SPSA problem on 32.9k expert body parameters, and the head adds no perturbed coordinates.

The body update uses local scalar losses. The width-32 ensemble recipe applies the decoder-path update only during seed training. The wide $N=1,2,8$ controls continue applying it to each expert's own head, without sharing head updates across experts.

\section{Gradient variance derivations}
\label{app:block-sharding-proof}\label{app:background}
Local losses exclude other experts' perturbations from each estimate. We derive how this changes relative variance at fixed total model size, then examine its implications for update cost and convergence on a separable objective. The subsequent batch calculation distinguishes averaging more data from averaging more perturbations.

Let \(\mathcal D=\bigsqcup_{e=1}^{N}\mathcal D_e\) be the partition of the training corpus among
\(N\) experts. Let \(M\) be the total number of perturbed coordinates. We compare one monolithic LSTM with \(\theta^{\rm mono}\in\R^M\) against a SOMA ensemble with these coordinates divided equally as
\(\theta^{(e)}\in\R^{M/N}\). For a fixed minibatch \(\mathcal B\) of effective size \(B\), write the
monolithic loss as \(L_{\mathcal B}^{\rm mono}(\theta^{\rm mono})\). For expert \(e\), let
\(\mathcal B_e\subset\mathcal D_e\) denote its fixed local minibatch and write its loss as
\(L_{\mathcal B_e}^{(e)}(\theta^{(e)})\). The calculation below is conditional on these minibatches:
\(B\) controls ordinary sampling noise, whereas the result isolates variance from the random SPSA
directions.

For \(\varepsilon>0\), the monolithic gradient estimate \(\widehat g^{\rm mono}\) uses independent Rademacher probes \(z_i\in\{-1,+1\}^{M}\), each with independent, equally likely signs,
\begin{equation}
\widehat g^{\rm mono}
=\frac{1}{n_{\rm pert}}\sum_{i=1}^{n_{\rm pert}}
\frac{L_{\mathcal B}^{\rm mono}(\theta^{\rm mono}+\varepsilon z_i)
-L_{\mathcal B}^{\rm mono}(\theta^{\rm mono}-\varepsilon z_i)}
{2\varepsilon}\,z_i .
\label{eq:appendix-monolithic-estimator}
\end{equation}
Expert \(e\)'s gradient estimate \(\widehat g^{(e)}\) independently uses Rademacher probes \(z_i^{(e)}\in\{-1,+1\}^{M/N}\),
\begin{equation}
\widehat g^{(e)}
=\frac{1}{n_{\rm pert}}\sum_{i=1}^{n_{\rm pert}}
\frac{L_{\mathcal B_e}^{(e)}(\theta^{(e)}+\varepsilon z_i^{(e)})
-L_{\mathcal B_e}^{(e)}(\theta^{(e)}-\varepsilon z_i^{(e)})}
{2\varepsilon}\,z_i^{(e)} .
\label{eq:appendix-expert-estimator}
\end{equation}
Define
\(g^{\rm mono}:=\nabla_{\theta^{\rm mono}}L_{\mathcal B}^{\rm mono}\),
\(g^{(e)}:=\nabla_{\theta^{(e)}}L_{\mathcal B_e}^{(e)}\). Let \(\widehat G^{\rm SOMA}\) and \(G^{\rm SOMA}\) concatenate the estimated and exact expert gradients,
\[
\widehat G^{\rm SOMA}
:=\operatorname{concat}(\widehat g^{(1)},\ldots,\widehat g^{(N)}),
\qquad
G^{\rm SOMA}
:=\operatorname{concat}(g^{(1)},\ldots,g^{(N)}).
\]

\begin{proposition}[Gradient variance]
\label{prop:block-variance}
Condition on fixed minibatches and nonzero gradients, with \(\E_z\) averaging over the probes. For losses with three continuous derivatives, to leading order as \(\varepsilon\to0\),
\begin{align}
\frac{\E_z\lVert\widehat g^{\rm mono}-g^{\rm mono}\rVert_2^2}
{\lVert g^{\rm mono}\rVert_2^2}
&=\frac{M-1}{n_{\rm pert}}+O(\varepsilon^2), \\
\frac{\E_z\lVert\widehat G^{\rm SOMA}-G^{\rm SOMA}\rVert_2^2}
{\lVert G^{\rm SOMA}\rVert_2^2}
&=\frac{M/N-1}{n_{\rm pert}}+O(\varepsilon^2).
\label{eq:block-variance}
\end{align}
Consequently, the ratio of their leading relative perturbation variances is
\begin{equation}
\frac{M/N-1}{M-1}
\xrightarrow[M/N\to\infty]{N\ {\rm fixed}}\frac{1}{N}.
\label{eq:matched-compute-variance}
\end{equation}
\end{proposition}

\begin{proof}
Consider either estimator in a parameter space of dimension equal to \(M\) for the monolith or
\(M/N\) for one expert. Write its loss as \(L(\theta)\), its exact gradient as \(g=\nabla L(\theta)\), and a Rademacher direction as \(z_i\). A central Taylor expansion gives
\[
\frac{L(\theta+\varepsilon z_i)-L(\theta-\varepsilon z_i)}
{2\varepsilon}\,z_i
=z_i z_i^\top g+O(\varepsilon^2).
\]
For coordinate \(j\), the leading estimation error from one direction is
\[
(z_i z_i^\top g-g)_j
=z_{i,j}\sum_{k\ne j}z_{i,k}g_k.
\]
The independent Rademacher coordinates make this expression mean zero and give
\[
\E_z\!\left[(z_i z_i^\top g-g)_j^2\right]
=\sum_{k\ne j}g_k^2.
\]
Summing over coordinates yields \((M-1)\lVert g\rVert_2^2\) for the monolith and
\((M/N-1)\lVert g\rVert_2^2\) for one expert. Averaging \(n_{\rm pert}\) independent directions
divides each variance by \(n_{\rm pert}\). Finally, summing the independent expert errors gives
\[
\sum_{e=1}^{N}\E_z\lVert\widehat g^{(e)}-g^{(e)}\rVert_2^2
=\frac{M/N-1}{n_{\rm pert}}\sum_{e=1}^{N}\lVert g^{(e)}\rVert_2^2
+O(\varepsilon^2),
\]
which is the stated SOMA identity because
\(\lVert G^{\rm SOMA}\rVert_2^2=\sum_e\lVert g^{(e)}\rVert_2^2\).
\end{proof}

The same variance ratio also holds on one fixed separable objective $F=\sum_eL_e$. Using one summed loss estimates an $M$-dimensional gradient. Using each expert's own loss estimates $N$ local gradients on the same model and minibatches. This isolates the estimator benefit from changes in model size or data partitioning. Figure~\ref{fig:theory-empirical}, left, checks this comparison using weights from training.

The comparison is also matched in parameter-forward compute under the usual linear-cost model. One
monolithic central-difference pair evaluates \(M\) parameters twice, for cost \(2M\). One SOMA round
evaluates \(N\) experts with \(M/N\) parameters twice, for the same aggregate cost
\(N\cdot2(M/N)=2M\). Thus both sides of
Eq.~\eqref{eq:matched-compute-variance} use the same \(n_{\rm pert}\) at equal idealized compute.

\begin{corollary}[Perturbation-limited convergence rate]
\label{cor:block-convergence}
Suppose the separable SOMA training objective with fixed local minibatches $\mathcal B_e$
\begin{equation}
L_{\rm train}(\theta^{(1)},\ldots,\theta^{(N)})
:=\sum_{e=1}^{N}L_{\mathcal B_e}^{(e)}(\theta^{(e)})
\label{eq:separable-training-objective}
\end{equation}
has gradient Lipschitz constant \(L_{\rm sm}\) and lower bound \(L_*\). Let \(\theta_k\) concatenate the expert parameters at update \(k\), starting from \(\theta_0\), and let \(\widehat G_k^{\rm SOMA}\) be their gradient estimate. In the limit \(\varepsilon\to0\), use constant step size \(\eta\) in \(\theta_{k+1}=\theta_k-\eta\widehat G_k^{\rm SOMA}\), with independently sampled perturbations at every update. If
\[
0<\eta\le
\frac{1}{L_{\rm sm}\left(1+(M/N-1)/n_{\rm pert}\right)},
\]
then over \(T\) updates, with expectation over all sampled perturbations,
\begin{equation}
\frac{1}{T}\sum_{k=0}^{T-1}
\E\lVert\nabla L_{\rm train}(\theta_k)\rVert_2^2
\le
\frac{2\big(L_{\rm train}(\theta_0)-L_*\big)}{\eta T}.
\label{eq:block-stationarity-rate}
\end{equation}
With the largest allowed step size, the dimension-dependent multiplier is
\(1+(M/N-1)/n_{\rm pert}\), compared with
\(1+(M-1)/n_{\rm pert}\) for the monolithic estimator. Their ratio is
\begin{equation}
\frac{1+(M/N-1)/n_{\rm pert}}
{1+(M-1)/n_{\rm pert}}
\xrightarrow[M/n_{\rm pert}\to\infty]{N\ {\rm fixed}}\frac{1}{N}.
\label{eq:block-convergence-ratio}
\end{equation}
\end{corollary}

\begin{proof}
Proposition~\ref{prop:block-variance} gives
\[
\E_z\lVert\widehat G_k^{\rm SOMA}\rVert_2^2
=\left(1+\frac{M/N-1}{n_{\rm pert}}\right)
\lVert\nabla L_{\rm train}(\theta_k)\rVert_2^2.
\]
Applying the smoothness inequality to the update, taking conditional expectation, and using the step
size restriction gives
\[
\E[L_{\rm train}(\theta_{k+1})\mid\theta_k]
\le L_{\rm train}(\theta_k)
-\frac{\eta}{2}\lVert\nabla L_{\rm train}(\theta_k)\rVert_2^2.
\]
Summing over \(k=0,\ldots,T-1\) and using
\(L_{\rm train}(\theta_T)\ge L_*\) proves
Eq.~\eqref{eq:block-stationarity-rate}. Substituting the largest permitted step size gives
Eq.~\eqref{eq:block-convergence-ratio}.
\end{proof}

Under these assumptions, smaller independent subproblems reduce the variance term in the stationarity bound.
This corollary concerns convergence toward stationarity on each method's own training objective. It
does not compare attainable loss values, and it does not claim that the routed inference objective
improves by \(1/N\). Finite \(\varepsilon\), stochastic minibatches, adaptive momentum, non-linear
hardware costs, and the shared frozen head add effects outside this perturbation-only calculation.

\subsection{Allocating an update budget}\label{app:variance-allocation-theory}
We ask how to divide one update's compute between shards, perturbations and
batches. This supports the total-budget choice in
Section~\ref{sec:independent-evaluation}, where test loss is the goal.

We extend the dense Rademacher calculation in
Appendix~\ref{app:block-sharding-proof} to minibatch gradients and forward cost.
Let $M$ be the
number of perturbed coordinates, and $d=M/N$ the size of each block. We hold the separable objective and weights fixed, with nonzero population gradient $g$, concatenated across experts. We treat $N$ as continuous with
equal block sizes. Integer choices must give feasible partitions.

\subsubsection{Batch size and perturbations}

A larger batch reduces data-sampling variation, while additional perturbations reduce the variation caused by estimating a gradient from a finite set of directions. We separate these contributions to determine how they interact within one update. Let $g_B$ be the concatenated gradient from $B$ independent sequences per
expert, with $\E g_B=g$ and
$\E\|g_B-g\|^2=\sigma^2/B$. Define $\nu=\sigma^2/\|g\|^2$ and $a=d-1$.
Here $\sigma^2$ is the variance across individual example gradients. We hold
$g$ and $\sigma^2$ fixed across parameter groupings. Both can change during
training.
We use the same batch for all $n_{\rm pert}$ independent dense Rademacher probes and
both signs of every probe. Let $\widehat g$ concatenate the resulting expert estimates and $v$ denote their relative gradient variance. As $\varepsilon\to0$, expectation over batches and probes gives
\begin{align}
v(N,n_{\rm pert},B)
&=\frac{\E\|\widehat g-g\|^2}{\|g\|^2}
=\frac{a}{n_{\rm pert}}\frac{\E\|g_B\|^2}{\|g\|^2}+\frac{\nu}{B}\nonumber\\
&=\frac{a}{n_{\rm pert}}+\frac{\nu}{B}+\frac{a\nu}{n_{\rm pert}B}.
\label{eq:vre-variance}
\end{align}
The absolute estimator variance is $\|g\|^2v$. In this unbiased limit, it also equals mean squared gradient error.
If each probe instead gets an independently sampled batch, the variance is
$[a+(a+1)\nu/B]/n_{\rm pert}$.
With a shared batch, increasing $n_{\rm pert}$ reduces perturbation variance
and leaves the minibatch term $\nu/B$ unchanged.

At fixed $N$ and a fixed product $K=n_{\rm pert}B\ge1$, substituting $n_{\rm pert}=K/B$ gives $v=aB/K+\nu/B+a\nu/K$. For $a,\nu>0$, its derivative vanishes at $B=\sqrt{K\nu/a}$. Subject to $B,n_{\rm pert}\ge1$, the continuous optimum is
\begin{equation}
B^*=\min\!\left\{K,\max\!\left\{1,\sqrt{K\nu/a}\right\}\right\},
\qquad n_{\rm pert}^*=K/B^*.
\label{eq:shared-batch-optimum}
\end{equation}
For integer settings, evaluate the feasible $(B,n_{\rm pert})$ pairs under the product budget. This optimizes one update of the dense, shared-batch estimator as $\varepsilon\to0$. It does not optimize loss after training or the independent-batch accumulation schedule below.

\subsubsection{Accumulating independent batches}
The training sweeps implement larger effective batches by accumulating estimates formed with independently sampled data and directions. This changes both sources of variation, so its variance differs from evaluating all directions on one shared large batch. Averaging estimates from $A$ batches of
$b$ sequences, with $n_{\rm pert}$ probes per batch, gives relative variance $v_{\rm acc}$,
\[
v_{\rm acc}=\frac{a}{n_{\rm pert}A}+\frac{\nu}{bA}+\frac{a\nu}{n_{\rm pert}bA}.
\]
Using one batch $B=Ab$ with $n_{\rm pert}$ probes instead leaves the perturbation
term $a/n_{\rm pert}$.

More generally, let $Z(X,\delta)$ be one gradient estimate from batch $X$ and random probe $\delta$, and $\widehat g$ their average over $A$ batches and $n_{\rm pert}$ probes per batch. Subscripts specify which randomness is averaged over. Writing $\operatorname{tr}\operatorname{Cov}$ for the sum of coordinate variances, independent batches and directions give
\begin{equation}
\operatorname{tr}\operatorname{Cov}(\widehat g)
=\frac{\E_X\operatorname{tr}\operatorname{Cov}_{\delta}(Z\mid X)}{A n_{\rm pert}}
+\frac{\operatorname{tr}\operatorname{Cov}_X(\E_{\delta}[Z\mid X])}{A}.
\label{eq:accumulation-variance}
\end{equation}
This identity holds at finite $\varepsilon$ for sampled probes. It measures variance
around the estimator mean, which may be biased. At fixed $n_{\rm pert}B$ and
$b$, the total number of directions $A n_{\rm pert}$ stays fixed. More independently sampled
batches reduce the second term while the first stays unchanged. %

More shards reduce the predicted update variance when experts remain viable
and overhead is small. The best training allocation must also repay its cost
in test loss, as tested in Section~\ref{sec:independent-evaluation}.

\section{Tests of the variance predictions}
The variance identities describe a simplified estimator, while training uses saved weights, finite $\varepsilon$ and a recorded probe distribution. We first check how well the predictions describe those trained models. We then hold the objective fixed to isolate the effects of local losses and the batch--perturbation allocation.

\subsection{Sharding at fixed model size}\label{app:estimator-iso}

To determine whether the variance reduction survives the trained parameter values and numerical implementation, we replay the checkpoints in Figure~\ref{fig:allocation-variance} under both GPU and float64 arithmetic. At 5,000 and 10,000 updates, we fix two
ensemble validation chunks and include every expert and perturbed coordinate in the
roughly 8.44M models. We evaluate the same 128 directions and FP32 parameter endpoints on an
RTX~5090 and in CPU float64. Exact gradients use float64. The saved kernels
and precision settings are used for $N=1,2,8$. For $N=256$, we use the
retained trainer implementation.

For exact batch gradients $g_e$ and averaged estimates $\widehat g_e$, define relative centered gradient variance $R_g$, with covariance over random probes $\delta$,
\[
R_g=\frac{\sum_e\operatorname{tr}\operatorname{Cov}_{\delta}(\widehat g_e)}
{\sum_e\|g_e\|^2}.
\]
We estimate centered covariance from 128 independent single-probe estimates
and divide by $n_{\rm pert}=64$. We retain the recorded perturbation distribution and $\varepsilon$. Coordinate
activity probability is $q_{ej}=1$ for normalization coordinates of expert e and 1/2 otherwise. Let $q_e$ collect these probabilities and $g_{ej}$ denote gradient coordinate $j$.
The leading mean is $\operatorname{diag}(q_e)g_e$, so $R_g$ measures centered
variance, not total gradient error. With $s_e=\sum_jq_{ej}$, the small-$\varepsilon$
prediction is
\[
R_g\simeq\frac{\sum_{e,j}q_{ej}(1+s_e-2q_{ej})g_{ej}^2}
{n_{\rm pert}\sum_e\|g_e\|^2}.
\]
The GPU measurements follow this prediction. At $N=256$, relative gradient
variance is 136 and 118 times lower than at $N=1$ at the two stages. The
float64 ratios are 136 and 121. Sharding retains the measured variance
reduction under this GPU-precision replay.

\label{app:perturbation-scaling}

We test whether more perturbations give the expected variance reduction at
the same weights and batch. Let $X_{e,r}$ be the estimate from probe $r$ for expert $e$, and $a_N=\sum_e\operatorname{tr}\operatorname{Cov}(X_{e,1})/\sum_e\|g_e\|^2$ the ensemble's relative variance before averaging. Independent probes give
\begin{equation}
\widehat g_e=\frac{1}{n_{\rm pert}}\sum_{r=1}^{n_{\rm pert}}X_{e,r},
\qquad
\operatorname{Cov}(\widehat g_e)=\frac{\operatorname{Cov}(X_{e,1})}{n_{\rm pert}},
\qquad R_g=\frac{a_N}{n_{\rm pert}},
\label{eq:perturbation-average}
\end{equation}
The averaging law is exact at finite $\varepsilon$. The prediction of $a_N$ uses
the small-$\varepsilon$ approximation above.

We average the 128 GPU probe vectors in disjoint groups, using
$n_{\rm pert}=1,2,4,8,16,32$. We measure covariance across groups. The measured
variance follows $1/n_{\rm pert}$ at both stages for $N=1,2,8,256$.
Figure~\ref{fig:allocation-variance}, right, shows 10,000 updates. Doubling
$n_{\rm pert}$ halves relative gradient variance and doubles update compute.

\begin{figure}[!htbp]\centering
\begin{tikzpicture}
\begin{groupplot}[
 group style={group size=2 by 1,horizontal sep=1.55cm},
 width=5.5cm,height=4.65cm,scale only axis,
 xmode=log,log basis x=2,ymode=log,
 grid=major,grid style={gray!22},tick pos=left,tick align=outside,
 tick label style={font=\scriptsize},label style={font=\scriptsize},
 title style={font=\small},
 ylabel={Relative gradient variance},
 every axis plot/.append style={line width=.95pt,mark size=2pt},
 legend style={draw=none,fill=none,font=\scriptsize,
 at={(.5,1.03)},anchor=south,legend columns=2,column sep=6pt},
 title style={at={(.5,1.24)},anchor=south,font=\small}]
\nextgroupplot[
 title={More shards\\$n_{\rm pert}=64,\ B=2$},title style={align=center},
 xlabel={Number of experts $N$},xmin=.8,xmax=320,
 xtick={1,2,8,256},xticklabels={$1$,$2$,$8$,$256$},
 ymin=180,ymax=55000,ytick={1000,10000}]
\addplot[somaE,mark=*]
 table[x=N,y=relative_variance,col sep=comma]{figdata/controlled_tests/precision_stage10000.csv};
\addlegendentry{Measured}
\addplot[black,dashed]
 table[x=N,y=predicted_relative_variance,col sep=comma]{figdata/controlled_tests/precision_stage10000.csv};
\addlegendentry{Prediction}
\nextgroupplot[
 title={More perturbations\\$B=2$},title style={align=center},
 xlabel={Perturbations per expert $n_{\rm pert}$},xmin=.8,xmax=40,
 xtick={1,2,4,8,16,32},xticklabels={$1$,$2$,$4$,$8$,$16$,$32$},
 ymin=350,ymax=3500000,ytick={1000,10000,100000,1000000}]
\addplot[somaA,mark=*]
 table[x=npert,y=relative_variance,col sep=comma]{figdata/controlled_tests/precision_perturbation_N1_stage10000.csv};
\addlegendentry{$N=1$}
\addplot[somaB,mark=square*]
 table[x=npert,y=relative_variance,col sep=comma]{figdata/controlled_tests/precision_perturbation_N2_stage10000.csv};
\addlegendentry{$N=2$}
\addplot[somaC,mark=triangle*]
 table[x=npert,y=relative_variance,col sep=comma]{figdata/controlled_tests/precision_perturbation_N8_stage10000.csv};
\addlegendentry{$N=8$}
\addplot[somaE,mark=diamond*]
 table[x=npert,y=relative_variance,col sep=comma]{figdata/controlled_tests/precision_perturbation_N256_stage10000.csv};
\addlegendentry{$N=256$}
\addplot[somaA,dashed,forget plot]
 table[x=npert,y=theory_relative_variance,col sep=comma]{figdata/controlled_tests/precision_perturbation_N1_stage10000.csv};
\addplot[somaB,dashed,forget plot]
 table[x=npert,y=theory_relative_variance,col sep=comma]{figdata/controlled_tests/precision_perturbation_N2_stage10000.csv};
\addplot[somaC,dashed,forget plot]
 table[x=npert,y=theory_relative_variance,col sep=comma]{figdata/controlled_tests/precision_perturbation_N8_stage10000.csv};
\addplot[somaE,dashed,forget plot]
 table[x=npert,y=theory_relative_variance,col sep=comma]{figdata/controlled_tests/precision_perturbation_N256_stage10000.csv};
\end{groupplot}
\end{tikzpicture}
\caption{More shards and more perturbations reduce relative gradient variance.
Lower is better. Centered estimator variance is divided by the squared exact gradient norm. Model size is roughly 8.44M in total (Section~\ref{sec:model-size}).
Left, $N=256$ has $118\times$ lower variance than $N=1$.
Right, doubling $n_{\rm pert}$ halves variance and doubles update compute.
Points use weights at 10,000 updates and the same two ensemble validation chunks on
an RTX~5090. Dashed lines are predictions for the recorded perturbation distribution, without
fitting. These tests measure variance at fixed weights, not learning speed.
Appendix~\ref{app:estimator-iso} gives the calculation.
}
\label{fig:allocation-variance}
\label{fig:estimator-iso}
\label{fig:perturbation-scaling}
\end{figure}
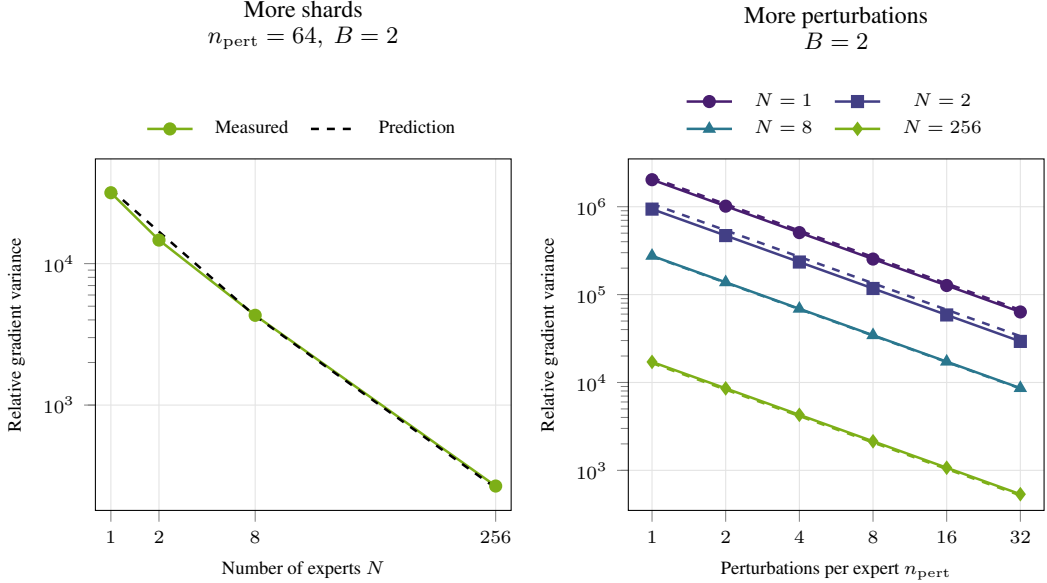
\subsection{Testing the variance predictions}\label{app:theory-empirical}

Comparing different trained ensembles changes more than the loss feedback. We therefore keep one ensemble fixed to test the predicted benefit of local losses and the preferred allocation between batches and perturbations. Both tests use four experts from the $N=256$ ensemble at 900,000 updates, with fixed weights in float64. Relative gradient error is
squared estimation error divided by the squared exact gradient norm.
The small-$\varepsilon$ calculation predicts variance. Finite-difference error
also includes any bias.

In Figure~\ref{fig:theory-empirical}, left, we either sum all four losses
before estimating the gradient or estimate each expert's gradient from its
own loss. We hold the weights and two ensemble validation chunks fixed.

For the batch test, we randomly select 128 coordinates per expert and
32 ensemble validation chunks. Exact sequence gradients give the pool mean gradient $g$ and relative sequence variance $\nu=8.32$, as defined in Appendix~\ref{app:variance-allocation-theory}.
Batches sample this pool with replacement. We use one sum over four experts, two sums over two experts each, or four individual losses. Each loss then supplies an
estimate over 512, 256, or 128 coordinates, respectively. The four experts,
data, and remaining coordinates stay fixed.

We choose $n_{\rm pert}B=256$ from the measured gradient variation before
measuring perturbation errors. All probes and both signs share each batch.
Errors are measured against the pool mean gradient. At $\varepsilon=10^{-4}$,
directional differences agree with exact derivatives to relative $\ell_2$
error below $6\times10^{-7}$.

The predicted best batch sizes are $B\simeq2.04,2.89,4.09$ for four experts per loss, two experts per loss,
and one expert per loss. The lowest measured values occur at $B=2,2,4$.
The predictions capture the error scale and the batch--perturbation tradeoff
on this fixed objective.

\begin{figure}[!htbp]\centering
\begin{tikzpicture}
\begin{groupplot}[
  group style={group size=2 by 1,horizontal sep=1.8cm},
  width=6.8cm,height=5.6cm,
  xmode=log,ymode=log,log basis x=2,
  grid=major,grid style={black!10},
  tick label style={font=\scriptsize},label style={font=\small},
  title style={font=\small},
  legend style={draw=none,fill=none,font=\scriptsize,
    cells={anchor=west},at={(.98,.98)},anchor=north east}]
\nextgroupplot[title={\shortstack{Independent expert losses\\$N=4$, $B=2$}},
  xlabel={Perturbations per expert $n_{\rm pert}$},ylabel={Relative gradient error},
  xmin=3.4,xmax=75,xtick={4,16,64},xticklabels={$4$,$16$,$64$},
  ymin=750,ymax=100000]
\addplot[secBlue,thick,dashed]
 table[x=npert,y=joint_theory,col sep=comma]{figdata/theory_empirical_A/theory.csv};
\addlegendentry{Four experts per loss}
\addplot[secTeal,thick,dashed]
 table[x=npert,y=separate_theory,col sep=comma]{figdata/theory_empirical_A/theory.csv};
\addlegendentry{One expert per loss}
\addplot[secBlue,forget plot,only marks,mark=*,mark size=2pt,
  error bars/.cd,y dir=both,y explicit]
 table[x=npert,y=joint_mean,y error=joint_se,col sep=comma]{figdata/theory_empirical_A/finite_summary.csv};
\addplot[secTeal,forget plot,only marks,mark=square*,mark size=2pt,
  error bars/.cd,y dir=both,y explicit]
 table[x=npert,y=separate_mean,y error=separate_se,col sep=comma]{figdata/theory_empirical_A/finite_summary.csv};
\nextgroupplot[title={\shortstack{Batch perturbation trade-off\\$N=4$, $n_{\rm pert}B=256$}},
  xlabel={Batch size $B$},ylabel={Relative gradient error},
  xmin=.85,xmax=38,xtick={1,2,4,8,16,32},
  xticklabels={$1$,$2$,$4$,$8$,$16$,$32$},
  legend style={at={(.02,.98)},anchor=north west}]
\addplot[secBlue,thick,dashed]
 table[x=B,y=predicted_relative_mse,col sep=comma]{figdata/theory_empirical_J/theory_fixed_work_N1.csv};
\addlegendentry{Four experts per loss}
\addplot[somaOrange,thick,dashed]
 table[x=B,y=predicted_relative_mse,col sep=comma]{figdata/theory_empirical_J/theory_fixed_work_N2.csv};
\addlegendentry{Two experts per loss}
\addplot[secTeal,thick,dashed]
 table[x=B,y=predicted_relative_mse,col sep=comma]{figdata/theory_empirical_J/theory_fixed_work_N4.csv};
\addlegendentry{One expert per loss}
\addplot[secBlue,forget plot,only marks,mark=*,mark size=2pt,
  error bars/.cd,y dir=both,y explicit]
 table[x=B,y=measured_relative_mse,y error=standard_error,col sep=comma]{figdata/theory_empirical_J/fixed_work_N1.csv};
\addplot[somaOrange,forget plot,only marks,mark=triangle*,mark size=2.3pt,
  error bars/.cd,y dir=both,y explicit]
 table[x=B,y=measured_relative_mse,y error=standard_error,col sep=comma]{figdata/theory_empirical_J/fixed_work_N2.csv};
\addplot[secTeal,forget plot,only marks,mark=square*,mark size=2pt,
  error bars/.cd,y dir=both,y explicit]
 table[x=B,y=measured_relative_mse,y error=standard_error,col sep=comma]{figdata/theory_empirical_J/fixed_work_N4.csv};
\end{groupplot}
\end{tikzpicture}
\caption{Keeping independent losses separate reduces gradient error at the same work.
Lower is better. Both panels use SOMA $N=4$ formed from SOMA $N=256$ at 900k updates, with fixed weights and dense probes. The right panel uses one loss for four experts, two losses for two experts each, or four individual losses. Relative error is $\E\|\widehat g-g\|^2/\|g\|^2$, using the exact
gradient $g$. Left fixes two ensemble validation chunks. At $n_{\rm pert}=64$, independent losses give
$4.02\times$ lower error, against $4.00\times$ predicted. Right selects 128
coordinates per expert and draws batches from a fixed pool of 32 chunks.
Its reference is the pool mean gradient, with measured $\nu=8.32$.
Perturbations are counted per expert. All loss groupings use the same
forward evaluations. Right holds $n_{\rm pert}B=256$, so larger batches leave
fewer perturbations. The total batch--perturbation budget is fixed at
$n_{\rm pert}B=256$, so increasing batch size reduces
the number of perturbation directions. Dashed lines are predictions.
Markers show mean central-difference error. Bars are standard errors over
32 sets of random perturbations on the left and 16 batch-and-perturbation
draws on the right.}
\label{fig:theory-empirical}
\end{figure}
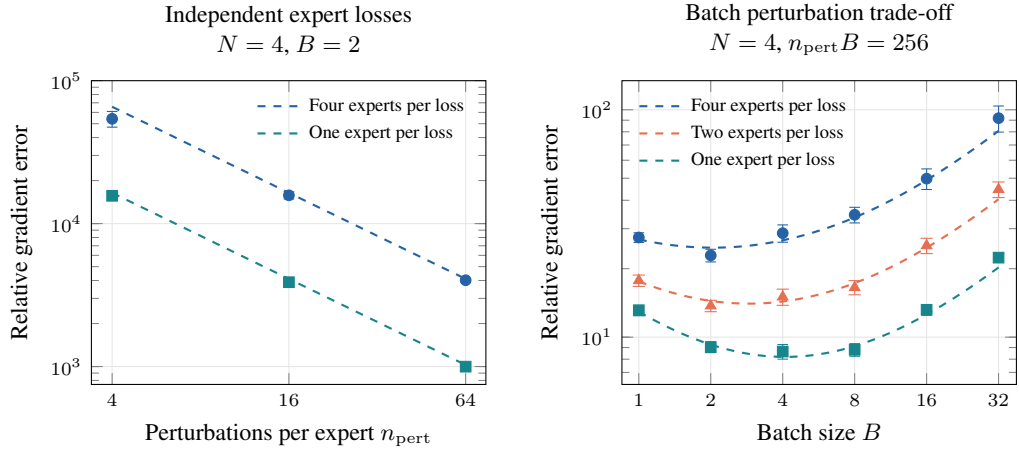

\section{Training compute and communication}\label{app:compute-details}
An update that uses more evaluations is worthwhile only if its improvement repays the added cost. Parallel execution can shorten that update without reducing its aggregate work. We therefore distinguish arithmetic compute, recorded GPU time and allocated device time, then describe how independent expert training and synchronized perturbation training incur these costs.

The different compute totals refer to the same $N=256$ checkpoint after 4.16M updates per expert (Table~\ref{tab:n256-compute-accounting}). For expert $e$, let $\bar t_e$ be its recorded mean update time in seconds on one GPU, and let $U$ be its update count. Figures~\ref{fig:independent-evaluation} and~\ref{fig:inference-budget-frontier} use $U\sum_{e=1}^{256}\bar t_e/3600$, which gives 41.9k GPU-hours. The recorded means average 0.142 seconds across experts. These estimates add no separate initialization, preprocessing or idle-allocation cost. The W\&B total is a separately logged quantity.

\begin{table}[htbp]
\centering
\small
\begin{tabular}{lrl}
\toprule
Accounting method & GPU-hours & Calculation or recorded field \\
\midrule
Recorded expert mean times & 41.9k & $U\sum_{e=1}^{256}\bar t_e/3600$ \\
W\&B-reported physical total & 49.6k & \texttt{physical\_gpu\_hr} \\
Fixed 0.168-second reference & 49.7k & $256U(0.168)/3600$ \\
Fixed 0.140-second normalization & 41.4k & $256U(0.140)/3600$ \\
\bottomrule
\end{tabular}
\caption{Compute accounting for the same $N=256$ checkpoint. The first row is used for the endpoint in Figures~\ref{fig:independent-evaluation} and~\ref{fig:inference-budget-frontier}. The other rows preserve the separate logged total and fixed-rate conventions. They do not identify different training endpoints.}
\label{tab:n256-compute-accounting}
\end{table}

We report compute either as arithmetic work or as aggregate GPU time, with the unit stated in each comparison. Arithmetic work counts operations in recurrent bodies, projections and decoders, with each multiply-add counted as two operations. Training includes both signs of every perturbation, and inference includes only the selected experts. These counts exclude seed training, head updates, routing, communication and smaller elementwise operations. Aggregate GPU-hours instead sum time across participating devices, including communication and synchronization within recorded updates.

For context length $T$, effective batch $B$ per expert, $s$ updates and per-expert compute $F_N$ in operations per token, total training compute $C_{
m train}$ is
\begin{equation}
C_{\rm train}=2Tn_{\rm pert}BNF_Ns.
\label{eq:training-flops}
\end{equation}
At fixed model size, more shards make each expert cheaper. At fixed expert width, $F_N$ stays fixed and work grows with $N$. A width-32 expert uses 81.9k arithmetic operations per token. These counts measure compute without communication or waiting time.

\subsection{Independent expert training}\label{sec:scale}
Local expert losses allow FSO to distribute complete learning problems across GPUs. Each expert owns one RTX~5090, a disjoint corpus cluster, optimizer state and training loop, so its updates do not wait for another expert. Within each expert, we vectorize SPSA by prepending
a perturbation axis to the usual batch-first tensor layout. A tensor with shape
$(B, T, d_t)$ becomes $(n_{\rm pert}, B, T, d_t)$, allowing PyTorch tensor operations to process the
perturbations and batch elements in parallel on the GPU. The \texttt{pert\_chunk} setting specifies how
many perturbation directions are processed together in one tensor tile. Similarly, the \texttt{micro\_batch}
setting specifies how many training examples are processed for each perturbation.
The expert-count sweep uses $n_{\rm pert}{=}64$, batch $64$,
\texttt{pert\_chunk}=64, and \texttt{micro\_batch}=64, enabling saturated parallelism. The processes exchange no gradients, optimizer
states, activations, finite differences, or samples. Consequently, adding experts increases total device
work without increasing the wall-clock time of an expert update. The recorded expert mean times average 0.142 seconds per update, giving 36.3 aggregate GPU-seconds per global step and 41.9k GPU-hours at the endpoint. Table~\ref{tab:n256-compute-accounting} preserves the separately logged physical total and fixed-rate normalizations.

\subsection{Distributed perturbations}
A monolithic model cannot distribute independent expert updates, but it can distribute the evaluations needed to form one update. DDPP splits batches and perturbations across devices and combines their gradient estimates before updating the shared model (Figure~\ref{fig:dc}). Appendix~\ref{app:variance-efficiency} gives the measured one-, two-, eight- and sixteen-GPU configurations and their communication costs.

\begin{figure}[!htbp]\centering
\begin{tikzpicture}[
  font=\small,>={Stealth[length=2mm]},
  coordinator/.style={draw,rounded corners,minimum width=48mm,minimum height=12mm,
    align=center,fill=orange!12},
  network/.style={draw,rounded corners,minimum width=116mm,minimum height=8mm,
    align=center,fill=blue!6},
  worker/.style={draw,rounded corners,minimum width=25mm,minimum height=11mm,
    align=center,fill=green!7},
  upload/.style={->,dashed,line width=0.8pt,blue!65!black},
  broadcast/.style={->,line width=0.8pt,orange!85!black}]
  \node[coordinator] (co) {coordinator\\\scriptsize combine local gradient estimates};
  \node[network,below=13mm of co] (net) {Communication network};

  \node[worker,below=10mm of net,xshift=-42mm] (a)
    {node A\\\scriptsize local evaluations};
  \node[worker,below=10mm of net,xshift=-14mm] (b)
    {node B\\\scriptsize local evaluations};
  \node[font=\large,below=12mm of net,xshift=14mm] (dots) {$\cdots$};
  \node[worker,below=10mm of net,xshift=42mm] (m)
    {last node\\\scriptsize local evaluations};

  \draw[upload] ([xshift=-3mm]net.north) --
    node[left=3mm,align=right,font=\scriptsize] {local gradient\\estimates}
    ([xshift=-3mm]co.south);
  \draw[broadcast] ([xshift=3mm]co.south) --
    node[right=3mm,align=left,font=\scriptsize] {combined gradient\\estimate}
    ([xshift=3mm]net.north);

  \foreach \w in {a,b,m}{
    \draw[upload] ([xshift=-2mm]\w.north) -- ([xshift=-2mm]\w.north |- net.south);
    \draw[broadcast] ([xshift=2mm]\w.north |- net.south) -- ([xshift=2mm]\w.north);
  }
\end{tikzpicture}
\caption{DDPP synchronizes every global step. Nodes evaluate subsets of the batches and perturbations for one shared model. Dashed blue arrows carry local gradient estimates to the coordinator, and solid orange arrows return the combined estimate. FSO instead updates each expert using its own loss and exchanges no per-step updates.}
\label{fig:dc}
\end{figure}
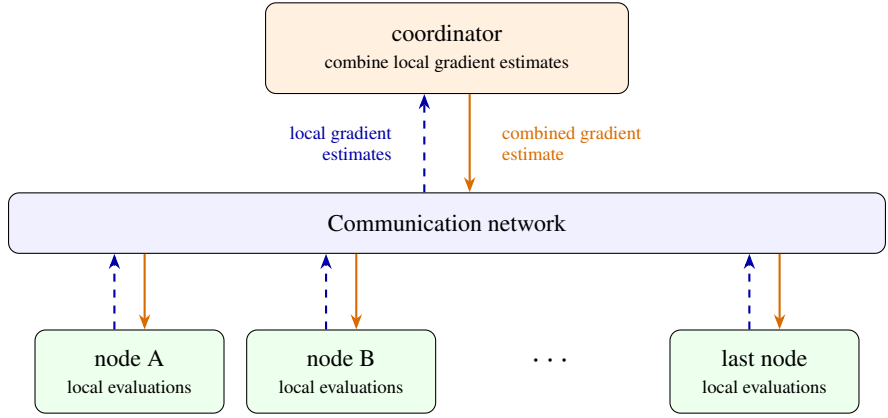

\section{Optimizer controls}\label{app:figure2-controls}
Figure~\ref{fig:independent-evaluation} asks whether sharding improves test loss relative to spending the training budget on a monolithic model. We document the model sizes, optimizer settings, saved checkpoints and timing used for that comparison. The additional BPTT controls place these ZO results in the context of recurrent backpropagation, while the early tuning study tests sensitivity to optimizer settings.

Figure~\ref{fig:independent-evaluation} sums the per-expert GPU time to estimate aggregate compute at each saved checkpoint, without interpolating test losses. For the eight-GPU EGGROLL run, we sum the logged perturbation-generation and update times and multiply by eight for GPU-hours. Repeated log entries at resumed steps are counted once. The selected checkpoints nearest 10, 20, 30 and 40 hours are at 3.80k, 8.50k, 13.3k and 18.0k updates. Figure~\ref{fig:independent-evaluation} includes the final monolithic SPSA checkpoint at 34.0k updates and 172 GPU-hours, with test loss 1.93. Its ensemble validation loss is 1.95, a score on a different evaluation set. Training time excludes initialization, evaluation and checkpointing. The $N=2,8$ wall-clock estimates assume one GPU per expert, using rates from those recorded placements.

The two late $N=8$ ensembles in Figure~\ref{fig:independent-evaluation} use experts 0--3 at 53.5k, 58.5k, 54.5k and 54.5k updates. Experts 4--7 are at 14.5k, 13.0k, 15.0k and 15.5k updates in the first ensemble, and each is at 70.0k updates in the second. Multiplying each expert's update count by its recorded step time gives 390 and 700 aggregate GPU-hours. The longest expert time gives 80.7 and 104 wall-clock hours. 
Figure~\ref{fig:independent-evaluation} uses the long $N=256$ campaign, with 21 evaluated checkpoints through 4.16M updates. Its recorded mean step time across experts is 0.142 seconds and the slowest expert mean is 0.165 seconds. Summing expert times gives 41.9k GPU-hours, and the slowest expert gives 191 wall-clock hours. The EGGROLL curve includes 28 checkpoints through 122k updates, 2.12k GPU-hours and 265 hours. All plotted markers are checkpoint evaluations. Values below 10 aggregate GPU-hours are outside the displayed range, and curves start at the first measured checkpoint in that range. No shared starting point is imposed. Four available $N=8$ ensembles lie in this range, ending at 104 hours.

\paragraph{Monolithic SPSA.}
The 8.33M-parameter curve is the warm width-509 monolith from the fixed-parameter comparison. It reaches ensemble validation loss 1.95 and test loss 1.93 at 34.0k updates and 172 estimated GPU-hours, excluding its 3.66-hour seed.

\subsection{Monolithic perturbation controls near 150 GPU-hours}
\label{app:near150-controls}
A larger perturbation population offers the monolith an alternative to sharding. To test this allocation at a comparable total training cost, Table~\ref{tab:near150-controls} records the measured checkpoints used in the main-text comparison. The two new monolithic SPSA runs retain the common starting checkpoint, effective batch size $B=64$ and 1,024-byte context while increasing $n_{\rm pert}$ to 256 or 1024. They match the wide controls' initial learning rate 0.0025, $\varepsilon$ schedule tied to the learning rate, Adam settings and local decoder updates. These runs vary perturbation count without constituting a learning-rate sweep. They use 16 or 32 RTX~5090 GPUs with DDPP, respectively. Their aggregate cost sums recorded time across participating GPUs, including communication and waiting. All rows use the same 4,096-document test set. Figure~\ref{fig:independent-evaluation} displays every evaluated checkpoint within its aggregate-compute range, including the endpoints of both runs.

\begin{table}[htbp]
\centering
\small
\begin{tabular}{lrrrr}
\toprule
Configuration & $n_{\rm pert}$ & Updates & GPU-hours & Test loss\\
\midrule
SOMA $N=2$ & 64 & 65.0k & 150 & 1.76\\
Monolithic SPSA & 64 & 30.0k & 152 & 2.00\\
Monolithic SPSA & 256 & 1.32k & 150 & 2.11\\
Monolithic SPSA & 1024 & 839 & 148 & 2.00\\
Monolithic EGGROLL & --- & 7.40k & 141 & 2.21\\
\bottomrule
\end{tabular}
\caption{Measured test losses near 150 aggregate GPU-hours for models of approximately 8.44M parameters. Updates are checkpoint counters, per expert for SOMA. Test loss is in nats/byte. Device costs include communication and waiting, so this compares the measured training systems rather than arithmetic work alone. Values are rounded to three significant figures. EGGROLL uses its population of $2^{20}$ and $B=256$.}
\label{tab:near150-controls}
\end{table}

\paragraph{EGGROLL reproduction.}
We evaluate the archived FP32 six-layer GRUs at widths $20,29,58,118,339$. The reproduction uses rank-one perturbations, population $2^{20}$, batch 256, and 100 bytes per stream update on the 100B-byte corpus. Its embedding and output head are untied. The recorded settings are $\alpha=0.01$, \texttt{sigma\_shift}=4, and \texttt{float\_step}=0.01. For each width, we select the lowest logged ensemble validation loss from the seed-0 history and reevaluate its checkpoint. These results reproduce the logged losses within $7.4\times10^{-6}$ nats/byte. The selected updates are 435,300, 311,000, 206,700, 105,400, and 102,600, respectively. The four smaller jobs use two GPUs and 494--496 recorded GPU-hours. The largest uses eight GPUs and 1.80k GPU-hours. This compares the reproduced GRU-based EGGROLL recipe with LSTM-based SOMA. The matched LSTM SPSA controls isolate the effect of perturbation budget more directly.

\paragraph{SOMA and BPTT at common training budgets.}
We check whether BPTT's lower loss persists earlier in training.
Table~\ref{tab:training-budget-comparison} compares the recorded recipes
near 8.44M parameters. GPU-hours sum work across devices,
$H=s\sum_i g_i\tau_i/3600$, where $s$ is the update count, $g_i$ the GPU
count and $\tau_i$ the seconds per update for expert $i$.
SOMA uses the reference placements above. BPTT uses $4/2/1$ GPUs per
expert for $N=1/2/8$ and median update times with evaluation and checkpoint
windows removed. Both estimates exclude initialization. BPTT head
initialization is estimated at 1.52 GPU-hours on its recorded hardware. Its exact checkpoint
step is unknown.

\begin{table}[H]
\centering
\caption{BPTT reaches lower ensemble validation loss at these budgets.
Values are in nats/byte on the ensemble validation set, using roughly 8.44M
models (Section~\ref{sec:model-size}). Values interpolate measured losses,
except the marked BPTT $N=1$ entry, which uses its last checkpoint at
97.3 GPU-hours. Hardware and training recipes differ.}
\label{tab:training-budget-comparison}
\setlength{\tabcolsep}{7pt}
\begin{filecontents*}{figdata/budget_comparison/comparison_table.csv}
N,soma50,bptt50,soma75,bptt75,soma100,bptt100,bptt100_gpu_hours,bptt100_display
1,2.1237496386204926,1.0197612832907217,2.0590126541320415,1.020567160760289,2.0343014593876445,1.0091165435642147,97.25405038727655,1.0091$^{*}$
2,1.936271488818369,1.0711175090364402,1.8734037038449043,1.0652662814064278,1.8256569907050797,1.0534647329147635,100,1.0535
8,2.0879775941497902,1.1825662791655336,2.045562434445222,1.1797393989935843,2.0179161055583448,1.1776438742290267,100,1.1776
\end{filecontents*}
\pgfplotstabletypeset[
 col sep=comma,
 columns={N,soma50,bptt50,soma75,bptt75,soma100,bptt100_display},
 every column/.style={column type=r,fixed,fixed zerofill,precision=4},
 columns/N/.style={column name={$N$},column type=r,fixed,precision=0},
 columns/soma50/.style={column name={SOMA}},
 columns/bptt50/.style={column name={BPTT}},
 columns/soma75/.style={column name={SOMA}},
 columns/bptt75/.style={column name={BPTT}},
 columns/soma100/.style={column name={SOMA}},
 columns/bptt100_display/.style={column name={BPTT},string type,column type=r},
 every head row/.style={before row={\toprule
 &\multicolumn{2}{c}{50 GPU-hours}&\multicolumn{2}{c}{75 GPU-hours}&\multicolumn{2}{c}{100 GPU-hours}\\
 \cmidrule(lr){2-3}\cmidrule(lr){4-5}\cmidrule(lr){6-7}},after row=\midrule},
 every last row/.style={after row=\bottomrule}
]{figdata/budget_comparison/comparison_table.csv}
\end{table}

These SOMA recipes do not show a training advantage over backpropagation
on this task.

\paragraph{Fixed-parameter SOMA recipes.}
The wide $N=1,2,8$ SPSA controls retain a separate tied embedding/decoder in each expert and update it through Equation~\eqref{eq:delta}. The saved heads differ across experts and checkpoints. Their parameter totals therefore count each head, giving 8.33M, 8.41M and 8.55M parameters, respectively. The width-32 $N=256$ campaign instead uses one shared frozen head, giving 8.44M parameters. Table~\ref{tab:fixed-parameter-counts} records these conventions.
Figure~\ref{fig:independent-evaluation} compares these approximately equal-size models on the test set across aggregate training budgets. Near 150 GPU-hours, SOMA $N=2$ reaches lower loss than the tested monolithic controls and the more heavily sharded ensembles. This comparison addresses the predictive return on total training work, rather than the loss achieved by giving each expert the same elapsed time.

\subsection{Early learning-rate and perturbation-radius sensitivity}
\label{app:early-tuning}
A fixed optimizer setting can favor one architecture even when both receive the same training budget. We therefore give monolithic SPSA and SOMA $N=2$ the same learning-rate and $\varepsilon$ search, then repeat it across optimization seeds to test the early loss comparison. The monolith has width 509 and 8.33M parameters, and SOMA $N=2$ has width 361 per expert and 8.41M total parameters. Each configuration starts from its architecture's saved step-1 checkpoint, including optimizer state. We repeat the search with three optimization seeds that change the sampled batches and perturbations while holding those starting weights and optimizer states fixed. The monolith uses the original 100B-byte FineWeb-Edu stream, and the experts use its two semantic shards. All configurations use $n_{\rm pert}=64$, $B=64$ per expert, 1,024-byte context, coordinate activity 0.5 for ordinary weights and 1 for normalization gains, local decoder updates and coupled weight decay $10^{-4}$ before Adam. Adam uses $\beta_1=0.9$, $\beta_2=0.999$ and stabilizer $10^{-8}$. Training retains the wide controls' TF32 settings and uses microbatches and perturbation chunks of 16.

We first compare learning rates $\{0.00125,0.0025,0.005\}$ at initial $\varepsilon=0.0025$. At the selected learning rate, we then compare initial $\varepsilon\in\{0.00125,0.0025,0.005\}$, reusing the middle setting. This gives five configurations per architecture and optimization seed, or 30 configurations in total. Initial $\varepsilon$ is varied independently of the learning rate, with both retaining the inherited plateau multiplier schedule. Within each architecture and seed, configurations share the sampled batch and perturbation sequence at each update. Each configuration receives the same aggregate GPU compute budget, 1.65 GPU-hours on identical hardware. The monolith uses one GPU, and each of the two experts receives 0.825 hours on its own GPU. Runs stop after a complete update, giving 149--151 new updates for the monolith and 130--132 per expert for the ensemble.

Selection uses the same 976 ensemble validation windows for both architectures. We score saved checkpoints near one-quarter, one-half, three-quarters and the full training budget, and select settings separately for each architecture and seed using only the final validation loss. Each selected checkpoint is then evaluated once on the common 4,096-document test set. Evaluation disables TF32 and uses FP64 probability-mixture and loss reductions.

\begin{table}[htbp]
\centering
\small
\begin{tabular}{rrrr}
\toprule
Learning rate & Initial $\varepsilon$ & $N=1$ & $N=2$\\
 & & \multicolumn{2}{c}{Ensemble validation loss}\\
\midrule
0.00125 & 0.00250 & $2.2768 \pm 0.0040$ & $2.2370 \pm 0.0034$\\
0.00250 & 0.00250 & $2.3247 \pm 0.0149$ & $2.2708 \pm 0.0025$\\
0.00500 & 0.00250 & $2.4179 \pm 0.0267$ & $2.3379 \pm 0.0073$\\
0.00125 & 0.00125 & $2.2764 \pm 0.0040$ & $2.2370 \pm 0.0034$\\
0.00125 & 0.00500 & $2.2811 \pm 0.0048$ & $2.2404 \pm 0.0033$\\
\bottomrule
\end{tabular}
\caption{Learning-rate and $\varepsilon$ sensitivity at an equal aggregate GPU compute budget per configuration. Entries are mean $\pm$ sample standard deviation across three optimization seeds, in nats/byte. Selection uses each seed's unrounded validation losses.}
\label{tab:early-tuning}
\end{table}

All six searches select learning rate 0.00125, the lower edge of the tested range. The selected initial $\varepsilon$ is 0.00125 except for SOMA $N=2$ in seed 2, which selects 0.0025. After this equal tuning allowance, SOMA $N=2$ achieves lower test loss in all three seeds (Table~\ref{tab:early-tuning-test}), with a mean difference of 0.0396 nats/byte. This supports the early loss comparison across batch and perturbation randomness. Because the starting checkpoints are fixed and the training budgets are short, it does not establish repeatability across independently pretrained initializations or fully converged models.

\begin{table}[htbp]
\centering
\small
\begin{tabular}{lrrr}
\toprule
Optimization seed & Monolithic SPSA & SOMA ($N=2$) & Difference\\
\midrule
1 & 2.2641 & 2.2273 & 0.0368\\
2 & 2.2666 & 2.2204 & 0.0462\\
3 & 2.2576 & 2.2218 & 0.0358\\
\midrule
Mean & 2.2628 & 2.2232 & 0.0396\\
Sample standard deviation & 0.0047 & 0.0036 & 0.0057\\
\bottomrule
\end{tabular}
\caption{Test loss after selecting each seed's settings using only the ensemble validation set. All losses are nats/byte, and the difference is monolithic SPSA minus SOMA. Standard deviations measure optimization-seed variation with fixed starting checkpoints, not test-document uncertainty.}
\label{tab:early-tuning-test}
\end{table}

\section{Independent versus summed training losses}
The variance calculation isolates a benefit of independent loss feedback, but a smaller estimation error need not change training loss by the same amount. We test its effect on learning by changing only whether each expert receives its own loss or the sum of the experts' losses. Paired continuations hold the model, starting state, data, perturbations and compute fixed at both early and mature checkpoints.

\label{app:loss-separation}

We form SOMA $N=4$ using experts 0--3 from the earliest complete trained SOMA $N=8$,
at 2,000 updates. The model has 140k parameters under
Section~\ref{sec:model-size}. Each paired continuation starts with the same
weights and Adam moments. Heads stay frozen, $n_{\rm pert}=B=64$, and both
the learning rate and $\varepsilon$ stay at $10^{-3}$.

One arm uses each expert's loss. The other sums the four losses before
forming each finite difference. Summing keeps the expected gradient scale
unchanged. Both use the same data and sampled perturbation vectors, TF32
forward passes and FP64 scalar subtraction. We draw training windows from
one fixed 512 MiB sample of each original shard. The samples were selected
before these tests. Test evaluation follows Table~\ref{tab:experimental-setup}.

We test both loss rules with three random seeds for 1,000 further updates and evaluate at 0, 250,
500 and 1,000 updates. Each arm uses $2.75\times10^{15}$ arithmetic operations. All six
continuations take 0.94 GPU-hours including profiling. The final paired
loss reductions are 0.0362, 0.0339 and 0.0353 nats/byte
(Figure~\ref{fig:loss-separation}). Separate losses improve training with
architecture and compute held fixed. A check of the saved document scores
finds lower loss with separate losses on 4,090 of the 4,096 test documents in
every seed at 1,000 updates.

We then measure gradient variance at the common start and all six final ensembles.
At each state, both loss rules use the same weights, two batches of 64
training windows and 128 sampled perturbation vectors per expert.
Using each expert's own loss reduces variance by a factor of 3.19--4.64 on
the GPU and 3.20--4.65 in CPU FP64. The small-$\varepsilon$ prediction is $4\times$. Separate losses therefore reduce variance on the same models
where they improve learning.

\raggedbottom
\afterpage{\flushbottom}

We repeat the test to see whether local losses still help later in training.
The same SOMA $N=4$ now starts at 4M updates, with their saved Adam states,
frozen heads, learning rate and $\varepsilon$ of $10^{-5}$.
We retain the four training samples and test both loss rules with three random seeds and
$n_{\rm pert}=B=64$. We take 10,000 further updates and evaluate at 0,
1,000, 2,500, 5,000 and 10,000 updates on the same test set.
All paired data and direction hashes match at every update.
Each arm uses $2.75\times10^{16}$ arithmetic operations. The six runs take 32.2 aggregate GPU-hours on the same hardware.

Starting test loss is 1.80. Local losses lower it by 0.000235, while summed losses
raise it by 0.0000720. The three paired differences are 0.000300, 0.000332 and
0.000289 nats/byte. The benefit persists but is small at these saved settings.
The learning rate, $\varepsilon$ and continuation length differ between the early
and mature tests, so the change in effect size cannot be attributed to training
stage alone.

Repeating the fixed-batch variance test at the mature start and six final
states gives a 3.42--4.55-fold reduction with local losses, near the
small-$\varepsilon$ prediction of four. These measurements use the saved $\varepsilon$
of $10^{-5}$ and GPU forward passes.

\section{Measuring gradient error and GPU time}\label{app:variance-efficiency}
To compare allocations before committing to a full training run, we measure the accuracy and cost of one global update. A large independent BPTT batch supplies the reference gradient, and measured GPU time includes the communication and waiting required to form the ZO estimate. Pairing these quantities shows what each allocation buys per global step, while Figure~\ref{fig:independent-evaluation} tests the resulting training outcomes.

Figure~\ref{fig:variance-efficiency} shows measurements after 10,000 updates in models near 8.44M parameters (Section~\ref{sec:model-size}). We also test randomly-initialized expert bodies. Experts within each ensemble share weights from the original initializer. The saved runs began from pretrained weights, so initialization is a separate untrained control. At each stage, the $n_{\rm pert}$ and $B$ curves share the same $N=1$ weights and BPTT reference. The $N$ curve compares different ensembles.

Overlapping batches passed the $0.01$ cosine distance test early, so we checked larger BPTT references with disjoint batches. Table~\ref{tab:population-reference} gives the reference sizes and independent agreement.

For expert $e$, let $g_{e,\rm ref}$ be its BPTT gradient on the reference batch shared by its ensemble, and $\widehat g_e$ its ZO estimate on a sampled batch. Holding the reference fixed, expectation over batches and perturbations gives relative gradient estimation error
\[
\frac{\sum_e \E\|\widehat g_e-g_{e,\rm ref}\|^2}
{\sum_e\|g_{e,\rm ref}\|^2}.
\]
Both sums cover all experts and their perturbed parameters. This measures bias and variation across batches and perturbations. Lower is better.

We use 64 independent direction pairs with the saved $\varepsilon$ values and original perturbation sampling rule. For $N=1$, each pair uses 64 independent blocks of 16 sequences. We average the error over all subsets of $B/16$ blocks, reducing variation from batch choice. Both directions share each subset. Forward losses use FP32 and the subset moments use FP64. The $n_{\rm pert}$ curve averages independent directions using these moments at $B=64$. Other ensembles use directly measured $B=64$ batches. Two directions are evaluated together, so rounding may differ from larger training groups. Reference and ZO measurements use disjoint data. At initialization, identical weights and $\varepsilon$ are measured once and weighted by the expert count.

One global step updates every expert once. With $n_{\rm pert}=B=64$, FSO assigns
one GPU per expert for $N\in\{1,2,8,256\}$. We sum separately measured expert
update times for this placement. For $N=1$, increasing $n_{\rm pert}$ or $B$
to 128, 256 and 1024 uses DDPP on two, eight and 16 GPUs, respectively. The latter spans
two nodes. DDPP cost multiplies the global step time by the allocated GPU count,
including gradient communication and waiting. We average five timed steps after
three warm-up steps and report aggregate GPU-seconds per global step.
These update measurements are separate from the historical training rates used to place saved models at a total budget.

Increasing $n_{\rm pert}$ from 64 to 1024 lowers gradient error by about 94\% at both stages. Increasing $B$ over the same range lowers error by 0.27\% at initialization and 47\% after 10,000 updates. At initialization, averaging perturbations is much more effective than increasing $B$. The $N=256$ ensembles have $120\times$ and $210\times$ lower error than $N=1$, respectively.
Increasing $N$ replaces larger experts with smaller ones. Cheaper forwards
can offset the larger expert count. Lower gradient error alone does not
establish lower test loss.

\begin{table}[htbp]\centering\small
\begin{tabular}{rrrrr}
\toprule
& \multicolumn{2}{c}{Initialization} & \multicolumn{2}{c}{10,000 updates}\\
\cmidrule(lr){2-3}\cmidrule(lr){4-5}
$N$ & Batch size & Cosine distance & Batch size & Cosine distance\\
\midrule
1 & 8000 & 0.00002 & 8000 & 0.01045\\
2 & 8000 & 0.00002 & 8000 & 0.00503\\
8 & 8000 & 0.00002 & 4000 & 0.00492\\
256 & 8000 & 0.00013 & 1600 & 0.00759\\
\bottomrule
\end{tabular}
\caption{BPTT reference batch sizes and cosine distance to an independent batch of the same size. Trained $N=1$ is just above $0.01$. Using the independent references changes every plotted value by less than $4\%$.}
\label{tab:population-reference}
\end{table}

\section{Test set construction and evaluation}\label{app:independent-evaluation-data}
The training comparisons need an evaluation set that is separate from the documents used to fit the models. We construct a common test set screened against the reconstructed training streams and score saved checkpoints with the protocol in Table~\ref{tab:experimental-setup}. This appendix specifies how the documents are selected, which bytes are scored and what overlap checks are performed.

We construct this set from \path{HuggingFaceFW/fineweb-edu}, subset \texttt{sample-100BT}, using source files 40--139 in zero-based order. With seed 20260906, we draw 256 source row groups and forty candidate documents per group. The training-overlap checks exclude 1,433 of these 10,240 candidates. We then select sixteen documents per group. The final set spans 91 files and 3,700 web domains. No model scores are used to select documents.

Each document contributes one 1,025-byte window that stays within its boundaries. We count UTF-8 bytes, not Unicode characters. The set contains 4.20M window bytes. Routing sees the first 256 input bytes and loss scores the following 768 next-byte targets, giving 3.15M scored bytes per model. Hidden state resets between windows.

We exclude candidates sharing a source ID or an exact or normalized full-document copy with either reconstructed training prefix, including full boundary documents. Normalization applies Unicode NFKC, case folding and whitespace collapse. The selected documents and windows have unique hashes. No selected window shares a 128-byte substring with the BPTT expert validation set. The full 100B stream matches the retained corpus SHA256. The older 10B reconstruction matches the original ensemble validation range and all 256 archived chunk probes. These checks exclude known source-document overlap. They do not exhaustively rule out near duplicates or shorter shared passages.

The original eleven models, sample and analysis were fixed before scoring. After seeing those results, we added twenty models to cover every available $N$ and the additional budgets on the same documents. Their checkpoints and interpolation weights were fixed from training records before scoring each addition. We later evaluated the 34,000-update monolithic SPSA control on the same documents. The frozen dataset manifest, document IDs, byte offsets and hashes are recorded in \path{SOMA_EVALUATION/data/independent_v1/}.

Results apply to this sampling frame. The 95\% intervals use 10,000 paired document-bootstrap draws. We also resample whole hosts and source row groups. These measure sampling uncertainty conditional on these trained models, not training-seed variability.

On the test set, SOMA $N=2$ has the lowest estimated loss at the fixed-size budgets (Figure~\ref{fig:independent-evaluation}). Adding fixed-width experts also lowers loss through $N=256$, while total size and work grow. These results remain conditional on the trained models and budgets.

The supplement contains the figure CSVs, analysis scripts, and checkpoint manifest. A separate CPU evaluation package supplies frozen weights, fitted routers, examples, expected scores, and pinned dependencies. CPU runs reproduce the endpoint for SOMA $N=256$ and two representative controlled-test ensembles on the test set.

\subsection{Frozen evaluation on WikiText-103}
\label{app:wikitext-transfer}
We evaluate the saved FineWeb-Edu checkpoints on the raw WikiText-103 validation and test splits~\citep{merity2016wikitext} without updating model weights or the router. We retain each archived forward implementation to reproduce the original FineWeb scores. Table~\ref{tab:wikitext-transfer} contains the same approximately 8.44M-parameter checkpoints compared near 150 aggregate GPU-hours in Figure~\ref{fig:independent-evaluation}. SOMA $N=2$ also reaches lower loss on this external corpus, with 2.07 test nats/byte compared with 2.25--2.36 for the monolithic SPSA controls and 2.49 for EGGROLL.

We use the official \texttt{wikitext-103-raw-v1} splits from \texttt{Salesforce/wikitext}, revision \texttt{b08601e04326c79dfdd32d625aee71d232d685c3}. Text is encoded as UTF-8 without normalization. We split at top-level-heading-shaped rows and form 1,025-byte windows within each segment, advancing by 768 bytes. Each model receives the first 1,024 bytes and scores the targets at byte indices 257--1,024 from output indices 256--1,023, using zero-based indices. Hidden state resets for each window, and SOMA routes using only the observed prefix. Incomplete final windows are omitted. This gives 1.11M scored validation targets and 1.25M scored test targets, with identical targets for all models. We report byte-level loss rather than token perplexity. The original 4,096-document FineWeb test losses are reproduced within $1.3\times10^{-8}$ nats/byte before scoring WikiText. Forward operations use FP32 with TF32 disabled, with FP64 mixture and loss reduction. This is an external-corpus evaluation, not a verified absence of overlap with the pretraining corpus.

\begin{table}[htbp]
\centering
\small
\setlength{\tabcolsep}{4pt}
\begin{tabular}{lrrr}
\toprule
Checkpoint & Training compute & Validation & Test\\
 & (GPU-hours) & (nats/byte) & (nats/byte)\\
\midrule
\multicolumn{4}{l}{8.44M parameters}\\
SOMA $N=2$, $n_{\rm pert}=64$ & 150 & 2.047 & 2.066\\
Monolithic SPSA, $n_{\rm pert}=64$ & 152 & 2.240 & 2.254\\
Monolithic SPSA, $n_{\rm pert}=256$ & 150 & 2.348 & 2.363\\
Monolithic SPSA, $n_{\rm pert}=1024$ & 148 & 2.271 & 2.290\\
Monolithic EGGROLL & 141 & 2.477 & 2.493\\
\bottomrule
\end{tabular}
\caption{Frozen next-byte evaluation on WikiText-103 for approximately 8.44M-parameter models near 150 aggregate training GPU-hours. Training compute uses each checkpoint's recorded RTX~5090 accounting. Models and checkpoints are selected before observing WikiText scores.}
\label{tab:wikitext-transfer}
\end{table}

\section{Choosing a training allocation}\label{sec:independent-evaluation}
Choosing a training configuration requires deciding how much to spend on each update and how many updates to complete. The saved learning curves let us compare those choices over their measured range. We use them to select the lowest loss within a training budget, then consider how a deadline changes the feasible allocations.

\subsection{Training budget and target loss}\label{app:practitioner-guide}

To compare allocations at a fixed budget, we first translate the cost of one update into the number of updates that budget can buy. For context length $T$ and per-expert cost $F_N$ arithmetic operations per token, update cost is $u=2Nn_{\rm pert}BT F_N$. For GPU-hours, let $t_{\rm step}$ be each allocated GPU's full step time in seconds, including communication and waiting, and use $u=\sum_{\rm GPUs}t_{\rm step}/3600$. A budget $C$ in the same unit buys $\lfloor C/u\rfloor$ updates. Let $\widehat{\mathcal L}_{N,n_{\rm pert},B}(C)$ be loss estimated from the saved curves at that budget. Among settings that fit the available memory and GPUs, the loss-minimizing choice is
\begin{equation}
(N^*,n_{\rm pert}^*,B^*)\in\arg\min_{N,n_{\rm pert},B}
\widehat{\mathcal L}_{N,n_{\rm pert},B}(C).
\label{eq:practitioner-budget-loss}
\end{equation}
Interpolate between saved losses when they cover $C$.
At $10^{18}$ arithmetic operations, this leaves $B=256$ and 1024 as candidates for the
0.272M model. Larger budgets favor $(8,64,1024)$ below.
Near 8.44M parameters, the tested choice at 100 training GPU-hours is
$(2,64,64)$ (Figure~\ref{fig:independent-evaluation}).

\begin{table}[H]\centering\small
\caption{The saved 0.272M curves favor $(N,n_{\rm pert},B)=(8,64,1024)$ at these budgets.
Losses use the ensemble validation set from Figure~\ref{fig:allocation-loss}.
Ranges span linear- and log-compute interpolation.}
\label{tab:practitioner-budget}
\begin{tabular}{rr}
\toprule
Compute (operations) & Estimated ensemble validation loss\\
\midrule
$2\times10^{18}$ & 1.8046--1.8057\\
$4\times10^{18}$ & 1.7947--1.8000\\
$8\times10^{18}$ & 1.7848--1.7887\\
\bottomrule
\end{tabular}
\end{table}

This setting first reaches 1.90 at $1.76\times10^{18}$ arithmetic operations.
Interpolating the crossing gives $(0.965\text{--}1.35)\times10^{18}$ arithmetic operations.
The matched continuation in Table~\ref{tab:allocation-continuation}
instead retains all three allocations because their losses are nearly tied.

Gradient error per global step provides a first comparison of what an allocation buys (Figure~\ref{fig:variance-efficiency}), but the ensemble validation curves determine whether that accuracy leads to a target loss with less total compute.
For allocations that accumulate independently sampled batches, Equation~\eqref{eq:accumulation-variance} accounts for the additional averaging over directions as well as data.

\label{app:tail-plots}

\subsection{Trading perturbations for larger batches}\label{app:allocation-continuation}

We test whether larger batches improve learning more than additional
perturbations. The earlier $N=8$ sweeps in Figure~\ref{fig:allocation-loss}
motivate a paired continuation from the same weights and optimizer state.

\begin{figure}[!htbp]\centering
\begin{tikzpicture}
\begin{groupplot}[
 group style={group size=2 by 1,horizontal sep=1.05cm},
 width=5.60cm,height=4.55cm,scale only axis,
 xmode=log,xmin=.1,xmax=2000,
 xtick={.1,1,10,100,1000},
 ymin=1.65,ymax=2.65,ytick={1.8,1.9,2.2,2.4,2.6},
 xlabel={Aggregate GPU-hours},grid=major,grid style={gray!22},
 tick pos=left,tick align=outside,
 tick label style={font=\scriptsize},label style={font=\scriptsize},
 every axis plot/.append style={line width=.95pt,mark=*,mark size=1.9pt},
 legend style={draw=none,fill=none,font=\scriptsize,
 at={(.5,1.03)},anchor=south,legend columns=2,column sep=6pt},
 title style={at={(.5,1.24)},anchor=south,font=\small}]
\nextgroupplot[title={\shortstack{More perturbations\\$N=8$, $B=64$}},ylabel={\shortstack{Ensemble validation loss\\(nats/byte)}}]
\addplot[gray!70,densely dotted,mark=none,forget plot]
 table[x=aggregate_gpu_hours,y=validation_loss,col sep=comma]{figdata/aggregate_gpu_hours/allocation_target.csv};
\addplot[somaA] table[x=aggregate_gpu_hours,y=validation_loss,col sep=comma]{figdata/aggregate_gpu_hours/allocation_np16.csv};
\addlegendentry{$n_{\rm pert}=16$}
\addplot[somaB] table[x=aggregate_gpu_hours,y=validation_loss,col sep=comma]{figdata/aggregate_gpu_hours/allocation_N8.csv};
\addlegendentry{$n_{\rm pert}=64$}
\addplot[somaC] table[x=aggregate_gpu_hours,y=validation_loss,col sep=comma]{figdata/aggregate_gpu_hours/allocation_np256.csv};
\addlegendentry{$n_{\rm pert}=256$}
\addplot[somaD] table[x=aggregate_gpu_hours,y=validation_loss,col sep=comma]{figdata/aggregate_gpu_hours/allocation_np1024.csv};
\addlegendentry{$n_{\rm pert}=1024$}
\nextgroupplot[title={\shortstack{Larger batches\\$N=8$, $n_{\rm pert}=64$}},yticklabels=\empty]
\addplot[gray!70,densely dotted,mark=none,forget plot]
 table[x=aggregate_gpu_hours,y=validation_loss,col sep=comma]{figdata/aggregate_gpu_hours/allocation_target.csv};
\addplot[somaA] table[x=aggregate_gpu_hours,y=validation_loss,col sep=comma]{figdata/aggregate_gpu_hours/allocation_b16.csv};
\addlegendentry{$B=16$}
\addplot[somaB] table[x=aggregate_gpu_hours,y=validation_loss,col sep=comma]{figdata/aggregate_gpu_hours/allocation_N8.csv};
\addlegendentry{$B=64$}
\addplot[somaC] table[x=aggregate_gpu_hours,y=validation_loss,col sep=comma]{figdata/aggregate_gpu_hours/allocation_ac4.csv};
\addlegendentry{$B=256$}
\addplot[somaD] table[x=aggregate_gpu_hours,y=validation_loss,col sep=comma]{figdata/aggregate_gpu_hours/allocation_ac16.csv};
\addlegendentry{$B=1024$}
\end{groupplot}
\end{tikzpicture}
\caption{Larger batches reach 1.90 with fewer GPU-hours in these runs.
All curves use $N=8$ and a 0.272M model (Section~\ref{sec:model-size}).
For $n_{\rm pert}=64,B=1024$, the first plotted point at or below 1.90 costs 28.7 GPU-hours, compared with 44.5 for $n_{\rm pert}=1024,B=64$. The larger batch uses $1.55\times$ fewer GPU-hours. Costs sum recorded mean step times across the eight experts. The dotted line marks 1.90. Batches of 256 and 1024 accumulate four and sixteen independently sampled batches of
64 with independently sampled directions. These runs are separate from
the 8.44M comparison in Figure~\ref{fig:independent-evaluation}.
The controlled continuation in Table~\ref{tab:allocation-continuation}
finds nearly the same loss for three allocations.
}
\label{fig:allocation-loss}
\label{fig:target-loss-work}
\end{figure}
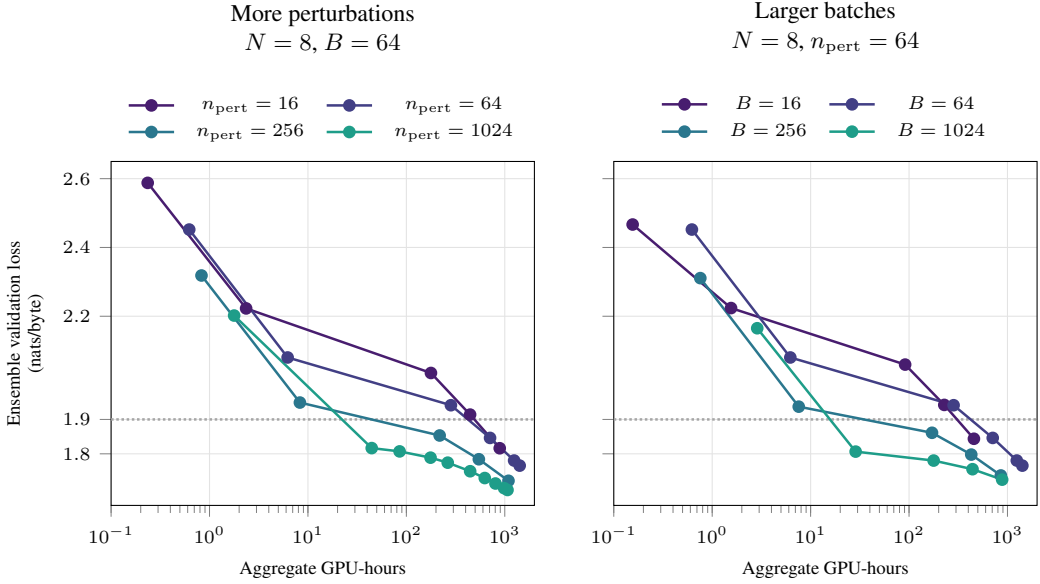

All three settings continue the same $N=8$, 0.272M model from 2,000 updates.
Heads stay frozen, and the learning rate and $\varepsilon$ stay at
$10^{-3}$. Training data and precision follow
Appendix~\ref{app:loss-separation}. Each run takes 512 further updates and
$4.50\times10^{16}$ arithmetic operations. We use two paired seeds.

\begin{table}[H]
\centering
\caption{Nearly equal test loss after 512 further updates at equal
arithmetic operations. All settings continue the same $N=8$, 0.272M model. Values average
two paired seeds on all 4,096 test documents.}
\label{tab:allocation-continuation}
\begin{tabular}{rrr}
\toprule
$n_{\rm pert}$ & $B$ & Test loss (nats/byte) \\
\midrule
64 & 1024 & 2.376587 \\
256 & 256 & 2.376575 \\
1024 & 64 & 2.376625 \\
\bottomrule
\end{tabular}
\end{table}

Each update uses 1,024 directions per expert. More independently sampled batches reduce the
second term in Equation~\eqref{eq:accumulation-variance}, leaving the first
unchanged. The three settings finish within 0.00006 nats/byte. No setting
gives a clear loss advantage in this continuation.

We ask which saved checkpoint gives the lowest loss within limits on
total work and work per expert.

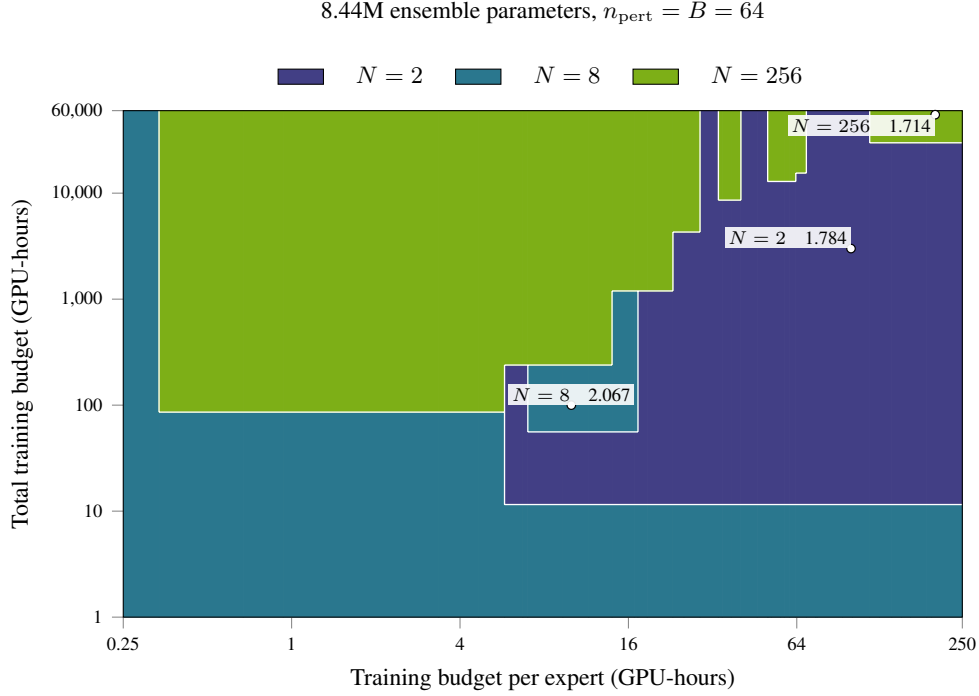
\begin{figure}[!t]\centering
\begin{tikzpicture}
\begin{axis}[
 width=11.1cm,height=6.7cm,scale only axis,
 xmode=log,ymode=log,
 xmin=.25,xmax=250,ymin=1,ymax=60000,
 xtick={.25,1,4,16,64,250},
 xticklabels={0.25,1,4,16,64,250},
 ytick={1,10,100,1000,10000,60000},
 yticklabels={1,10,100,{1,000},{10,000},{60,000}},
 xlabel={Training budget per expert (GPU-hours)},
 ylabel={Total training budget (GPU-hours)},
 title={8.44M ensemble parameters, $n_{\rm pert}=B=64$},
 title style={font=\small,yshift=23pt},
 tick pos=left,tick align=outside,
 tick label style={font=\scriptsize},label style={font=\small},
 legend columns=3,
 legend style={at={(.5,1.025)},anchor=south,draw=none,
 font=\small,cells={anchor=west},column sep=1em},
 axis on top]
\addplot[patch,patch type=rectangle,shader=flat,
 draw=somaB,line width=.08pt,forget plot,
 colormap={deadlineTwo}{color(0cm)=(somaB);color(1cm)=(somaB)}]
 table[x=y,y=x,col sep=comma]
 {figdata/plot_clarity_v2/deadline/vertices_N2.csv};
\addplot[patch,patch type=rectangle,shader=flat,
 draw=somaC,line width=.08pt,forget plot,
 colormap={deadlineEight}{color(0cm)=(somaC);color(1cm)=(somaC)}]
 table[x=y,y=x,col sep=comma]
 {figdata/plot_clarity_v2/deadline/vertices_N8.csv};
\addplot[patch,patch type=rectangle,shader=flat,
 draw=somaE,line width=.08pt,forget plot,
 colormap={deadlineMany}{color(0cm)=(somaE);color(1cm)=(somaE)}]
 table[x=y,y=x,col sep=comma]
 {figdata/plot_clarity_v2/deadline/vertices_N256.csv};
\addplot[white,line width=.5pt,no marks,forget plot,unbounded coords=jump]
 table[x=y,y=x,col sep=comma]
 {figdata/plot_clarity_v2/deadline/boundaries.csv};
\addplot[only marks,mark=*,mark size=1.6pt,
 mark options={draw=black,fill=white,line width=.45pt},forget plot,
 point meta=explicit symbolic,nodes near coords={\pgfplotspointmeta},
 every node near coord/.append style={anchor=south east,font=\scriptsize,
 fill=white,fill opacity=.9,text opacity=1,inner sep=1.4pt}]
 table[x=average_budget_gpu_hours,y=total_budget_gpu_hours,meta=label,col sep=comma]
 {figdata/plot_clarity_v2/deadline/examples_N2.csv};
\addplot[only marks,mark=*,mark size=1.6pt,
 mark options={draw=black,fill=white,line width=.45pt},forget plot,
 point meta=explicit symbolic,nodes near coords={\pgfplotspointmeta},
 every node near coord/.append style={anchor=south,font=\scriptsize,
 fill=white,fill opacity=.9,text opacity=1,inner sep=1.4pt}]
 table[x=average_budget_gpu_hours,y=total_budget_gpu_hours,meta=label,col sep=comma]
 {figdata/plot_clarity_v2/deadline/examples_N8.csv};
\addplot[only marks,mark=*,mark size=1.6pt,
 mark options={draw=black,fill=white,line width=.45pt},forget plot,
 point meta=explicit symbolic,nodes near coords={\pgfplotspointmeta},
 every node near coord/.append style={anchor=north east,font=\scriptsize,
 fill=white,fill opacity=.9,text opacity=1,inner sep=1.4pt}]
 table[x=average_budget_gpu_hours,y=total_budget_gpu_hours,meta=label,col sep=comma]
 {figdata/plot_clarity_v2/deadline/examples_N256.csv};
\addlegendimage{area legend,draw=none,fill=somaB}\addlegendentry{$N=2$}
\addlegendimage{area legend,draw=none,fill=somaC}\addlegendentry{$N=8$}
\addlegendimage{area legend,draw=none,fill=somaE}\addlegendentry{$N=256$}
\end{axis}
\end{tikzpicture}
\caption{More shards can lower loss under a tight per-expert budget.
Color gives the expert count of the lowest-loss saved checkpoint that
fits both budgets. We compare 44 checkpoints from $N=1,2,8,256$ on the
ensemble validation set of Figure~\ref{fig:loss_curves}, including both
$N=256$ recipes through 900k and 4.16M updates. $N=1$ is never selected. Average work per expert is total
GPU-hours divided by $N$. At one GPU per expert, it approximates average
elapsed training time. Labels give the selected checkpoint loss at the
marked budgets. Work is estimated from recorded update rates
and excludes initialization. The $N=256$ ensemble uses 0.168 seconds
per expert update. Region boundaries come from the saved checkpoints.
Training recipes follow the fixed-size controls in Figure~\ref{fig:independent-evaluation}, detailed in Appendix~\ref{app:figure2-controls}.
Model size follows Section~\ref{sec:model-size}.}
\label{fig:deadline-allocation}
\end{figure}

We also ask what more hardware buys under a deadline.
Near 8.44M parameters with $n_{\rm pert}=B=64$, the earlier $N=256$ recipe reaches an ensemble validation loss of
1.94 in an estimated 16.8 hours using 4.30k GPU-hours.
The $N=2$ recipe reaches 1.99 in 17.3 hours using 34.6 GPU-hours.
The $N=256$ recipe gives 0.051 lower loss at a similar training time,
but uses $124\times$ more total work and requires 256 concurrent GPUs
instead of two. Appendix~\ref{app:figure2-controls} specifies the training recipes and recorded update rates. Work excludes initialization.

\section{Inference cost and routing}\label{app:inference-benchmarks}
A model that is economical to train may still be expensive to use. Routing changes this trade-off by making only a subset of the ensemble active for each sequence. We first count the neural compute of that subset, then measure complete inference including routing, and finally examine how selecting more experts affects loss.

\label{sec:inference-allocation}
A training budget alone does not determine the deployment choice. At fixed ensemble parameter count, more shards make the routed experts smaller. Let $Q$ be the number of tokens and $F_N$ the arithmetic operations per token through one expert. Processing them costs $Q\min(4,N)F_N$ arithmetic operations.

Figure~\ref{fig:inference-budget-frontier} plots test loss against inference cost. SOMA $N=2$ gives the lowest loss at 100 estimated training GPU-hours. The later checkpoint for SOMA $N=256$ gives lower loss and uses $51.3\times$ less expert compute per token than the checkpoint for SOMA $N=2$, after much more training. Changing $n_{\rm pert}$ or $B$ affects training, but not inference cost for a fixed model.

\begin{figure}[!t]\centering
\begin{tikzpicture}
\begin{groupplot}[
 group style={group size=2 by 1,horizontal sep=.55cm},
 width=5.05cm,height=4.25cm,scale only axis,
 xmode=log,log basis x=10,xmin=.25,xmax=32,
 xtick={.32768,1,4,16},xticklabels={0.328,1,4,16},
 ymin=1.63,ymax=2.24,ytick={1.7,1.8,1.9,2.0,2.1,2.2},
 xlabel={\shortstack{Neural inference compute\\(million operations/token)}},
 title style={font=\small},
 tick align=outside,tick pos=left,
 tick label style={font=\scriptsize},label style={font=\small},
 ymajorgrids,grid style={gray!20},
 every axis plot/.append style={only marks,mark=*,mark size=2.5pt},
 error bars/y dir=both,error bars/y explicit,
 error bars/error bar style={line width=.8pt},
 error bars/error mark options={rotate=90,mark size=2.4pt,line width=.8pt}]
\nextgroupplot[title={100 training GPU-hours},ylabel={Test loss (nats/byte)}]
\addplot[somaA] table[x=inference_mflops_per_token,y=loss,y error minus=err_minus,y error plus=err_plus,col sep=comma]{figdata/plot_clarity_v2/inference/matched_N1.csv};
\addplot[somaB] table[x=inference_mflops_per_token,y=loss,y error minus=err_minus,y error plus=err_plus,col sep=comma]{figdata/plot_clarity_v2/inference/matched_N2.csv};
\addplot[somaC] table[x=inference_mflops_per_token,y=loss,y error minus=err_minus,y error plus=err_plus,col sep=comma]{figdata/plot_clarity_v2/inference/matched_N8.csv};
\addplot[somaE] table[x=inference_mflops_per_token,y=loss,y error minus=err_minus,y error plus=err_plus,col sep=comma]{figdata/plot_clarity_v2/inference/matched_N256.csv};
\node[font=\small,text=somaA,anchor=south] at (axis cs:16,2.039) {$N=1$};
\node[font=\small,text=somaB,anchor=north] at (axis cs:16,1.783) {$N=2$};
\node[font=\small,text=somaC,anchor=north] at (axis cs:7.4,1.980) {$N=8$};
\node[font=\small,text=somaE,anchor=west] at (axis cs:.41,2.172) {$N=256$};
\nextgroupplot[title={Later checkpoints},yticklabels={}]
\addplot[somaA] table[x=inference_mflops_per_token,y=loss,y error minus=err_minus,y error plus=err_plus,col sep=comma]{figdata/plot_clarity_v2/inference/late_N1.csv};
\addplot[somaB] table[x=inference_mflops_per_token,y=loss,y error minus=err_minus,y error plus=err_plus,col sep=comma]{figdata/plot_clarity_v2/inference/late_N2.csv};
\addplot[somaE] table[x=inference_mflops_per_token,y=loss,y error minus=err_minus,y error plus=err_plus,col sep=comma]{figdata/plot_clarity_v2/inference/late_N256.csv};
\node[font=\small,text=somaA,anchor=south east,align=center] at (axis cs:25,2.018) {$N=1$\\\scriptsize(152 GPU-hours)};
\node[font=\small,text=somaB,anchor=south east,align=center] at (axis cs:25,1.785) {$N=2$\\\scriptsize(150 GPU-hours)};
\node[font=\small,text=somaE,anchor=west,align=left] at (axis cs:.41,1.690) {$N=256$\\\scriptsize(41.9k GPU-hours)};

\end{groupplot}
\node[font=\small,anchor=south] at ([yshift=1.05cm]group c1r1.north east) {$n_{\rm pert}=B=64$, 8.44M ensemble parameters};
\end{tikzpicture}
\caption{More shards reduce inference cost, but training progress changes their ranking.
Left, the four settings from Figure~\ref{fig:independent-evaluation} at 100 estimated training GPU-hours, with losses interpolated between checkpoints. Right, recorded $N=1,2,256$ checkpoints at 152, 150, and 41.9k estimated training GPU-hours, using the recorded-update accounting in Table~\ref{tab:n256-compute-accounting}. Their training budgets differ. Both panels use the same test set. Bars show 95\% document-bootstrap intervals and exclude interpolation uncertainty on the left. Arithmetic operations count the selected experts and exclude routing. Parameter counts and arithmetic operations follow Sections~\ref{sec:model-size} and~\ref{sec:compute-cost}.}
\label{fig:inference-budget-frontier}
\end{figure}
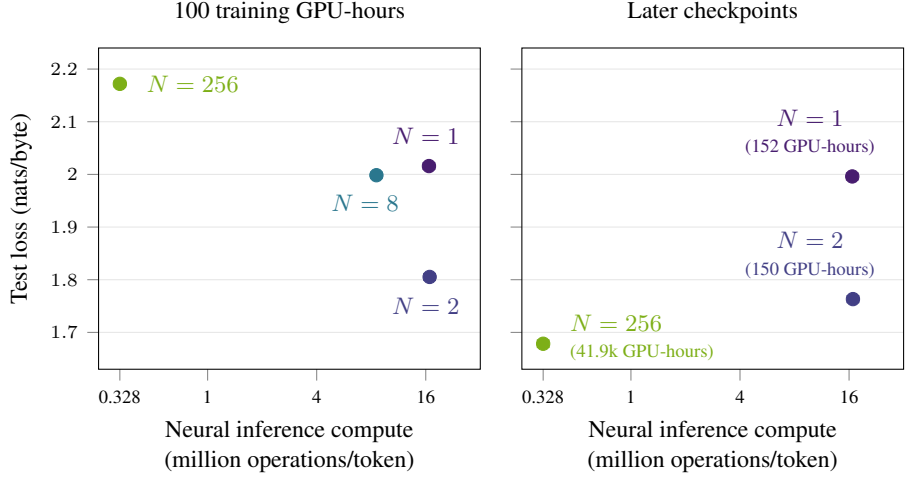
\subsection{Routing and model execution}
Reducing active neural compute does not determine total latency because routing and grouping requests also take time. We therefore compare complete inference for a monolith and SOMA $N=2,8,256$ on the same GPU, with total neural sizes of 8.33M, 8.41M, 8.55M and 8.44M parameters, respectively. Each request processes a 1,024-byte context and returns next-byte probabilities for its final 768 positions. The monolith executes one model and SOMA $N=2$ averages both experts, so neither requires routing. For SOMA $N=8,256$, top-$k$ routing with $k=4$ uses the first 256 observed bytes before recurrent execution.

The optimized router caches a contiguous transpose of the fitted SVD matrix on the CPU. All 4,096 tested prefixes retain identical ordered top-$k$ choices. Requests are grouped by expert, with both SOMA $N=8$ and $N=256$ using grouped Triton recurrent kernels. The timer includes routing, transfers, dispatch, recurrent execution, softmax and probability averaging, ending after GPU synchronization. Models and router remain resident. Loading, compilation and copying the full output to the CPU are excluded. Tokens/s measures teacher-forced sequence scoring, not autoregressive generation.

We tune batch sizes, dispatch padding and kernel settings on NVIDIA A40 GPUs, then confirm the fastest validated settings sequentially on the same idle GPU with two CPU threads, FP32 and TF32 disabled. Each timing is the median of seven repetitions after two warmups. Monolithic SPSA and SOMA $N=2$ retain cuDNN, which outperforms their tested Triton implementations. Both routed models use optimized grouped Triton. Their complete predictions agree with independent cuDNN references within $2.15\times10^{-6}$ in probability and $1.4\times10^{-8}$ in mean loss. These are the fastest validated configurations in our bounded search, not a guarantee of globally optimal implementations.

\begin{table}[!t]
\centering\small
\setlength{\tabcolsep}{4pt}
\begin{tabular}{lrrrr}
\toprule
Model & Batch & \shortstack{Throughput\\(M tokens/s)} & \shortstack{Batch\\latency (ms)} & \shortstack{Peak GPU\\memory (GiB)}\\
\midrule
Monolithic SPSA & 768 & 0.555 & 1,060 & 16.5\\
SOMA $N=2$ & 768 & 0.468 & 1,260 & 12.3\\
SOMA $N=8$ & 1,024 & 0.257 & 3,060 & 15.0\\
SOMA $N=256$ & 4,096 & 2.36 & 1,330 & 8.30\\
\bottomrule
\end{tabular}
\caption{Complete warm inference at the fastest validated settings, including routing when needed. Each token is one byte, with 768 evaluated tokens per 1,024-byte request. Values are medians of seven repetitions on the same GPU. Peak allocated memory includes intermediates. Compilation and warmup are excluded.}
\label{tab:inference-full}
\end{table}

SOMA $N=256$ achieves $9.19\times$ the throughput of SOMA $N=8$. At the selected batches, routing takes 0.87\% and 7.50\% of total time for SOMA $N=8$ and $N=256$, respectively. Optimizing SOMA $N=8$ increases throughput $3.07\times$ over its earlier implementation. It still trails the monolith: profiling identifies recurrent execution, normalization and dispatch padding as substantial costs. Four width-181 experts also have a combined activation width of 724, versus 509 for the monolith, so fewer active parameters need not imply proportionally less execution time.

To distinguish inference throughput from training progress, Table~\ref{tab:inference-n8-n256-loss} reports the tested checkpoints and their training budgets. The later SOMA $N=256$ checkpoint has 0.0330 nats/byte lower test loss than SOMA $N=8$, after longer training. At roughly 100 estimated training hours, SOMA $N=8$ instead has 0.0410 nats/byte lower loss. Thus the endpoint comparison does not establish a matched-duration loss advantage for SOMA $N=256$.

\begin{table}[!t]
\centering\small
\setlength{\tabcolsep}{5pt}
\begin{tabular}{lrrrrr}
\toprule
Model & Parameters & Updates & \shortstack{Estimated training\\wall-clock hours} & \shortstack{Aggregate training\\GPU-hours} & Test loss\\
\midrule
SOMA $N=8$ & 8.55M & 53.5k--70.0k & 104 & 700 & 1.71\\
SOMA $N=256$ & 8.44M & 2.18M & 100 & 22.0k & 1.75\\
SOMA $N=256$ & 8.44M & 4.16M & 191 & 41.9k & 1.68\\
\bottomrule
\end{tabular}
\caption{Test loss and training budgets for the approximately equal-total-size inference comparison. All rows use top-$k$ routing with $k=4$. The SOMA $N=8$ and later SOMA $N=256$ checkpoints are timed in Table~\ref{tab:inference-full}. The middle row supplies the nearest measured SOMA $N=256$ checkpoint to the training duration of SOMA $N=8$. Wall-clock and aggregate time retain the accounting used in Figure~\ref{fig:independent-evaluation}, excluding initialization, evaluation and checkpointing.}
\label{tab:inference-n8-n256-loss}
\end{table}

\subsection{Number of active experts}\label{app:topk-ablation}
The number of selected experts trades inference compute against the benefit of averaging more predictions. To measure that trade-off without changing training, the sweep in Section~\ref{sec:active-parameter-inference} and Figure~\ref{fig:topk-flops-frontier} evaluates $k=1,\ldots,N$ using the same saved 4M-update $N=8,32,256$ ensembles. For each sequence, the router selects the nearest $k$ centroids before the experts run. Scores use the ensemble validation set and average the selected experts' probabilities.

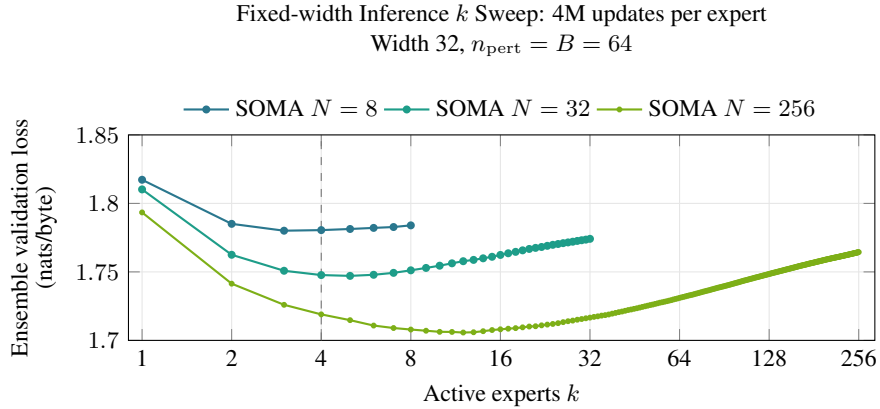
\begin{figure}[H]\centering
\begin{tikzpicture}
\begin{axis}[
  width=.82\linewidth,height=4.3cm,
  xmode=log,log basis x=2,
  xlabel={Active experts $k$},ylabel={\shortstack{Ensemble validation loss\\(nats/byte)}},
  xmin=.9,xmax=290,ymin=1.7,ymax=1.85,
  xtick={1,2,4,8,16,32,64,128,256},
  xticklabels={1,2,4,8,16,32,64,128,256},
  title={\shortstack{Fixed-width Inference $k$ Sweep: 4M updates per expert\\Width 32, $n_{\rm pert}=B=64$}},
  title style={font=\footnotesize,yshift=20pt},
  grid=both,grid style={black!10},legend columns=3,
  legend style={draw=none,at={(.5,1.02)},anchor=south,font=\footnotesize},
  tick label style={font=\footnotesize},label style={font=\footnotesize}]
\addplot[somaC,thick,mark=*,mark size=1.1pt]
 table[x=k,y=valcanon_nats,col sep=comma]{figdata/plot_clarity_v2/topk/n8.csv};
\addlegendentry{SOMA $N=8$}
\addplot[somaD,thick,mark=*,mark size=1.1pt]
 table[x=k,y=valcanon_nats,col sep=comma]{figdata/plot_clarity_v2/topk/n32.csv};
\addlegendentry{SOMA $N=32$}
\addplot[somaE,thick,mark=*,mark size=.7pt]
 table[x=k,y=valcanon_nats,col sep=comma]{figdata/plot_clarity_v2/topk/n256.csv};
\addlegendentry{SOMA $N=256$}
\addplot[black!55,dashed,no marks,forget plot]
 table[x=k,y=loss,col sep=comma]{figdata/plot_clarity_v2/topk/top4_line.csv};
\end{axis}
\end{tikzpicture}
\caption{Top-$k$ routing with $k=4$ captures most of the gain at the same training stage.
Each curve varies $k$ in a saved 4M-update ensemble. All experts have the same width and use $n_{\rm pert}=B=64$. SOMA $N=8$ and SOMA $N=256$ reach these checkpoints after approximately 176 and 192 elapsed hours, respectively. All scores use the same ensemble validation set and average the selected experts' probabilities. The dashed line marks $k=4$.}
\label{fig:topk-flops-frontier}
\end{figure}

\subsection{Complementary expert predictions}\label{app:wd-diversity}
We ask whether averaging two experts helps more when they make different errors.
For each of the 32,640 pairs in the $N=256$ ensemble, we evaluate the two
experts separately and average their predicted probabilities.
The loss reduction is the mean of the two separate losses minus the loss
of their averaged prediction. It averages 0.083 nats/byte across pairs.

To measure whether two experts struggle on the same tokens, we subtract
the across-expert mean loss at each token and correlate the remaining
errors. This removes the difficulty shared by all experts.
Figure~\ref{fig:n256-correlation} shows that pairs with less similar errors
benefit more from averaging.

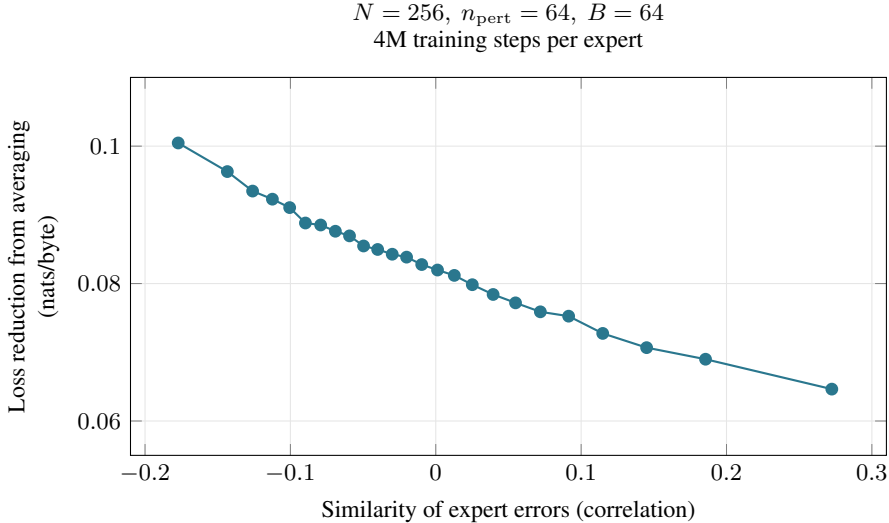
\begin{figure}[H]\centering
\begin{tikzpicture}
\begin{axis}[
width=10cm,height=5.0cm,scale only axis,
tick label style={font=\small},label style={font=\small},
title style={font=\small,align=center},
title={$N=256,\ n_{\rm pert}=64,\ B=64$\\4M training steps per expert},
xlabel={Similarity of expert errors (correlation)},
ylabel={\shortstack{Loss reduction from averaging\\(nats/byte)}},
xmin=-.21,xmax=.31,ymin=.055,ymax=.11,
xtick={-.2,-.1,0,.1,.2,.3},
xticklabels={$-0.2$,$-0.1$,$0$,$0.1$,$0.2$,$0.3$},
ytick={.06,.08,.10},
yticklabel style={/pgf/number format/fixed,/pgf/number format/precision=2,font=\small},
grid=major,grid style={gray!20}]
\addplot[somaC,thick,mark=*,mark size=2pt]
table[x=correlation,y=gain,col sep=comma]{figdata/reviewer_evidence/n256_pair_gain_bins.csv};
\end{axis}
\end{tikzpicture}
\caption{Averaging helps more when experts make different errors.
The horizontal axis measures error similarity after removing shared token difficulty.
The vertical axis measures how much lower the averaged prediction's ensemble validation loss is
than the mean loss of the two experts alone. For example, 0.10 means a
reduction of 0.10 nats/byte. Each point averages about 1,300 expert pairs
with similar error correlations, using 749,568 ensemble validation tokens.
Pairs share experts. This tests averaging pairs of experts, not top-4 routing.}
\label{fig:n256-correlation}
\end{figure}

\section{Limitations and open questions}\label{app:limitations}

The measured ZO improvements do not establish an advantage over backpropagation. BPTT reaches lower ensemble validation loss at every tested budget in Table~\ref{tab:training-budget-comparison}.

At fixed total model size, sharding reduces expert width and joint adaptation across domains. Routing also limits the active model size. Improved estimation therefore need not give the lowest loss, as the modest-budget comparisons in Figure~\ref{fig:independent-evaluation} show. Each expert sees approximately $1/N$ of the corpus, with about 700 training passes per shard at $N=256$ (Appendix~\ref{app:experimental-details}). Late improvement thus involves substantial data reuse.

Gradient estimates remain noisy. At 10,000 updates, $N=256$ has relative centered variance 268 with $n_{\rm pert}=64$ (Appendix~\ref{app:estimator-iso}). This describes one training stage. Figure~\ref{fig:independent-evaluation} tests the alternative of increasing monolithic SPSA to $n_{\rm pert}=256$ and 1024 over the recorded budgets.

Paired continuations establish repeatability of the local-loss intervention, but not variability across full pretraining runs. Our pretraining scope is also limited to FineWeb-Edu, LSTM experts and one fixed tf--idf router. Other corpora, architectures and routing methods may have different trade-offs. The router does not adapt as experts learn, and the seed-trained byte decoder does not establish the feasibility of a larger token vocabulary.

\end{document}